\pdfoutput=1
\documentclass[10pt]{article}
\usepackage{enumerate}
\usepackage[OT1]{fontenc}
\usepackage{amsmath,amssymb}
\usepackage{natbib}
\usepackage[usenames,dvipsnames]{xcolor}
\usepackage{geometry}
\usepackage{dsfont}
\usepackage{pgfplots}
\usepackage{smile}
\usepackage{multirow}
\usepackage{rotating}
\usepackage{enumerate}
\usepackage{esvect}
\usepackage{tikz}
\usetikzlibrary{patterns}
\usetikzlibrary{arrows}
\usepackage[colorlinks,
            linkcolor=red,
            anchorcolor=blue,
            citecolor=blue
            ]{hyperref}
\usepackage{algorithm}
\usepackage{algorithmic}

\def\supp{\mathop{\text{supp}}}

\def\rank{\mathrm{rank}}
\def\tr{\mathop{\text{Tr}}}

\def\cS{{\mathcal{S}}}

\def\oper{\mathop{\text{op}}}

\newcommand{\bel}{\begin{eqnarray}\label}
\newcommand{\eel}{\end{eqnarray}}
\newcommand{\bes}{\begin{eqnarray*}}
\newcommand{\ees}{\end{eqnarray*}}

\newcommand{\la}{\langle}
\newcommand{\ra}{\rangle}

\newcommand{\pess}{\texttt{Pess}}
\newcommand{\linpess}{\texttt{Pess}}

\let\hat\widehat
\let\tilde\widetilde

\def\E{{\mathbb E}}

\def\supp{\mathop{\text{supp}\kern.2ex}}
\def\argmin{\mathop{\text{\rm arg\,min}}}
\def\argmax{\mathop{\text{\rm arg\,max}}}

\def\given{{\,|\,}}

\def\tr{{\rm{Tr}}}

\def\supp{\mathop{\text{supp}}}

\def\rank{\mathrm{rank}}

\def\tr{\mathrm{Tr}}

\theoremstyle{plain}

\usepackage{mathrsfs}
\usepackage{fullpage}

\def \iidtext {\textrm{i.i.d.}}

\usepackage{hyperref}
\usepackage[protrusion=false,expansion=true]{microtype}

\def\##1\#{\begin{align}#1\end{align}}
\def\$#1\${\begin{align*}#1\end{align*}}

\newcommand{\revise}[1]{{\color{black} #1}}

\title{\huge Is Pessimism Provably Efficient for Offline RL?}
\author{Ying Jin\thanks{Equal contribution. This work is done when Z. Yang and Z. Wang are  participating in the Theory of Reinforcement Learning Program  at the Simons Institute for the Theory of Computing.}~\thanks{Stanford University. Email: \texttt{ying531@stanford.edu}.}  \qquad  \qquad Zhuoran Yang$^*$\thanks{Princeton University. Email: \texttt{zy6@princeton.edu}.}  \qquad \qquad Zhaoran Wang$^*$\thanks{Northwestern University. Email: \texttt{zhaoranwang@gmail.com}.}}
\date{}
\begin{document}
\maketitle

\begin{abstract}
%\begin{flushleft}
We study offline reinforcement learning (RL), which aims to learn an optimal policy based on a dataset collected a priori. Due to the lack of further interactions with the environment,  offline RL suffers from the insufficient coverage of the dataset, which eludes most existing theoretical analysis. In this paper, we propose a \underline{pe}ssimistic variant of the \underline{v}alue \underline{i}teration algorithm (PEVI), which incorporates an uncertainty quantifier as the penalty function. Such a penalty function simply flips the sign of the bonus function for promoting exploration in online RL, which makes it easily implementable and compatible with general function approximators. 

Without assuming the sufficient coverage of the dataset (e.g., finite concentrability coefficients or uniformly lower bounded densities of visitation measures), we establish a data-dependent upper bound on the suboptimality of PEVI for general Markov decision processes (MDPs). When specialized to linear MDPs, it matches the information-theoretic lower bound up to multiplicative factors of the dimension and horizon. In other words, pessimism is not only provably efficient but also minimax optimal. In particular, given the dataset, the learned policy serves as the ``best effort'' among all policies, as no other policies can do better. Our theoretical analysis identifies the critical role of pessimism in eliminating a notion of spurious correlation, which arises from the ``irrelevant'' trajectories that are less covered by the dataset and not informative for the optimal policy. 

\end{abstract}

% !TEX root = paper.tex
%\begin{flushleft}
\section{Introduction}
The empirical success of online (deep) reinforcement learning (RL) \citep{mnih2015human,silver2016mastering,silver2017mastering,vinyals2017starcraft} relies on two ingredients: (i) expressive function approximators, e.g., deep neural networks \citep{lecun2015deep}, which approximate policies and values, and (ii) efficient data generators, e.g., game engines \citep{bellemare2013arcade} and physics simulators \citep{todorov2012mujoco}, which serve as environments. In particular, learning the deep neural network in an online manner often necessitates millions to billions of interactions with the environment. Due to such a barrier of sample complexity, it remains notably more challenging to apply online RL in critical domains, e.g., precision medicine \citep{gottesman2019guidelines} and autonomous driving \citep{shalev2016safe}, where interactive data collecting processes can be costly and risky. To this end, we study offline RL in this paper, which aims to learn an optimal policy based on a dataset collected a priori without further interactions with the environment. Such datasets are abundantly available in various domains, e.g., electronic health records for precision medicine \citep{chakraborty2014dynamic} and human driving trajectories for autonomous driving \citep{sun2020scalability}. 

In comparison with online RL \citep{lattimore2020bandit, agarwal2019reinforcement}, offline RL remains even less understood in theory \citep{lange2012batch, levine2020offline}, which hinders principled developments of trustworthy algorithms in practice. In particular, as active interactions with the environment are infeasible, it remains unclear how to maximally exploit the dataset without further exploration. Due to such a lack of continuing exploration, which plays a key role in online RL, any algorithm for offline RL possibly suffers from the insufficient coverage of the dataset \citep{wang2020statistical}. Specifically, as illustrated in Section \ref{sec:demo}, two challenges arise: 
\begin{itemize}
\item[(i)] the intrinsic uncertainty, that is, the dataset possibly fails to cover the trajectory induced by the optimal policy, which however carries the essential information, and 
\item[(ii)] the spurious correlation, that is, the dataset possibly happens to cover a trajectory unrelated to the optimal policy, which by chance induces a large cumulative reward and hence misleads the learned policy. 
\end{itemize}
See Figures \ref{fig:w01} and \ref{fig:w02} for illustrations. As the dataset is collected a priori, which is often beyond the control of the learner, any assumption on the sufficient coverage of the dataset possibly fails to hold in practice \citep{fujimoto2019off, agarwal2020optimistic, fu2020d4rl, gulcehre2020rl}. 

In this paper, we aim to answer the following question: 
 \begin{center}
{\it Is it possible to design a provably efficient algorithm for offline RL\\ under minimal assumptions on the dataset?}
\end{center}

To this end, we propose a pessimistic value iteration algorithm (PEVI), which incorporates a penalty function (pessimism) into the value iteration algorithm \citep{sutton2018reinforcement, szepesvari2010algorithms}. Here the penalty function simply flips the sign of the bonus function (optimism) for promoting exploration in online RL \citep{jaksch2010near, azar2017minimax}, which enables a straightforward implementation of PEVI in practice. Specifically, we study the episodic setting of the Markov decision process (MDP). Our theoretical contribution is fourfold: 
\begin{enumerate}[(i)]
\item  We decompose the suboptimality of any algorithm for offline RL into three sources, namely the intrinsic uncertainty, spurious correlation, and optimization error. In particular, we identify the key role of the spurious correlation, even in the multi-armed bandit (MAB), a special case of the MDP. 
\item  For any general MDP, we establish the suboptimality of PEVI under a sufficient condition on the penalty function. In particular, we prove as long as the penalty function is an uncertainty quantifier, which is defined in Section \ref{sec:pes_general}, pessimism allows PEVI to eliminate the spurious correlation from its suboptimality. 
\item  For the linear MDP \citep{yang2019sample, jin2020provably}, we instantiate PEVI by specifying the penalty function. In particular, we prove such a penalty function is an uncertainty quantifier, which verifies the sufficient condition imposed in (ii). Correspondingly, we establish the suboptimality of PEVI for the linear MDP. %\textcolor{red}{As a consequence, we show that the sublinear decay of suboptimality only depends on the sufficient coverage of the optimal trajectory, ruling out the spurious correlation arising from the dependency of the output policy on the dataset.}
\item  We prove PEVI is minimax optimal for the linear MDP up to multiplicative factors of the dimension and horizon. In particular, we prove the intrinsic uncertainty identified in (i) is impossible to eliminate, as it arises from the information-theoretic lower bound. Moreover, such a fundamental limit certifies an oracle property of PEVI, which is defined in Section \ref{sec:pes_linear}. Specifically, the suboptimality of PEVI only depends on how well the dataset covers the trajectory induced by the optimal policy, which carries the essential information, rather than any trajectory unrelated to the optimal policy, which causes the spurious correlation. 
\end{enumerate}
Throughout our theory, we only require an assumption on the compliance of the dataset, that is, the data collecting process is carried out in the underlying MDP of interest. Such an assumption is minimal. In comparison with existing literature, we require no assumptions on the sufficient coverage of the dataset, e.g., finite concentrability coefficients \citep{chen2019information} and uniformly lower bounded densities of visitation measures \citep{yin2020near}, which often fail to hold in practice. Meanwhile, we impose no restrictions on the affinity between the learned policy and behavior policy (for collecting data) \citep{liu2020provably}, which is often employed as a regularizer (or equivalently, a constraint) in existing literature. See Section \ref{sec:rw} for a detailed discussion. 

%\begin{flushleft}

\subsection{Related Works}\label{sec:rw}

% !TEX root = paper.tex

Our work adds to the vast body of existing literature on offline RL (also known as 
batch RL) \citep{lange2012batch, levine2020offline}, 
where a learner only has access to a dataset collected a priori. Existing literature studies two tasks: (i) offline policy evaluation, which estimates the expected cumulative reward or (action- and state-) value functions of a target policy, and (ii) offline policy optimization, which learns an optimal policy that maximizes the expected cumulative reward. Note that (i) is also known as off-policy policy evaluation, which can be adapted to handle the online setting. Also, note that the target policy in (i) is known, while the optimal policy in (ii) is unknown. As (ii) is more challenging than (i), various algorithms for solving (ii), especially the value-based approaches, can be adapted to solve (i). Although we focus on (ii), we discuss the existing works on (i) and (ii) together.  

%RL agent is only given access to a fixed dataset collected a priori using one or many behavior policies. 
%Under such an offline setting, 
%the agent usually has one of the following two objectives: (i) predicting the expected total return or the value function associated with a desired policy, also known as the target policy, or (ii) learning the optimal policy  of the underlying MDP.  
%The first task is also referred to as off-policy evaluation (OPE) in the literature. 
%For the second task, a  value-based RL method learns the optimal policy  by estimating the corresponding value function, and thus a value-based algorithm 
%often can be readily modified for solving OPE. 
%This is also the case for PEVI ({\red Maybe say this in the algorithm section?} ). 
%Thus, in this section, we discuss the existing works on these two tasks together.   

A key challenge of offline RL is the insufficient coverage of the dataset \citep{wang2020statistical}, which arises from the lack of continuing exploration \citep{szepesvari2010algorithms}. In particular, the trajectories given in the dataset and those induced by the optimal policy (or the target policy) possibly have different distributions, which is also known as distribution shift \citep{levine2020offline}. As a result, intertwined with overparameterized function approximators, e.g., deep neural networks, offline RL possibly suffers from the extrapolation error \citep{fujimoto2019off}, which is large on the states and actions that are less covered by the dataset. Such an extrapolation error further propagates through each iteration of the algorithm for offline RL, as it often relies on bootstrapping \citep{sutton2018reinforcement}. 

To address such a challenge, the recent works \citep{fujimoto2019off, laroche2019safe, jaques2019way, wu2019behavior, kumar2019stabilizing, kumar2020conservative, agarwal2020optimistic, yu2020mopo, kidambi2020morel, wang2020critic, siegel2020keep, nair2020accelerating, liu2020provably} demonstrate the empirical success of various algorithms, which fall into two (possibly overlapping) categories: (i) regularized policy-based approaches and (ii) pessimistic value-based approaches. Specifically, (i) regularizes (or equivalently, constrains) the policy to avoid visiting the states and actions that are less covered by the dataset, while (ii) penalizes the (action- or state-) value function on such states and actions. 
 
On the other hand, the empirical success of offline RL mostly eludes existing theory. Specifically, the existing works require various assumptions on the sufficient coverage of the dataset, which is also known as data diversity \citep{levine2020offline}. For example, offline policy evaluation often requires the visitation measure of the behavior policy to be lower bounded uniformly over the state-action space. An alternative assumption requires the ratio between the visitation measure of the target policy and that of the behavior policy to be upper bounded uniformly over the state-action space. See, e.g., \cite{jiang2016doubly, thomas2016data, farajtabar2018more, liu2018breaking, xie2019towards, nachum2019dualdice, nachum2019algaedice, tang2019doubly, kallus2019efficiently, kallus2020doubly, jiang2020minimax, uehara2020minimax, duan2020minimaxoptimal, yin2020asymptotic, yin2020near, nachum2020reinforcement, yang2020off, zhang2019gendice} and the references therein. As another example, offline policy optimization often requires the concentrability coefficient to be upper bounded, whose definition mostly involves taking the supremum of a similarly defined ratio over the state-action space. See, e.g., \cite{antos2007fitted, antos2008learning, munos2008finite, farahmand2010error, farahmand2016regularized, scherrer2015approximate, chen2019information, liu2019neural, wang2019neural, fu2020single, fan2020theoretical, xie2020batch, xie2020q, liao2020batch, zhang2020variational} and the references therein.

In practice, such assumptions on the sufficient coverage of the dataset often fail to hold \citep{fujimoto2019off, agarwal2020optimistic, fu2020d4rl, gulcehre2020rl}, which possibly invalidates existing theory. For example, even for the MAB, a special case of the MDP, it remains unclear how to maximally exploit the dataset without such assumptions, e.g.,  when each action (arm) is taken a different number of times. As illustrated in Section \ref{sec:demo}, assuming there exists a suboptimal action that is less covered by the dataset, it possibly interferes with the learned policy via the spurious correlation. As a result, it remains unclear how to learn a policy whose suboptimality only depends on how well the dataset covers the optimal action instead of the suboptimal ones. In contrast, our work proves that pessimism resolves such a challenge by eliminating the spurious correlation, which enables exploiting the essential information, e.g., the observations of the optimal action in the dataset, in a minimax optimal manner. Although the optimal action is unknown, our algorithm adapts to identify the essential information in the dataset via the oracle property. See Section \ref{sec:pess} for a detailed discussion. 

Our work adds to the recent works on pessimism \citep{yu2020mopo, kidambi2020morel, kumar2020conservative, liu2020provably, buckman2020importance}. Specifically, \cite{yu2020mopo, kidambi2020morel} propose a pessimistic model-based approach, while \cite{kumar2020conservative} propose a  pessimistic value-based approach, both of which demonstrate  empirical successes. From a theoretical perspective, \cite{liu2020provably} propose a regularized (and pessimistic) variant of the fitted Q-iteration algorithm \citep{antos2007fitted, antos2008learning, munos2008finite}, which attains the optimal policy within a restricted class of policies without assuming the sufficient coverage of the dataset. In contrast, our work imposes no restrictions on the affinity between the learned policy and behavior policy. In particular, our algorithm attains the information-theoretic lower bound for the linear MDP \citep{yang2019sample, jin2020provably} (up to multiplicative factors of the dimension and horizon), which implies that given the dataset, the learned policy serves as the ``best effort'' among all policies since no other can do better. From another theoretical perspective, \cite{buckman2020importance} characterize the importance of pessimism, especially when the assumption on the sufficient coverage of the dataset fails to hold. In contrast, we propose a principled framework for achieving pessimism via the notion of uncertainty quantifier, which serves as a sufficient condition for general function approximators. See Section \ref{sec:pess} for a detailed discussion. Moreover, we instantiate such a framework for the linear MDP and establish its minimax optimality via the information-theoretic lower bound.~In other words, our work complements \cite{buckman2020importance} by proving that pessimism is not only ``important'' but also optimal in the sense of information theory.

\section{Preliminaries}
%\begin{flushleft}

In this section, we first introduce the episodic Markov decision process (MDP) and the corresponding performance metric. Then we introduce the offline setting and the corresponding data collecting process.

%To accommodate stochastic reward, we denote with a slight abuse of notation by $\cP_h(\cdot\given \cdot,\cdot)$ the transition kernel, which specifies the conditional distribution of next state $s_{h+1}$ and reward $r_h$ given current state-action pair $(s_h,a_h)$.

\subsection{Episodic MDP and Performance Metric} 
We consider an episodic MDP $(\cS,\cA,H,\cP,r)$ with the state space $\cS$, action space $\cA$, horizon $H$, transition kernel $\cP = \{\cP_h\}_{h=1}^H$, and reward function $r = \{r_h\}_{h=1}^H$. We assume the reward function is bounded, that is, $r_h\in [0,1]$ for all $h\in[H]$. For any policy $\pi=\{\pi_h\}_{h=1}^H$, we define the (state-)value function $V_h^\pi:\cS\to \RR$ at each step $h\in[H]$ as
\begin{equation}
V_h^\pi(x) = \EE_{\pi}\Big[ \sum_{i=h}^H r_i(s_i, a_i)\Biggiven s_h=x  \Big]
\label{eq:def_value_fct}
\end{equation}
and the action-value function (Q-function) $Q_h^\pi:\cS\times \cA\to \RR$ at each step $h\in[H]$ as
\begin{equation}
Q_h^\pi(x,a) = \EE_\pi\Big[\sum_{i=h}^H r_i(s_i, a_i)\Biggiven s_h=x, a_h=a  \Big].
\label{eq:def_q_fct}
\end{equation}
Here the expectation  $\EE_{\pi}$ in Equations  \eqref{eq:def_value_fct} and \eqref{eq:def_q_fct} is   taken with respect to the randomness of the trajectory induced by $\pi$, which is obtained by  taking the action $a_i\sim \pi_i(\cdot\given s_i)$ at the state $s_i$ and observing the next state $s_{i+1} \sim \cP_i(\cdot \given s_i, a_i)$ at each step $i\in[H]$. 
Meanwhile, we fix $s_h = x \in \cS $ in Equation \eqref{eq:def_value_fct} and $(s_h, a_h) = (x, a) \in \cS\times \cA$ in Equation \eqref{eq:def_q_fct}. 
By the definition in Equations \eqref{eq:def_value_fct} and \eqref{eq:def_q_fct}, we have the Bellman equation
\begin{equation*}
V_h^\pi(x) = \langle Q_h^\pi(x, \cdot),\pi_h(\cdot\given x)\rangle_{\cA},\quad Q_h^\pi(x,a) = \EE\bigl[r_h(s_h, a_h) + V_{h+1}^\pi(s_{h+1})\biggiven s_h=x,a_h=a\bigr],
\end{equation*}
where $\langle \cdot,\cdot\rangle_{\cA}$ is the inner product over $\cA$, while $\EE$ is taken with respect to the randomness of the immediate  reward $r_h(s_h, a_h)$ and next state $s_{h+1}$. For any function $f:\cS\to \RR$, we define the transition operator at each step $h\in[H]$ as
\begin{equation}
(\PP_h f)(x,a) = \EE\bigl[ f(s_{h+1})\biggiven s_h=x,a_h=a\bigr]
\label{eq:def_transition_op}
\end{equation}
and the Bellman operator at each step $h\in[H]$ as
\begin{align}
(\BB_h f)(x,a) &= \EE\bigl[r_h(s_h, a_h) + f(s_{h+1})\biggiven s_h=x,a_h=a\bigr]\notag\\
&= \EE\bigl[r_h(s_h, a_h) \biggiven s_h=x,a_h=a\bigr] + (\PP_h f)(x,a).
\label{eq:def_bellman_op}
\end{align}
For the episodic  MDP $(\cS,\cA,H,\cP,r)$, we use $\pi^*$, $Q_h^*$, and $V_h^*$ to denote the optimal policy, optimal Q-function, and optimal value function, respectively. We have $V_{H+1}^*= 0$ and the Bellman optimality equation
\begin{equation}
 V_{h}^*(x) = \max_{a\in \cA}Q_h^*(x,a),\quad Q_h^*(x,a) = (\BB_h V_{h+1}^*) (x,a).
\label{eq:dp_optimal_values}
\end{equation}
Meanwhile, the optimal policy $\pi^*$ is specified by 
$$
\pi^*_h (\cdot \given x)=\argmax_{\pi_h}\langle Q_h^*(x, \cdot),\pi_h(\cdot\given x)\rangle_{\cA},\quad V_h^*(x)= \langle Q_h^*(x, \cdot),\pi_h^*(\cdot\given x)\rangle_{\cA},
$$
where the maximum is taken over all functions mapping from $\cS$ to distributions over $\cA$. 
We aim to learn a policy that maximizes the expected cumulative reward. Correspondingly, we define the performance metric as   
\begin{equation}
\text{SubOpt}(\pi;x) = V_1^{\pi^*}(x) - V_1^{\pi}(x),
\label{eq:def_regret}
\end{equation}
which is the suboptimality of the policy $\pi$ given the initial state $s_1 = x$.

\subsection{Offline Data Collecting Process}
We consider the offline setting, that is, a learner only has access to a dataset $\cD$ consisting of $K$ trajectories $\{(x_h^\tau,a_h^\tau,r_h^\tau) \}_{\tau , h= 1}^{K, H}$, which is collected  a priori  by an experimenter. In other words, at each step $h\in[H]$ of each trajectory $\tau \in [K]$, the experimenter takes the action $a_h^\tau$ at the state $x_h^\tau$, receives the reward $r_h^\tau = r_h(x_h^\tau, a_h^\tau)$, and observes the next state $x_{h+1}^\tau \sim \cP_h(\cdot \given s_h=x_h^\tau ,a_h=a_h^\tau)$. Here $a_h^\tau$ can be arbitrarily chosen, while $r_h$ and $\cP_h$ are the reward function and transition kernel of an underlying MDP. We define the compliance of such a dataset with the underlying MDP as follows.

\begin{definition}[Compliance]\label{def:comp}
For a dataset $\cD=\{(x_h^\tau,a_h^\tau,r_h^\tau)\}_{\tau, h=1}^{K, H}$, let $\PP_{\cD}$ be the joint distribution of the data collecting process. We say $\cD$ is compliant with an underlying MDP $(\cS,\cA,H,\cP,r)$ if 
\#
&\PP_{\cD}\big(r_h^\tau = r' ,x_{h+1}^\tau = x' \biggiven \{(x_h^j,a_h^j)\}_{j=1}^\tau ,\{(r_h^j,x_{h+1}^j)\}_{j=1}^{\tau -1 } \big)\notag\\
&\quad= \PP\bigl(r_h(s_h, a_h) = r' , s_{h+1} = x'\biggiven s_h=x_h^\tau,a_h=a_h^\tau\bigr) %\cdot \cP_h(s_{h+1} = x' \given s_h=x_h^\tau,a_h=a_h^\tau)
\label{eq:assump_data_generate}
\#
for all $r'\in [0,1]$ and $x'\in \cS$ at each step $h\in[H]$ of each trajectory $\tau\in[K]$. Here $\PP$ on   the right-hand side of Equation \eqref{eq:assump_data_generate} is taken with respect to the underlying MDP.  
\label{def:compliant}
\end{definition}

%Equation \eqref{eq:assump_data_generate} implies the following two conditions on $\PP_\cD$ hold simultaneously: (i) at each step $h\in[H]$, $\{(r_h^\tau,x_{h+1}^\tau)\}_{\tau=1}^K$ are independent across each trajectory $\tau\in[K]$ conditioning on $\{(x_h^\tau,a_h^\tau)\}_{\tau=1}^K$, and (ii) at each step $h\in[H]$ of each trajectory $\tau\in[K]$, $(r_h^\tau,x_{h+1}^\tau)$ is generated by the reward function and transition kernel of the underlying MDP conditioning on $(x_h^\tau,a_h^\tau)$. See Appendix \ref{sec:appendix_data_collect} for a detailed discussion. 

Equation \eqref{eq:assump_data_generate} implies the following two conditions on $\PP_\cD$ hold simultaneously: (i) at each step $h\in[H]$ of each trajectory $\tau\in[K]$, 
$ (r_h^\tau,x_{h+1}^\tau) $ only
depends on $ \{(x_h^j,a_h^j)\}_{j=1}^\tau \cup\{(r_h^j,x_{h+1}^j)\}_{j=1}^{\tau -1 }$ via $(x_h^\tau, a_h^\tau)$, and (ii) conditioning on $(x_h^\tau,a_h^\tau)$, $(r_h^\tau,x_{h+1}^\tau)$ is generated by the reward function and transition kernel of the underlying MDP. 
Intuitively, (i) ensures $\cD$ possesses the Markov property. Specifically, (i) allows the $K$ trajectories to  be interdependent, that is, at each step $h\in[H]$, $\{ (x_h^\tau, a_h^\tau, r_h^\tau, x_{h+1}^\tau )\}_{\tau = 1}^K$ are interdependent across each trajectory $\tau\in[K]$. Meanwhile, (i) requires the randomness of $\{ (x_h^j, a_h^j, r_h^j,  x_{h+1}^j)\}_{j=1}^{\tau -1 }$ to be fully captured by $(x_h^\tau, a_h^\tau)$ when we examine the randomness of $(r_h^\tau,x_{h+1}^\tau)$.

%But for each $\tau \in [K]$, when focusing on the $h$-th transition in the $\tau$-th trajectory, the effect of the transitions in previous trajectories, $\{ (x_h^j, a_h^j, r_h^j,  x_{h+1}^j)\}_{j< \tau}$,  is fully captured by $(x_h^\tau, a_h^\tau)$. 
%Such a condition is naturally satisfied  when the $K$ trajectories are independently collected. 
%More importantly, the notion of compliance further allow $\cD$ to be adaptively acquired by an iterative method which update is behavior  policy based on the collected data, e.g., the optimistic algorithms designed for online RL \citep{}.

%Condition \eqref{eq:assump_data_generate} is equivalent to the condition that transitions of all trajectories at step $h$ given current states and actions are conditionally independent and follow the true MDP. For a detailed discussion, see Appendix \ref{sec:appendix_data_collect}.

\begin{assumption}[Data Collecting Process]
The dataset $\cD$ that the learner has access to is compliant with the underlying MDP $(\cS,\cA,H,\cP,r)$.
\label{assump:data_generate}
\end{assumption}

As a special case, Assumption \ref{assump:data_generate} holds if the experimenter follows a fixed behavior policy. More generally, Assumption \ref{assump:data_generate} allows $a_h^\tau$ to be arbitrarily chosen, even in an adaptive or adversarial manner, in the sense that the experimenter does not necessarily follow a fixed behavior policy. In particular, $a_h^\tau$ can be interdependent across each trajectory $\tau\in[K]$. For example, the experimenter can sequentially improve the behavior policy using any algorithm for online RL. Furthermore, Assumption \ref{assump:data_generate} does not require the data collecting process to well explore the state space and action space. 

%Such a requirement is often associated with an eigenvalue condition on the covariance matrix in the literature \citep{}.

%Equation \eqref{eq:assump_data_generate} implies the conditional independence of the reward and state evolution across each trajectory $\tau\in[K]$, that is, $\{(r_h^\tau,x_{h+1}^\tau)\}_{\tau=1}^K$ are jointly independent conditional on last state-actions $\{x_h^\tau,a_h^\tau\}_{\tau=1}^K$, with each next state following the MDP transition kernel $\cP_h(\cdot\given \cdot,\cdot)$. 

%Conversely, the conditional independence implies $$\PP_{\cD}\big(x_{h+1}^\tau =x',r_h^\tau =r \given \{x_h^k,a_h^k\}_{k=1}^K,\{r_h^\tau,x_{h+1}^k\}_{k\in[K],k\neq \tau}\big) = \PP_{\cD}\big(x_{h+1}^\tau =x',r_h^\tau =r \given \{x_h^k,a_h^k\}_{k=1}^K\big)$$ for all $\tau\in [K]$, which is a bit weaker than Assumption \ref{assump:data_generate}. The difference is that the reward $r_h$ and evolution of $x_{h+1}^\tau$ given all last state-action pairs $\{x_h^k,a_h^k\}_{k=1}^K$ should still follow MDP transition based on its own history. This enforces a natural condition that the dynamics along trajectories does not influence each other, which is also of the similar flavor to the SUTVA assumption \citep{imbens2015causal} in causal inference framework. 
%\end{flushleft}

% !TEX root = paper.tex

\section{What Causes Suboptimality?}\label{sec:demo}
%\begin{flushleft}
In this section, we decompose the suboptimality of any policy into three sources, namely the spurious correlation, intrinsic uncertainty, and optimization error. We first analyze the MDP and then specialize the general analysis to the multi-armed bandit (MAB) for illustration. 

%In this section, we discuss the suboptimality of a policy associated with value-based estimations. With a suboptimality decomposition lemma, we show that the entanglement of learned policy and model evaluation errors makes it hard to control the suboptimality of learned policy. Moreover, we consider some special cases to show how such concerns causes hazardness.

\subsection{Spurious Correlation Versus Intrinsic Uncertainty}

We consider a meta-algorithm, which constructs an  estimated Q-function $\hat{Q}_h:\cS\times \cA\to \RR$ and an estimated value function $\hat{V}_h:\cS\to \RR$ based on the dataset $\cD$. We define the model evaluation error at each step $h\in[H]$ as 
\begin{equation}
\iota_h(x,a) = (\BB_h \hat{V}_{h+1})(x,a) - \hat{Q}_h(x,a).%\quad \forall (x,a)\in \cS\times \cA,
\label{eq:def_iota}
\end{equation}
In other words, $\iota_h$ is the error that arises from estimating the Bellman operator $\BB_h$ defined in Equation \eqref{eq:def_bellman_op}, especially the transition operator $\PP_h$ therein, based on $\cD$. Note that $\iota_h$ in Equation \eqref{eq:def_iota} is defined in a pointwise manner for all $(x,a)\in \cS\times \cA$, where $\hat{V}_{h+1}$ and $\hat{Q}_h$ depend on $\cD$. The suboptimality of the policy $\hat\pi$ corresponding to $\hat{V}_{h}$ and $\hat{Q}_h$ (in the sense that $\hat{V}_h(x)=\langle \hat{Q}_h(x,\cdot),\hat\pi_h(\cdot\given x)\rangle_{\cA}$), which is defined in Equation \eqref{eq:def_regret}, admits the following decomposition.

%To begin with, consider a general value-based learning setting, where at each step $h$, we learn an approximate Q-function $\hat{Q}_h(\cdot,\cdot):\cS\times \cA\to \RR$ and an approximate value function $\hat{V}_h(\cdot):\cS\to \RR$ from the given dataset. Often, $\hat{Q}_h(\cdot,\cdot)$ is an empirical Bellman evaluation of $\hat{V}_{h+1}(\cdot)$, and we define the model evaluation error as
%\begin{equation}
%\iota_h(x,a) = \BB_h \hat{V}_{h+1}(x,a) - \hat{Q}_h(x,a),\quad \forall (x,a)\in \cS\times \cA,
%\label{eq:def_iota}
%\end{equation}
%which is the error in estimating $\BB_h\hat{V}_{h+1}$ in Bellman equation based on offline data. Note that they are defined point-wise for all $(x,a)\in \cS\times \cA$, while the functions themselves may depend on the dataset.

\begin{lemma}[Decomposition of Suboptimality]\label{lem:dec}
Let $\hat\pi=\{\hat\pi_h\}_{h=1}^H$ be the policy such that $\hat{V}_h(x)=\langle \hat{Q}_h(x,\cdot),\hat\pi_h(\cdot\given x)\rangle_{\cA}$. For any $\hat\pi$ and $x \in \cS$, we have 
\#\normalfont
\text{SubOpt}(\hat\pi;x)
&= \underbrace{-\sum_{h=1}^H \EE_{\hat\pi}\big[ \iota_h(s_h,a_h) \biggiven s_1=x\big]}_{\displaystyle\normalfont\textrm{(i): Spurious Correlation}} + \underbrace{\sum_{h=1}^H \EE_{\pi^*}\big[ \iota_h(s_h,a_h) \biggiven s_1=x \big]}_{\displaystyle \normalfont \textrm{(ii): Intrinsic Uncertainty}}\notag\\
&\qquad + \underbrace{\sum_{h=1}^H \EE_{\pi^*}\big[  \langle \hat{Q}_h(s_h,\cdot), \pi_h^*(\cdot\given s_h) - \hat\pi_h(\cdot\given s_h)\rangle_{\cA} \biggiven s_1=x\big]}_{\displaystyle \normalfont\textrm{(iii): Optimization Error}}.
\label{eq:reg_decomp}			
\#
Here $\EE_{\hat\pi}$ and $\EE_{\pi^*}$ are taken with respect to the trajectories induced by $\hat\pi$ and $\pi^*$ in the underlying MDP given the fixed functions $\hat{V}_{h+1}$ and $\hat{Q}_h$, which determine $\iota_h$. 
\label{lem:reg_decomp}
\end{lemma}

\begin{proof}[Proof of Lemma \ref{lem:reg_decomp}]
See Section \ref{sec:appendix_proof_decomp} for a detailed proof.
\end{proof}

In Equation \eqref{eq:reg_decomp}, term (i) is more challenging to control, as $\hat\pi$ and $\iota_h$ simultaneously depend on $\cD$ and hence spuriously correlate with each other. In Section \ref{sec:mab}, we show  such a spurious correlation can ``mislead'' $\hat{\pi}$, which incurs a significant suboptimality, even in the MAB. Specifically, assuming hypothetically $\hat\pi$ and $\iota_h$ are independent, term (i) is mean zero with respect to $\PP_\cD$ as long as $\iota_h$ is mean zero for all $(x,a)\in \cS\times \cA$, which only necessitates an unbiased estimator of $\BB_h$ in Equation \eqref{eq:def_iota}, e.g., the sample average estimator in the MAB. However, as $\hat\pi$ and $\iota_h$ are spuriously correlated, term (i) can be rather large in expectation. 

In contrast, term (ii) is less challenging to control, as $\pi^*$ is intrinsic to the underlying MDP and hence does not depend on $\cD$, especially the corresponding $\iota_h$, which quantifies the uncertainty that arises from approximating $\BB_h \hat{V}_{h+1}$. In Section \ref{sec:lower}, we show such an intrinsic uncertainty is impossible to eliminate, as it arises from the information-theoretic lower bound. In addition, as the optimization error, term (iii) is nonpositive as long as $\hat\pi$ is greedy with respect to $\hat{Q}_h$, that is, $\hat{\pi}_h (\cdot \given x)=\argmax_{\pi_h}\langle \hat{Q}_h(x, \cdot),\pi_h(\cdot\given x)\rangle_{\cA}$ (although Equation \eqref{eq:reg_decomp} holds for any $\hat{\pi}$ such that $\hat{V}_h(x)=\langle \hat{Q}_h(x,\cdot),\hat\pi_h(\cdot\given x)\rangle_{\cA}$). 

%\begin{remark}
%When $\hat\pi$ is the learned policy and $\hat{Q}_h(\cdot,\cdot)$, $\hat{V}_h(\cdot)$ are the estimated value functions that $\hat\pi$ is based on, 
%Term (i) in Equation \eqref{eq:reg_decomp} corresponds to the model evaluation error, while the expectation is taken with respect to $(s_h,a_h)\sim \PP_{\hat\pi}$, the distribution induced by performing the learned policy $\hat\pi$ on true MDP. Term (ii) is also the point-wise model evaluation error, while term (iii) corresponds to the difference in approximate value of learned policy and the true optimal policy, with expectation both taken under optimal policy $\pi^*$. 
%\end{remark}

%In Equation \eqref{eq:reg_decomp}, term (i) is particularly hard to control, since both $\hat\pi$ and $\iota_h(\cdot,\cdot)$ depend on the dataset thus heavily entangled. Such issue is often implicitly addressed in literature by imposing some unrealistic and hard-to-verify conditions on the learned policy. However, we are to show that the term can be eliminated with a proper incorporation of uncertainty quantification.

\subsection{Illustration via a  Special Case: MAB}\label{sec:mab}

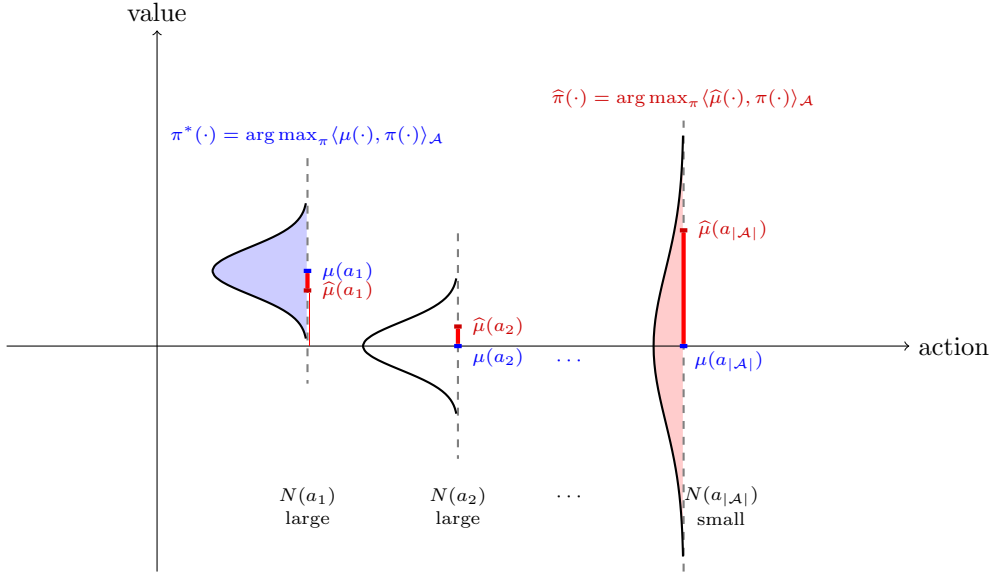
\begin{figure}
\centering
\begin{tikzpicture}
\fill [blue!20, domain=0.1:1.9, variable=\y]
      (0.99, 0.1)
      -- plot ({-exp(-(\y-1)*(\y-1)*5) * sqrt(10/(2*pi)) + 1},{\y})
      -- (0.99, 1.9)
      -- cycle;

\fill [red!20, domain=-2.8:2.8, variable=\y]
      (5.99, -2.8)
      -- plot ({-exp(-(\y)*(\y)/2) * sqrt(1/(2*pi))+ 6}, {\y})
      -- (5.99, 2.8)
      -- cycle;

\draw[->] (-3, 0) -- (9, 0) node[right] {action};
\draw[->] (-1, -3) -- (-1, 4.2) node[above] {value};
\draw [gray, dashed, thick] (1,2.5) --  (1,-0.5) ;
%\draw (0.75,2.7) node[right] {\small $a_1$};
\draw[domain=0.1:1.9, smooth, variable=\y, black, thick]  plot ({-exp(-(\y-1)*(\y-1)*5) * sqrt(10/(2*pi)) + 1}, {\y}); 
\draw [gray, dashed, thick] (3,1.5) -- (3,-1.5);
\draw[domain=-0.9:0.9, smooth, variable=\y, black, thick]  plot ({-exp(-(\y)*(\y)*5)* sqrt(10/(2*pi)) + 3}, {\y});
%\draw (2.75,1.7) node[right] {\small $a_2$};
\draw [gray, dashed, thick] (6,-3) -- (6,3);
\draw[domain=-2.8:2.8, smooth, variable=\y, black, thick]  plot ({-exp(-(\y)*(\y)/2) * sqrt(1/(2*pi))+ 6}, {\y});
%\draw (5.75,3.2) node[right] {\small $a_{|\cA|}$};

\draw [blue, ultra thick] (0.95,1) -- (1.05,1) node[right]{\scriptsize $\mu(a_1)$};
\draw [red!80!black, ultra thick] (0.95,0.74) -- (1.05,0.74) node[right]{\scriptsize $\hat\mu(a_1)$};
\draw [blue, ultra thick] (2.95,0) -- (3.05,0) ;
\draw [blue, ultra thick] (3.05,-0.15) node[right]{\scriptsize $\mu(a_2)$};
\draw [red!80!black, ultra thick] (2.95,0.26) -- (3.05,0.26) node[right]{\scriptsize $\hat\mu(a_2)$};
\draw [blue, ultra thick] (5.95,0) -- (6.05,0) ;
\draw [blue, ultra thick] (6,-0.2) node[right]{\scriptsize $\mu(a_{|\cA|})$};
\draw [red!80!black, ultra thick] (5.95,1.54) -- (6.05,1.54) node[right]{\scriptsize $\hat\mu(a_{|\cA|})$};

\draw [red, ultra thick] (1,0.77) -- (1,0.97);
\draw [red, ultra thick] (3,0.03) -- (3,0.23);
\draw [red, ultra thick] (6,0.03) -- (6,1.51);

\fill [red]
      (1.02, 0.01) -- (1.03,0.01) -- (1.03, 0.71) -- (1.02, 0.71)
      -- cycle;

\node[text=black] at (1,-2) {\scriptsize $N(a_1)$};
\node[text=black] at (1,-2.3) {\scriptsize large};

\node[text=black] at (3,-2) {\scriptsize $N(a_2)$};
\node[text=black] at (3,-2.3) {\scriptsize large};

\node[text=black] at (6.5,-2) {\scriptsize $N(a_{|\cA|})$};
\node[text=black] at (6.5,-2.3) {\scriptsize small};

\node[text=blue] at (4.5,-0.2) {\scriptsize $\cdots$};
\node[text=black] at (4.5,-2) {\scriptsize $\cdots$};

\node[text=red!80!black] at (6,3.3) {\scriptsize $\hat\pi(\cdot) = \argmax_{\pi}\langle \hat\mu(\cdot),\pi(\cdot)\rangle_{\cA}$};
\node[text=blue] at (1,2.8) {\scriptsize $\pi^*(\cdot) = \argmax_{\pi}\langle \mu(\cdot),\pi(\cdot)\rangle_{\cA}$};
\end{tikzpicture}
\caption{An illustration of the spurious correlation in the MAB, a special case of the MDP, where $\cS$ is a singleton, $\cA$ is discrete, and $H = 1$. Here $\mu(a)$ is the expected reward of each action $a \in \cA$ and $\hat{\mu}(a)$ is its sample average estimator, which follows the Gaussian distribution in Equation \eqref{eq:whatmu}. Correspondingly, $\iota(a) = \mu(a) - \hat{\mu}(a)$ is the model evaluation error. As the greedy policy with respect to $\hat{\mu}$, $\hat{\pi}$ wrongly takes the action $a_{|\cA|} = \argmax_{a \in \cA} \hat{\mu}(a)$ with probability one only because $N(a_{|\cA|})$ is relatively small, which allows $\hat{\mu}(a_{|\cA|})$ to be rather large, even though $\mu(a_{|\cA|}) = 0$. Due to such a spurious correlation, $\hat{\pi}$ incurs a significant suboptimality in comparison with $\pi^*$, which takes the action $a_1 = \argmax_{a \in \cA} \mu(a)$ with probability one.
}
\label{fig:w01}
\end{figure}

We consider the MAB, a special case of the MDP, where $\cS$ is a singleton, $\cA$ is discrete, and $H = 1$. To simplify the subsequent discussion, we assume without loss of generality
\$
r(a) = \mu(a) + \epsilon, \quad \text{where~~} \epsilon \sim \text{N}(0, 1).
\$
Here $\mu(a)$ is the expected reward of each action $a \in \cA$ and $\epsilon$ is independently drawn. For notational simplicity, we omit the dependency on $h \in [H]$ and $x\in \cS$, as $H = 1$ and $\cS$ is a singleton. Based on the dataset $\cD = \{(a^\tau, r^\tau)\}_{\tau=1}^K$, where $ r^\tau = r(a^\tau)$, we consider the sample average estimator
\$
\hat{\mu}(a) = \frac{1}{N(a)}\sum_{\tau=1}^K r^\tau \cdot \ind {\{a^\tau=a\}},\quad \text{where~~} N(a) = \sum_{\tau=1}^K  \ind {\{a^\tau=a\}}.
\$
Note that $\hat{\mu}$ serves as the estimated Q-function. Under Assumption \ref{assump:data_generate}, we have 
\#\label{eq:whatmu}
\hat{\mu}(a) \sim \text{N}\bigl(\mu(a), 1/N(a)\bigr).
\#
In particular, $\{\hat{\mu}(a)\}_{a \in \cA}$ are independent across each action $a \in \cA$ conditioning on $\{a^\tau\}_{\tau=1}^K$. 
%See Appendix \ref{sec:appendix_data_collect} for a detailed discussion. 
We consider the policy 
\#\label{eq:w01}
\hat{\pi} (\cdot) = \argmax_{\pi} \langle \hat{\mu}(\cdot),\pi(\cdot)\rangle_{\cA},
\#
which is greedy with respect to $\hat{\mu}$, as it takes the action $\argmax_{a \in \cA} \hat{\mu}(a)$ with probability one. 

%We now consider the tabular MDP or multi-armed bandit case to provide more intuitions, where the model evaluation has a relatively explicit form.

By Equation \eqref{eq:def_iota}, Lemma \ref{lem:dec}, and Equation \eqref{eq:w01}, we have 
 \$
 \text{SubOpt}(\hat\pi;x) \leq \underbrace{- \EE_{\hat\pi}\big[ \iota(a) \big]}_{\displaystyle\text{(i)}} +  \underbrace{\EE_{\pi^*}\big[ \iota(a) \big]}_{\displaystyle \text{(ii)}}, \quad \text{where~~} \iota(a) = \mu(a) - \hat{\mu}(a).
 \$
 Note that $\iota(a)$ is mean zero with respect to $\PP_{\cD}$ for each action $a \in \cA$. Therefore, assuming hypothetically $\hat\pi$ and $\iota$ are independent, term (i) is mean zero with respect to $\PP_\cD$. Meanwhile, as $\pi^*$ and $\iota$ are independent, term (ii) is also mean zero with respect to $\PP_\cD$. However, as $\hat\pi$ and $\iota$ are spuriously correlated due to their dependency on $\cD$, term (i) can be rather large in expectation. See Figure \ref{fig:w01} for an illustration. Specifically, we have 
 \#\label{eq:w022}
 - \EE_{\hat\pi}\bigl[ \iota(a) \bigr] = \langle\hat{\mu}(\cdot) - \mu(\cdot), \hat{\pi}(\cdot) \rangle_{\cA} = \bigl\langle\hat{\mu}(\cdot) - \mu(\cdot), \argmax_{\pi} \langle \hat{\mu}(\cdot),\pi(\cdot)\rangle_{\cA} \bigr\rangle_{\cA}.
 \#
For example, assuming $\mu(a) = 0$ for each action $a \in \cA$, term (i) is the maximum of $|\cA|$ Gaussians $\{\text{N}(0, 1/N(a))\}_{a \in \cA}$, which can be rather large in expectation, especially when $N(a^\sharp)$ is relatively small for a certain action $a^\sharp \in \cA$, e.g., $N(a^\sharp) = 1$. More generally, it is quite possible that $\hat\pi$ takes a certain action $a^\natural \in \cA$ with probability one only because $N(a^\natural)$ is relatively small, which allows $\hat{\mu}(a^\natural)$ to be rather large, even when $\mu(a^\natural)$ is relatively small. Due to such a spurious correlation, $\langle\hat{\mu}(\cdot) - \mu(\cdot), \hat{\pi}(\cdot) \rangle_{\cA} = \hat{\mu}(a^\natural) - \mu(a^\natural)$ in Equation \eqref{eq:w022} can be rather large in expectation, which incurs a significant suboptimality. More importantly, such an undesired situation can be quite common in practice, as $\cD$ does not necessarily have a ``uniform coverage'' over each action $a\in \cA$. In other words, $N(a^\sharp)$ is often relatively small for at least a certain action $a^\sharp \in \cA$.

 Going beyond the MAB, that is, $H \geq 1$, such a spurious correlation is further exacerbated, as it is more challenging to ensure each state $x\in \cS$ and each action $a \in \cA$ are visited sufficiently many times in $\cD$. To this end, existing literature \citep{antos2007fitted, antos2008learning, munos2008finite, farahmand2010error, farahmand2016regularized, scherrer2015approximate, liu2018breaking, nachum2019dualdice, nachum2019algaedice, chen2019information, tang2019doubly, kallus2019efficiently, kallus2020doubly, fan2020theoretical, xie2020batch, xie2020q, jiang2020minimax, uehara2020minimax, duan2020minimaxoptimal, yin2020near, qu2020finite, li2020sample,  liao2020batch, nachum2020reinforcement, yang2020off, zhang2020variational, zhang2019gendice} relies on various assumptions on the ``uniform coverage'' of $\cD$, e.g., finite concentrability coefficients and uniformly lower bounded densities of visitation measures, which however often fail to hold in practice.

\section{Pessimism is Provably Efficient}\label{sec:pess}
%\begin{flushleft}
In this section, we present the algorithm and theory. Specifically, we introduce a penalty function to develop a pessimistic value iteration algorithm (PEVI), which simply flips the sign of the bonus function for promoting exploration in online RL \citep{jaksch2010near, abbasi2011improved,russo2013eluder, osband2014model, chowdhury2017kernelized, azar2017minimax, jin2018q, jin2020provably, cai2020provably, yang2020bridging, ayoub2020model, wang2020reinforcement}. In Section \ref{sec:pes_general}, we provide a sufficient condition for eliminating the spurious correlation from the suboptimality for any general MDP. In Section \ref{sec:pes_linear}, we characterize the suboptimality for the linear MDP \citep{yang2019sample, jin2020provably} by verifying the sufficient condition in Section \ref{sec:pes_general}. In Section \ref{sec:lower}, we establish the minimax optimality of PEVI via the information-theoretic lower bound.

%{\red intuition, iota < 0, spurious correlation < 0, construct gamma}

%We present the main results of the paper in this section. In Section \ref{sec:pes_general}, we introduce a notion of uncertainty penalty function, and show how the incorporation of the penalty eliminates the undesired term in suboptimality decomposition for general episodic MDP. In Section \ref{sec:pes_linear}, we consider a more specific linear MDP setting, where a data-dependent uncertainty penalty function is explicitly constructed and a pessimistic value iteration algorithm is proposed. In Section \ref{sec:info_gain}, we review the construction of uncertainty function from an information-gain perspective. In Section \ref{sec:lower},
%we discuss on optimality properties to illustrate the merits of our method.

\subsection{Pessimistic Value Iteration: General MDP}\label{sec:pes_general}
We consider a meta-algorithm, namely PEVI, which constructs an estimated Bellman operator $\hat\BB_h$ based on the dataset $\cD$ so that $\hat\BB_h \hat{V}_{h+1}:\cS\times \cA\to \RR$ approximates $\BB_h \hat{V}_{h+1}: \cS\times \cA\to \RR$. Here $\hat{V}_{h+1}:\cS\to \RR$ is an estimated value function constructed by the meta-algorithm based on $\cD$. Note that such a construction of $\hat\BB_h$ can be implicit in the sense that the meta-algorithm only relies on $\hat\BB_h \hat{V}_{h+1}$ instead of $\hat\BB_h$ itself. We define an uncertainty quantifier with the  confidence parameter $\xi \in (0, 1)$ as follows. Recall that $\PP_{\cD}$ is the joint distribution of the data collecting process.

%Consider a value-based learning procedure with approximate value functions $\hat{V}_h(\cdot):\cS\to \RR$, and an empirical Bellman operator $\hat\BB_h$ for $h\in [H]$, so that $\hat\BB_h \hat{V}_{h+1}(\cdot,\cdot):\cS\times \cA\to \RR$ is an approximation to the true Bellman update $\BB_h \hat{V}_{h+1}(\cdot,\cdot)$ at step $h$. Define a notion of $\xi$-uncertainty function as follows.

\begin{definition}[$\xi$-Uncertainty Quantifier] 
We say $\{\Gamma_h\}_{h=1}^H\, (\Gamma_h:\cS\times\cA\to \RR)$ is a $\xi$-uncertainty quantifier with respect to $\PP_\cD$ if the event
\begin{equation}
\cE = \Big\{ \big|(\hat\BB_h\hat{V}_{h+1})(x,a) - (\BB_h\hat{V}_{h+1})(x,a)\big|\leq \Gamma_h(x,a)~ \text{for all}~(x,a)\in \cS\times \cA, h\in [H]    \Big\}
\label{eq:def_event_eval_err_general}
\end{equation}
satisfies $\PP_{\cD}(\cE)\geq 1-\xi$.
\label{def:uncertainty_quantifier}
%Here $\hat\BB_h,$ $\hat{V}_{h+1}(\cdot)$ and $\Gamma_h(\cdot,\cdot)$ all may depend on dataset $\cD$.
\end{definition}

By Equation \eqref{eq:def_event_eval_err_general}, $\Gamma_h$ quantifies the uncertainty that arises from approximating $\BB_h \hat{V}_{h+1}$ using $\hat\BB_h \hat{V}_{h+1}$, which allows us to develop the meta-algorithm (Algorithm \ref{alg:pess_greedy_general}).

%Basically, $\Gamma_h(\cdot,\cdot)$ is an uncertainty quantity that upper bounds the empirical Bellman evaluation error with a certain confidence level $1-\xi$. We propose a general algorithm based on a $\xi$-uncertainty penalty function in Algorithm \ref{alg:pess_greedy_general}.

\begin{algorithm}[H]
\caption{Pessimistic Value Iteration (PEVI): General MDP}\label{alg:pess_greedy_general}
\begin{algorithmic}[1]
\STATE Input: Dataset $\cD=\{(x_h^\tau,a_h^\tau,r_h^\tau)\}_{\tau, h=1}^{K, H}$.
\STATE Initialization: Set $\hat{V}_{H+1}(\cdot) \leftarrow 0$.
\FOR{step $h=H,H-1,\ldots,1$}
\STATE Construct $(\hat\BB_h \hat{V}_{h+1})(\cdot,\cdot)$ and $\Gamma_h(\cdot, \cdot)$ based on $\cD$.\hfill {//Estimation \& Uncertainty}
\STATE Set $\overline{Q}_h(\cdot,\cdot) \leftarrow  (\hat\BB_h \hat{V}_{h+1})(\cdot,\cdot)- \Gamma_h(\cdot,\cdot)$.\hfill {//Pessimism}\label{alg:general_Qbar}
\STATE Set $\hat{Q}_h(\cdot,\cdot) \leftarrow \min\{\overline{Q}_h(\cdot,\cdot),H-h+1\}^+$.\hfill {//Truncation}\label{alg:general_Qhat}
\STATE Set $\hat{\pi}_h (\cdot \given \cdot) \leftarrow \argmax_{\pi_h}\langle \hat{Q}_h(\cdot, \cdot),\pi_h(\cdot\given \cdot)\rangle_{\cA}$.\hfill {//Optimization}
\STATE Set $\hat{V}_h(\cdot) \leftarrow \langle \hat{Q}_h(\cdot,\cdot),\hat\pi_h(\cdot \given \cdot)\rangle_{\cA}$.\hfill {//Evaluation} \label{alg:general_Vhat}
\ENDFOR 
\STATE Output: $\pess(\cD) = \{\hat{\pi}_h\}_{h=1}^H$.
\end{algorithmic}
\end{algorithm}

The following theorem characterizes the suboptimality of Algorithm \ref{alg:pess_greedy_general}, which is defined in Equation \eqref{eq:def_regret}. 

%{\bf \color{red} Explain the Algorithm} 

\begin{theorem}[Suboptimality for General MDP]
Suppose $\{\Gamma_h\}_{h=1}^H$ in Algorithm \ref{alg:pess_greedy_general} is a $\xi$-uncertainty quantifier. Under $\cE$ defined in Equation \eqref{eq:def_event_eval_err_general}, which satisfies $\PP_{\cD}(\cE)\geq 1-\xi$, for any $x \in \cS$, $\pess (\cD)$ in  Algorithm \ref{alg:pess_greedy_general} satisfies 
\begin{equation}\normalfont
\text{SubOpt}\big(\pess (\cD);x \big) \leq 2\sum_{h=1}^H\EE_{\pi^*}\big[ \Gamma_h(s_h,a_h) \biggiven s_1=x\big].
\label{eq:regret_upper_general}
\end{equation}
%where $\EE_{\pi^*}[\cdot]$ is with respect to carrying out the optimal policy $\pi^*$ on MDP $(\cS,\cA,H,\cP,r)$, and $\Gamma_h(\cdot,\cdot)$ is treated as fixed in the expectation.
Here $\EE_{\pi^*}$ is with respect to the trajectory induced by $\pi^*$ in the underlying MDP given the fixed function $\Gamma_h$. 
\label{thm:regret_upper_bound_general}
\end{theorem}

\begin{proof}[Proof of Theorem \ref{thm:regret_upper_bound_general}]
See Section \ref{sec:upper_sketch_general} for a proof sketch.
\end{proof}

Theorem \ref{thm:regret_upper_bound_general} establishes a sufficient condition for eliminating the spurious correlation, which corresponds to term (i) in Equation \eqref{eq:reg_decomp}, from the suboptimality for any general MDP. Specifically, $-\Gamma_h$ in Algorithm \ref{alg:pess_greedy_general} serves as the penalty function, which ensures $-\iota_h$ in Equation \eqref{eq:reg_decomp} is nonpositive under $\cE$ defined in Equation \eqref{eq:def_event_eval_err_general}, that is, 
\#\label{eq:w02}
-\iota_h(x,a) &= \hat{Q}_h(x,a) - (\BB_h \hat{V}_{h+1})(x,a) \leq \overline{Q}_h(x,a) - (\BB_h \hat{V}_{h+1})(x,a) \notag \\
&=  (\hat{\BB}_h \hat{V}_{h+1})(x,a) - (\BB_h \hat{V}_{h+1})(x,a) - \Gamma_h(x,a) \leq 0.
\#
Note that Equation \eqref{eq:w02} holds in a pointwise manner for all $(x,a)\in \cS\times \cA$. In other words, as long as $\Gamma_h$ is a $\xi$-uncertainty quantifier, the suboptimality in Equation \eqref{eq:regret_upper_general} only corresponds to term (ii) in Equation \eqref{eq:reg_decomp}, which characterizes the intrinsic uncertainty. In any concrete setting, e.g., the linear MDP, it only remains to specify $\Gamma_h$ and prove it is a $\xi$-uncertainty quantifier under Assumption \ref{assump:data_generate}. In particular, we aim to find a $\xi$-uncertainty quantifier that is sufficiently small to establish an adequately tight upper bound of the suboptimality in Equation \eqref{eq:regret_upper_general}. In the sequel, we show it suffices to employ the bonus function for promoting exploration in online RL. %{\red \bf \citep{ abbasi2011improved, chowdhury2017kernelized,russo2013eluder, osband2014model, azar2017minimax, jin2018q, jin2020provably, wang2020reinforcement, yang2020bridging}}.

%the upper bound in Equation \eqref{eq:regret_upper_general} to hold. This is a quite general result for policy learning in epsodic MDP: once we find such a $\xi$-uncertainty function, which is essentially a point-wise upper bound for the empirical Bellman update error, a simple pessimistic modification as Algorithm \ref{alg:pess_greedy_general} provides a concrete guarantee for the output policy.

\subsection{Pessimistic Value Iteration: Linear MDP}\label{sec:pes_linear}
As a concrete setting, we study the instantiation of PEVI for the linear MDP. We define the linear MDP \citep{yang2019sample, jin2020provably} as follows, where the transition kernel and expected reward function are linear in a feature map. 

%In this section, we narrow our scope to the setting where the transition dynamics are linear in a feature map. With the linear assumption, we explicitly construct the uncertainty penalty function $\Gamma_h^\xi(\cdot,\cdot):\cS\times \cA\to \RR$ and propose a pessimistic least-square value iteration algorithm with a data-dependent suboptimality guarantee. To begin with, consider the following definition of linear MDP.
\begin{definition}[Linear MDP]
We say an episodic MDP $(\cS,\cA,H,\cP,r)$ is a linear MDP with a known feature map $\phi:\cS\times \cA\to \RR^d$ if there exist $d$ unknown (signed) measures ${\mu}_h=(\mu_h^{(1)},\ldots,\mu_h^{(d)})$ over $\cS$ and an unknown vector $\theta_h\in \RR^d$ such that
\#\label{eq:w07}
\cP_h(x'\given x,a) = \langle \phi(x,a),\mu_h(x')\rangle,\quad \EE\bigl[r_h(s_h, a_h) \biggiven s_h=x,a_h=a\bigr] = \langle \phi(x,a),\theta_h\rangle
\#
for all $(x,a,x')\in \cS\times \cA\times \cS$ at each step $h\in[H]$. Here we assume $\|\phi(x,a)\|\leq 1$ for all $(x,a)\in \cS\times \cA$ and $\max\{\|\mu_h(\cS) \| ,\|\theta_h\|\}\leq \sqrt{d}$ at each step $h\in[H]$, where with an abuse of notation, we define  $\| \mu_h (\cS) \| = \int_{\cS } \| \mu_h (x) \| \,\ud x$.
\label{assump:linear_mdp}
\end{definition}

We specialize the meta-algorithm (Algorithm \ref{alg:pess_greedy_general}) by constructing $\hat\BB_h\hat{V}_{h+1}$, $\Gamma_h$, and $\hat{V}_{h}$ based on $\cD$, which leads to the algorithm for the linear MDP (Algorithm \ref{alg:pess_greedy}). Specifically, we construct $\hat\BB_h\hat{V}_{h+1}$ based on $\cD$ as follows. Recall that $\hat\BB_h\hat{V}_{h+1}$ approximates $\BB_h\hat{V}_{h+1}$, where $\BB_h$ is the Bellman operator defined in Equation \eqref{eq:def_bellman_op}, and $\cD=\{(x_h^\tau,a_h^\tau,r_h^\tau)\}_{\tau, h=1}^{K, H}$ is the dataset. We define the empirical mean squared Bellman error (MSBE) as
\begin{equation*}
M_h(w) = \sum_{\tau=1}^K \bigl(r_h^\tau + \hat{V}_{h+1}(x_{h+1}^\tau) - \phi (x_h^\tau,a_h^\tau)^\top w\bigr)^2
\end{equation*}
at each step $h \in [H]$. Correspondingly, we set
\#\label{eq:wlin}
(\hat\BB_h\hat{V}_{h+1})(x, a) = \phi(x, a)^\top \hat{w}_h, \quad \text{where~~} \hat{w}_h =  \argmin_{w\in \RR^d} M_h(w) + \lambda \cdot \|w\|_2^2
\#
at each step $h\in[H]$. Here $\lambda>0$ is the regularization parameter. Note that $\hat{w}_h$ has the closed form
\#\label{eq:w18}
&\hat{w}_h =  \Lambda_h ^{-1} \Big( \sum_{\tau=1}^{K} \phi(x_h^\tau,a_h^\tau) \cdot \bigl(r_h^\tau + \hat{V}_{h+1}(x_{h+1}^\tau)\bigr) \Bigr ) , \notag\\
&\text{where~~} \Lambda_h = \sum_{\tau=1}^K \phi(x_h^\tau,a_h^\tau)  \phi(x_h^\tau,a_h^\tau) ^\top + \lambda\cdot I. 
\#
Meanwhile, we construct $\Gamma_h$ based on $\cD$ as 
\#\label{eq:w05}
\Gamma_h(x, a) = \beta\cdot \big( \phi(x, a)^\top  \Lambda_h ^{-1} \phi(x, a)  \big)^{1/2}
\#
at each step $h\in[H]$. Here $\beta>0$ is the scaling parameter. In addition, we construct $\hat{V}_h$ based on $\cD$ as
\$
\hat{Q}_h(x,a) &= \min\{ \overline{Q}_h(x,a), H-h+1\}^+, \quad \text{where~~} \overline{Q}_h(x,a) = (\hat\BB_h \hat{V}_{h+1})(x,a)- \Gamma_h(x,a),\\
\hat{V}_h(x) &= \langle \hat{Q}_h(x,\cdot),\hat\pi_h(\cdot \given x)\rangle_{\cA},\quad \text{where~~} \hat{\pi}_h (\cdot \given x) = \argmax_{\pi_h}\langle \hat{Q}_h(x, \cdot),\pi_h(\cdot\given x)\rangle_{\cA}.
\$

\begin{algorithm}[H]
\caption{Pessimistic Value Iteration (PEVI): Linear MDP}\label{alg:pess_greedy}
\begin{algorithmic}[1]
\STATE Input: Dataset $\cD=\{(x_h^\tau,a_h^\tau,r_h^\tau)\}_{\tau, h=1}^{K, H}$.
\STATE Initialization: Set $\hat{V}_{H+1}(\cdot) \leftarrow 0$.
\FOR{step $h=H,H-1,\ldots,1$}
\STATE Set $\Lambda_h \leftarrow \sum_{\tau=1}^K \phi(x_h^\tau,a_h^\tau)  \phi(x_h^\tau,a_h^\tau) ^\top + \lambda\cdot I$. %\hfill  {//Estimation}
\STATE Set $\hat{w}_h\leftarrow  \Lambda_h ^{-1}( \sum_{\tau=1}^{K} \phi(x_h^\tau,a_h^\tau) \cdot (r_h^\tau + \hat{V}_{h+1}(x_{h+1}^\tau)) ) $. \hfill  {//Estimation}
\STATE Set $\Gamma_h(\cdot,\cdot) \leftarrow \beta\cdot ( \phi(\cdot,\cdot)^\top  \Lambda_h ^{-1} \phi(\cdot,\cdot) )^{1/2}$. \hfill  {//Uncertainty}
\STATE Set $\overline{Q}_h(\cdot,\cdot) \leftarrow  \phi(\cdot,\cdot)^\top \hat{w}_h - \Gamma_h(\cdot,\cdot)$. \hfill  {//Pessimism} 
\STATE Set $\hat{Q}_h(\cdot,\cdot) \leftarrow \min\{\overline{Q}_h(\cdot,\cdot),H-h+1\}^+$. \hfill  {//Truncation}
\STATE Set $\hat{\pi}_h (\cdot \given \cdot) \leftarrow \argmax_{\pi_h}\langle \hat{Q}_h(\cdot, \cdot),\pi_h(\cdot\given \cdot)\rangle_{\cA}$.\hfill {//Optimization}
\STATE Set $\hat{V}_h(\cdot) \leftarrow \langle \hat{Q}_h(\cdot,\cdot),\hat\pi_h(\cdot \given \cdot)\rangle_{\cA}$.\hfill {//Evaluation} \label{alg:linear_Vhat}
\ENDFOR 
\STATE Output: $\pess(\cD) = \{\hat{\pi}_h\}_{h=1}^H$.
\end{algorithmic}
\end{algorithm}

The following theorem characterizes the suboptimality of Algorithm \ref{alg:pess_greedy}, which is defined in Equation \eqref{eq:def_regret}. 

\begin{theorem}[Suboptimality for Linear MDP]
Suppose Assumption \ref{assump:data_generate} holds and the underlying MDP is a linear MDP. In Algorithm \ref{alg:pess_greedy}, we set\footnote{As a side note, 
the factor $d$ in $B$ can be improved with a sample splitting trick; 
we apply this trick 
and defer the corresponding discussion to the kernel setting in Section~\ref{sec:rkhs}.}
\$
\lambda=1,\quad \beta = c\cdot dH\sqrt{\zeta}, \quad \text{where~~}\zeta= \log(2dHK/\xi).
\$
Here $c>0$ is an absolute constant and $\xi \in (0,1)$ is the confidence parameter. The following statements hold: (i) $\{\Gamma_h\}_{h=1}^H$ in Algorithm \ref{alg:pess_greedy}, which is specified in Equation \eqref{eq:w05}, is a $\xi$-uncertainty quantifier, and hence (ii) under $\cE$ defined in Equation \eqref{eq:def_event_eval_err_general}, which satisfies $\PP_{\cD}(\cE)\geq 1-\xi$, for any $x \in \cS$, $\pess (\cD)$ in  Algorithm \ref{alg:pess_greedy} satisfies
\begin{equation}\normalfont
\textrm{SubOpt}\big(\linpess(\cD);x \big) \leq 2 \beta \sum_{h=1}^H\EE_{\pi^*}\Bigl[ \bigl(\phi(s_h,a_h)^\top \Lambda_h^{-1}\phi(s_h,a_h)\bigr)^{1/2} \Biggiven s_1=x\Bigr].
\label{eq:regret_upper_linear}
\end{equation}
Here $\EE_{\pi^*}$ is with respect to the trajectory induced by $\pi^*$ in the underlying MDP given the fixed matrix $\Lambda_h$. 
\label{thm:regret_upper_linear}
\end{theorem}

\begin{proof}[Proof of Theorem \ref{thm:regret_upper_linear}]
See Section \ref{sec:upper_sketch_linear} for a proof sketch.
\end{proof}

We highlight the following aspects of Theorem \ref{thm:regret_upper_linear}:
\vskip4pt
\noindent{\bf ``Assumption-Free'' Guarantee:} Theorem \ref{thm:regret_upper_linear} only relies on the compliance of $\cD$ with the linear MDP. In comparison with existing literature \citep{antos2007fitted, antos2008learning, munos2008finite, farahmand2010error, farahmand2016regularized, scherrer2015approximate, liu2018breaking, nachum2019dualdice, nachum2019algaedice, chen2019information, tang2019doubly, kallus2019efficiently, kallus2020doubly, fan2020theoretical, xie2020batch, xie2020q, jiang2020minimax, uehara2020minimax, duan2020minimaxoptimal, yin2020near, qu2020finite, li2020sample,  liao2020batch, nachum2020reinforcement, yang2020off, zhang2020variational, zhang2019gendice}, we require no assumptions on the ``uniform coverage'' of $\cD$, e.g., finite concentrability coefficients and uniformly lower bounded densities of visitation measures, which often fail to hold in practice. Meanwhile, we impose no restrictions on the affinity between $\linpess(\cD)$ and a fixed behavior policy that induces $\cD$, which is often employed as a regularizer (or equivalently, a constraint) in existing literature \citep{fujimoto2019off, laroche2019safe, jaques2019way, wu2019behavior, kumar2019stabilizing, wang2020critic, siegel2020keep, nair2020accelerating, liu2020provably}. %{\red (Empirical Work?)}.

%fujimoto2019off, kumar2019stabilizing, wu2019behavior, laroche2019safe, agarwal2020optimistic, wang2020critic, siegel2020keep, nair2020accelerating

\vskip4pt
\noindent{\bf Intrinsic Uncertainty Versus Spurious Correlation:} The suboptimality in Equation \eqref{eq:regret_upper_linear} only corresponds to term (ii) in Equation \eqref{eq:reg_decomp}, which characterizes the intrinsic uncertainty. Note that $\Lambda_h$ depends on $\cD$ but acts as a fixed matrix in the expectation, that is, $\EE_{\pi^*}$ is only taken with respect to $(s_h, a_h)$, which lies on the trajectory induced by $\pi^*$. In other words, as $\pi^*$ is intrinsic to the underlying MDP and hence does not depend on $\cD$, the suboptimality in Equation \eqref{eq:regret_upper_linear} does not suffer from the spurious correlation, that is, term (i) in Equation \eqref{eq:reg_decomp}, which arises from the dependency of $\hat{\pi} = \pess(\cD)$ on $\cD$.

%\begin{flushleft}
The following corollary proves as long as the trajectory induced by $\pi^*$ is ``covered'' by $\cD$ sufficiently well, the suboptimality of Algorithm \ref{alg:pess_greedy} decays at a $K^{-1/2}$ rate.

\begin{corollary}[Sufficient ``Coverage'']\label{cor:well_explore_optimal}
Suppose there exists an absolute constant $c^\dagger >0$ such that the event% for all $h\in [H]$, we have 
\#\label{eq:event_opt_explore}
\cE^\dagger = \Big\{ \Lambda_h \geq I + c^\dagger \cdot K \cdot \EE_{\pi^*}\bigl[\phi(s_h,a_h)\phi(s_h,a_h)^\top\biggiven s_1=x\bigr]~ \text{for all }x\in \cS, h\in [H] \Big\}
\#
satisfies $\PP_{\cD}(\cE^\dagger)\geq 1-\xi/2$.
%for all $h\in [H]$ and all $x\in \cS$ with probability at least $1-\xi/2$, which is taken with respect to $\PP_{\cD}$.
Here $\Lambda_h$ is defined in Equation \eqref{eq:w18} and $\EE_{\pi^*}$ is taken with respect to the trajectory induced by $\pi^*$ in the underlying MDP. In Algorithm \ref{alg:pess_greedy}, we set 
\$
\lambda=1,\quad \beta = c\cdot dH\sqrt{\zeta}, \quad \text{where~~}\zeta= \log(4dHK/\xi).
\$
Here $c>0$ is an absolute constant and $\xi \in (0,1)$ is the confidence parameter. For $\linpess(\cD)$ in Algorithm \ref{alg:pess_greedy}, the event
\#\label{eq:event_opt_explore_d}\normalfont
\cE' = \Big\{ \textrm{SubOpt}\big(\linpess(\cD);x \big) %\leq 2 \beta \sum_{h=1}^H \sqrt{\sum_{j=1}^d \frac{\lambda_{h,j}(x)}{1+c_0 \cdot \lambda_{h,j}(x) \cdot K}}
\leq c' \cdot d^{3/2} H^2 K^{-1/2} \sqrt{\zeta}~\text{for all }x\in \cS  \Big\}
\#
satisfies $\PP_{\cD}(\cE')\geq 1-\xi$, where $c'>0$ is an absolute constant that only depends on $c^\dagger$ and $c$.
%\#\label{eq:optimal_explore_general}
%\text{SubOpt}\big(\linpess(\cD);x \big) \leq 2 \beta \sum_{h=1}^H \sqrt{\sum_{j=1}^d \frac{\lambda_{h,j}(x)}{1+c_0 \cdot \lambda_{h,j}(x) \cdot K}}\leq c' \beta \cdot d^{1/2} H K^{-1/2}
%\#
%with probability at least $1-\xi$, which is with respect to $\PP_{\cD}$. Here  
In particular, if $\rank(\Sigma_h(x))\leq r$ for all $x\in \cS$ at each step $h\in [H]$, where 
\$
\Sigma_h(x)=\EE_{\pi^*}\bigl[\phi(s_h,a_h)\phi(s_h,a_h)^\top\biggiven s_1=x\bigr],
\$ 
for $\linpess(\cD)$ in Algorithm \ref{alg:pess_greedy}, the event
\#\label{eq:event_opt_explore_r}\normalfont
\cE'' = \Big\{ \textrm{SubOpt}\big(\linpess(\cD);x \big) %\leq 2 \beta \sum_{h=1}^H \sqrt{\sum_{j=1}^d \frac{\lambda_{h,j}(x)}{1+c_0 \cdot \lambda_{h,j}(x) \cdot K}}
\leq c''\cdot  d   H^2 K^{-1/2}\sqrt{\zeta}~\text{for all }x\in \cS  \Big\}
\#
%\#\label{eq:optimal_explore_rank}
%\text{SubOpt}\big(\linpess(\cD);x \big) \leq 2 \beta \sum_{h=1}^H \sqrt{\sum_{j=1}^d \frac{\lambda_{h,j}(x)}{1+c_0 \cdot \lambda_{h,j}(x) \cdot K}}\leq c' r^{1/2} \beta \cdot  H K^{-1/2}
%\#
satisfies $\PP_{\cD}(\cE'')\geq 1-\xi$, where $c''>0$ is an absolute constant that only depends on $c^\dagger$, $c$, and $r$.

%\textcolor{red}{If for some $r>0$, $\rank\big(\EE_{\pi^*}\bigl[\phi(s_h,a_h)\phi(s_h,a_h)^\top\biggiven s_1=x\bigr]\big)\leq r$ holds for all $h\in [H]$ and all $x\in \cS$, then 
%\$
%\text{SubOpt}\big(\linpess(\cD);x \big) \leq c'\cdot \beta\cdot r^{1/2} H K^{-1/2}
%\$
%holds for all $x\in \cS$ with probability at least $1-\xi$, which is with respect to $\PP_{\cD}$. Here $c'>0$ is the absolute constant in Equation \eqref{eq:optimal_explore_general}. }

%\textcolor{red}{Meanwhile, if for some absolute constant $\gamma>0$, the exponential decay 
%$
%\lambda_{h,j}(x) \leq \exp(-\gamma\cdot j) 
%$
%holds for all $h\in [H]$, all $x\in \cS$ and all $j\in [d]$, then 
%\$
%\text{SubOpt}\big(\linpess(\cD);x \big) \leq c''\cdot \beta H K^{-1/2}\sqrt{\log K}
%\$
%holds for all $x\in \cS$ with probability at least $1-\xi$, which is with respect to $\PP_{\cD}$. Here $c''>0$ is an absolute constant. }
\end{corollary}
%\end{flushleft}

\begin{proof}[Proof of Corollary \ref{cor:well_explore_optimal}]
See Appendix \ref{sec:appendix:optimal-explore} for a detailed proof.
\end{proof}

\vskip4pt
\noindent{\bf Intrinsic Uncertainty as Information Gain:} To understand Equation \eqref{eq:regret_upper_linear}, we interpret the intrinsic uncertainty in the suboptimality, which corresponds to term (ii) in Equation \eqref{eq:reg_decomp}, from a Bayesian perspective. Recall that constructing $\hat{\BB}_h \hat{V}_{h+1}$ based on $\cD$ at each step $h \in [H]$ involves solving the linear regression problem in Equation \eqref{eq:wlin}, where $\phi(x_h^\tau,a_h^\tau)$ is the covariate, $r_h^\tau + \hat{V}_{h+1}(x_{h+1}^\tau)$ is the response, and $\hat{w}_h$ is the estimated regression coefficient. Here $\hat{V}_{h+1}$ acts as a fixed function. We consider the Bayesian counterpart of such a linear regression problem. %To simplify the subsequent discussion, we assume hypothetically the response follows a Gaussian distribution with variance one conditioning on the covariate. Under a Gaussian prior $w_h \sim \text{N}(0, \lambda\cdot I)$, the posterior has the closed form 
Note that the estimator $\hat{w}_h$ specified in Equation \eqref{eq:w18} is the Bayesian estimator of $w_h$ under the prior $w_h \sim \text{N}(0, \lambda\cdot I)$ and Gaussian distribution of the response with variance one (conditioning on the covariate). 
Under the Bayesian framework, the posterior has the closed form 
\#\label{eq:wposterior}
w_h \given \cD \sim \text{N}(\hat{w}_h, \Lambda_h^{-1}),
\# 
where $\hat{w}_h$ and $\Lambda_h$ are defined in Equation \eqref{eq:w18}. Correspondingly, we have 
\$
\text{I}\bigl(w_h; \phi(s_h, a_h)\biggiven \cD\bigr)&= \text{H}(w_h \given \cD) - \text{H}\bigl(w_h \biggiven \cD, \phi(s_h, a_h)\bigr) \\
&= 1/2\cdot\log\frac{\det(\Lambda_h^\dagger)}{\det(\Lambda_h)},\quad \text{where~~} \Lambda_h^\dagger = \Lambda_h + \phi(s_h, a_h)\phi(s_h, a_h)^\top.
\$
Here $\text{I}$ is the (conditional) mutual information and $\text{H}$ is the (conditional) differential entropy. Meanwhile, we have 
\$
\log\frac{\det(\Lambda_h^\dagger)}{\det(\Lambda_h)} &= \log \det\bigl(I + \Lambda_h^{-1/2}\phi(s_h,a_h)\phi(s_h,a_h)^\top \Lambda_h^{-1/2}\bigr) \\
&= \log \bigl(1+\phi(s_h,a_h)^\top\Lambda_h^{-1}\phi(s_h,a_h)\bigr)
\approx\phi(s_h,a_h)^\top \Lambda_h^{-1} \phi(s_h,a_h),
\$
where the second equality follows from the matrix determinant lemma and the last equality holds when $\phi(s_h,a_h)^\top \Lambda_h^{-1} \phi(s_h,a_h)$ is close to zero. Therefore, in Equation \eqref{eq:regret_upper_linear}, we have 
\$
\EE_{\pi^*}\Bigl[ \bigl( \phi(s_h,a_h)^\top \Lambda_h^{-1}\phi(s_h,a_h)\bigr)^{1/2} \Biggiven s_1=x\Bigr] \approx \sqrt{2}\cdot \EE_{\pi^*}\Bigl[ \text{I}\bigl(w_h; \phi(s_h, a_h)\biggiven \cD\bigr)^{1/2} \Biggiven s_1=x\Bigr].
\$
In other words, the suboptimality in Equation \eqref{eq:regret_upper_linear}, which corresponds to the intrinsic uncertainty, can be cast as the mutual information between $w_h \given \cD$ in Equation \eqref{eq:wposterior} and $\phi(s_h, a_h)$ on the trajectory induced by $\pi^*$ in the underlying MDP. In particular, such a mutual information can be cast as the information gain \citep{schmidhuber1991curious, schmidhuber2010formal,sun2011planning,  still2012information, houthooft2016vime, russo2016information, russo2018learning} for estimating $w_h$, which is induced by observing $\phi(s_h, a_h)$ in addition to $\cD$. In other words, such a mutual information characterizes how much uncertainty in $w_h \given \cD$ can be eliminated when we additionally condition on $\phi(s_h, a_h)$.

\vskip4pt
\noindent{\bf Illustration via a Special Case: Tabular MDP:}
To understand Equation \eqref{eq:regret_upper_linear}, we consider the tabular MDP, a special case of the linear MDP, where $\cS$ and $\cA$ are discrete. Correspondingly, we set $\phi$ in Equation \eqref{eq:w07} as the canonical basis of $\RR^{|\cS| |\cA|}$. When $\cS$ is a singleton and $H = 1$, the tabular MDP reduces to the MAB, which is discussed in Section \ref{sec:mab}. Specifically, in the tabular MDP, we have
\$
\Lambda_h &= \sum_{\tau=1}^K \phi(x_h^\tau,a_h^\tau)  \phi(x_h^\tau,a_h^\tau) ^\top + \lambda\cdot I \\
&= \diag\bigl(\{N_h(x, a) + \lambda\}_{(x, a)\in \cS \times \cA}\bigr) \in \RR^{|\cS| |\cA|\times |\cS| |\cA|}, \quad \text{where~~} N_h(x, a) = \sum_{\tau=1}^K\ind{\{(x_h^\tau,a_h^\tau)=(x,a)\}}.
\$
To simplify the subsequent discussion, we assume $\cP_h$ is deterministic at each step $h \in [H]$. Let $\{(s_h^*, a_h^*)\}_{h=1}^H$ be the trajectory induced by $\pi^*$, which is also deterministic. In Equation \eqref{eq:regret_upper_linear}, we have 
\$
\EE_{\pi^*}\Bigl[ \bigl( \phi(s_h,a_h)^\top \Lambda_h^{-1}\phi(s_h,a_h)\bigr) ^{1/2} \Biggiven s_1=x\Bigr] = \bigl(N_h(s_h^*, a_h^*) + \lambda\bigr)^{-1/2}.
\$
In other words, the suboptimality in Equation \eqref{eq:regret_upper_linear} only depends on how well $\cD$ ``covers'' the trajectory induced by $\pi^*$ instead of its ``uniform coverage'' over $\cS$ and $\cA$. In particular, as long as $(s_h^\diamond, a_h^\diamond)$ lies off the trajectory induced by $\pi^*$, how well $\cD$  ``covers'' $(s_h^\diamond, a_h^\diamond)$, that is, $N_h(s_h^\diamond, a_h^\diamond)$, does not affect the suboptimality in Equation \eqref{eq:regret_upper_linear}. See Figure \ref{fig:w02} for an illustration.

\vskip4pt
\noindent{\bf Oracle Property:}
Following existing literature \citep{donoho1994ideal, fan2001variable, zou2006adaptive}, we refer to such a phenomenon as the oracle property, that is, the algorithm incurs an ``oracle'' suboptimality that automatically ``adapts'' to the support of the trajectory induced by $\pi^*$, even though $\pi^*$ is unknown a priori. From another perspective, assuming hypothetically $\pi^*$ is known a priori, the error that arises from estimating the transition kernel and expected reward function at $(s_h^*, a_h^*)$ scales as $N_h(s_h^*, a_h^*)^{-1/2}$, which can not be improved due to the information-theoretic lower bound. 

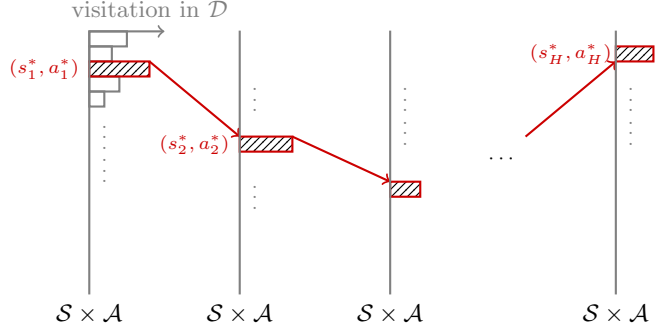
\begin{figure}
\centering
\begin{tikzpicture}
\draw [gray, thick] (-1, 2.5) -- (-0.5,2.5) -- (-0.5,2.3) -- (-1,2.3) -- cycle;
\draw [gray, thick] (-1, 2.3) -- (-0.7,2.3) -- (-0.7,2.1) -- (-1,2.1) -- cycle;
\draw [gray, thick] (-1, 1.9) -- (-0.6,1.9) -- (-0.6,1.7) -- (-1,1.7) -- cycle;
\draw [gray, thick] (-1, 1.7) -- (-0.8,1.7) -- (-0.8,1.5) -- (-1,1.5) -- cycle;
\fill [color = gray, pattern=north east lines] (-1, 2.1) -- (-0.2,2.1) -- (-0.2,1.9) -- (-1,1.9) -- cycle;
\draw [red!80!black, thick] (-1, 2.1) -- (-0.2,2.1) -- (-0.2,1.9) -- (-1,1.9) -- cycle;
\node [text=gray] at  (-0.8,1.2) {\scriptsize $\vdots$};
\node [text=gray] at  (-0.8,0.75) {\scriptsize $\vdots$};

\node[text=red!80!black] at  (-1.6,2.0) {\scriptsize $(s_1^*,a_1^*)$};

\fill [color = gray, pattern=north east lines] (1, 1.1) -- (1.7,1.1) -- (1.7,0.9) -- (1,0.9) -- cycle;
\draw [red!80!black, thick] (1, 1.1) -- (1.7,1.1) -- (1.7,0.9) -- (1,0.9) -- cycle;
\node [text=gray] at  (1.2,1.7) {\scriptsize $\vdots$};
\node [text=gray] at  (1.2,0.4) {\scriptsize $\vdots$};

\node[text=red!80!black] at  (0.4,1.0) {\scriptsize $(s_2^*,a_2^*)$};

\fill [color = gray, pattern=north east lines] (3,0.5) -- (3.4,0.5) -- (3.4,0.3) -- (3,0.3) -- cycle;
\draw [red!80!black, thick] (3,0.5) -- (3.4,0.5) -- (3.4,0.3) -- (3,0.3) -- cycle;
\node [text=gray] at  (3.2,1.7) {\scriptsize $\vdots$};
\node [text=gray] at  (3.2,1.25) {\scriptsize $\vdots$};

\fill [color = gray, pattern=north east lines] (6,2.3) -- (6.5,2.3) -- (6.5,2.1) -- (6,2.1) -- cycle;
\draw [red!80!black, thick] (6,2.3) -- (6.5,2.3) -- (6.5,2.1) -- (6,2.1) -- cycle;

\node[text=red!80!black] at  (5.4,2.2) {\scriptsize $(s_H^*,a_H^*)$};
\node [text=gray] at  (6.2,1.7) {\scriptsize $\vdots$};
\node [text=gray] at  (6.2,1.25) {\scriptsize $\vdots$};

\node[text=black] at  (4.5,0.8) {\scriptsize $\cdots$};

\draw[red!80!black, thick, ->] (-0.2,2.1) -- (1, 1.1);
\draw[red!80!black, thick, ->] (1.7,1.1) -- (3,0.5);
\draw[red!80!black, thick, ->] (4.8, 1.1) -- (6,2.1);
%\draw[red!80!black, thick, ->] (3.8, 1.4) -- (6.9, 2);

\draw[gray, thick, ->] (-1, 2.5) -- (0, 2.5);
\node[text=gray] at  (-0.2,2.8) {\small visitation in $\cD$};
\draw [gray, thick, text=black] (-1, 2.51) --(-1, -1) node[below] {\small $\cS\times\cA$};
%\draw[gray, thick, ->] (1, 2.5) -- (2, 2.5);
\draw [gray, thick, text=black] (1, 2.51) --(1, -1) node[below] {\small $\cS\times\cA$};
%\draw[gray, thick, ->] (3, 2.5) -- (4, 2.5);
\draw [gray, thick, text=black] (3, 2.51) --(3, -1) node[below] {\small $\cS\times\cA$};
\draw [gray, thick, text=black] (6, 2.51) --(6, -1) node[below] {\small $\cS\times\cA$};
\end{tikzpicture}
\caption{An illustration of the oracle property in the tabular MDP, where the transition kernel is deterministic. The histogram depicts the number of times $(s_h, a_h)$ is visited in $\cD$. The suboptimality in Equation \eqref{eq:regret_upper_linear} only depends on the number of times $(s_h^*, a_h^*)$, which lies on the trajectory induced by $\pi^*$, is visited in $\cD$, even though $\pi^*$ is unknown a priori.}
\label{fig:w02}
\end{figure}

\vskip4pt
\noindent{\bf Outperforming Demonstration:} Assuming hypothetically $\cD$ is induced by a fixed behavior policy $\bar\pi$ (namely the demonstration), such an oracle property allows $\linpess(\cD)$ to outperform $\bar\pi$ in terms of the suboptimality, which is defined in Equation \eqref{eq:def_regret}. Specifically, it is quite possible that $r_h(s_h^\diamond, a_h^\diamond)$ is relatively small and $N_h(s_h^\diamond, a_h^\diamond)$ is rather large for a certain $(s_h^\diamond, a_h^\diamond)$, which is ``covered'' by $\cD$ but lies off the trajectory induced by $\pi^*$. Correspondingly, the suboptimality of $\bar\pi$ can be rather large. On the other hand, as discussed above, $r_h(s_h^\diamond, a_h^\diamond)$ and $N_h(s_h^\diamond, a_h^\diamond)$ do not affect the suboptimality of $\linpess(\cD)$, which can be relatively small as long as $N_h(s_h^*, a_h^*)$ is sufficiently large. Here $(s_h^*, a_h^*)$ is ``covered'' by $\cD$ and lies on the trajectory induced by $\pi^*$.

\vskip4pt
\noindent{\bf Well-Explored Dataset:}
To connect existing literature \citep{duan2020minimaxoptimal}, the following corollary specializes Theorem \ref{thm:regret_upper_linear} under the additional assumption that the data collecting process well explores $\cS$ and $\cA$.
\begin{corollary}[Well-Explored Dataset]
Suppose $\cD$ consists of $K$ trajectories $\{(x_h^\tau,a_h^\tau,r_h^\tau)\}_{\tau, h=1}^{K, H}$ independently and identically induced by a fixed behavior policy $\bar\pi$ in the linear MDP. Meanwhile, suppose there exists an absolute constant $\underline{c} > 0$ such that 
\$
\lambda_{\min}(\Sigma_h)\geq \underline{c}/d, \quad \text{where~~} \Sigma_h = \EE_{\bar\pi}\bigl[\phi(s_h,a_h)\phi(s_h,a_h)^\top\bigr]
\$
at each step $h \in [H]$. Here $\EE_{\bar{\pi}}$ is taken with respect to the trajectory induced by $\bar{\pi}$ in the underlying MDP. In Algorithm \ref{alg:pess_greedy}, we set
\$
\lambda=1,\quad \beta = c\cdot d H\sqrt{\zeta}, \quad \text{where~~}\zeta= \log(4dHK/\xi).
\$
Here $c>0$ is an absolute constant and $\xi \in (0,1)$ is the confidence parameter. 
Suppose we have 
$K \geq C\cdot d \log (4 dH/ \xi)$, where $C > 0$ is a sufficiently large absolute constant that depends on $\underline{c}$. For $\pess (\cD)$ in Algorithm \ref{alg:pess_greedy}, the event 
\#\label{eq:w10}\normalfont
\cE^* = \Bigl\{ \textrm{SubOpt}\bigl(\linpess(\cD);x\bigr)\leq c'\cdot d^{3/2}H^2K^{-1/2}\sqrt{\zeta}~ \text{for all}~x \in \cS \Bigr\}
\#
satisfies $\PP_{\cD}(\cE^*)\geq 1-\xi$. Here $c'>0$ is an absolute constant that only depends on $\underline {c}$ and $c$.
%For any $x \in \cS$, $\pess (\cD)$ in  Algorithm \ref{alg:pess_greedy} satisfies
%
%Also let $\Sigma_h = \EE_{\bar\pi}[\phi(x_h,a_h)\phi(x_h,a_h)^\top]$, where $\EE_{\bar\pi}[\cdot]$ is the distribution induced by policy $\bar\pi$ on MDP $(\cS,\cA,H,\cP,r)$. Suppose the policy $\bar\pi$ explores well, so that there exists some constant $\underline{c}>0$ with $\lambda_{\min}(\Sigma_h)\geq \underline{c}$ for all $h\in [H]$.
%Then there exists some constant $c>0$ and $c'>0$ such that taking $\beta = c\cdot dH\sqrt{\log(4dHK/\xi)}$ and $\lambda=1$ in Algorithm \ref{alg:pess_greedy}, the output LinPes$(\cD)$ satisfies that $\PP_{\cD}(\cE_0^*)\geq 1-\xi$ for the event
%$$
%\cE_0^* = \Big\{ \text{SubOpt}(\linpess(\cD);x)\leq c'\cdot \frac{dH^2}{\sqrt{K}}\cdot \sqrt{\log(4dHK/\xi)},\quad \forall x\in \cS  \Big\}.
%$$
\label{cor:well_explore}
\end{corollary}

\begin{proof}[Proof of Corollary \ref{cor:well_explore}]
See Appendix \ref{sec:appendix:well-explored} for a detailed proof.
\end{proof}

The suboptimality in Equation \eqref{eq:w10} parallels the policy evaluation error established in \cite{duan2020minimaxoptimal}, which also scales as $H^2 K^{-1/2}$ and attains the information-theoretic lower bound for offline policy evaluation. In contrast, we focus on offline policy optimization, which is more challenging. As $K\rightarrow\infty$, the suboptimality in Equation \eqref{eq:w10} goes to zero.

\subsection{Minimax Optimality: Information-Theoretic Lower Bound}\label{sec:lower}

We establish the minimax optimality of Theorems \ref{thm:regret_upper_bound_general} and \ref{thm:regret_upper_linear} via the following information-theoretic lower bound. Recall that $\PP_{\cD}$ is the joint distribution of the data collecting process.

\begin{theorem}[Information-Theoretic Lower Bound]
For the output $\texttt{Algo}(\cD)$ of any algorithm,
there exist a linear MDP $\cM=(\cS,\cA,H,\cP,r)$, an initial state $x \in \cS$, and a dataset $\cD$, which is compliant with $\cM$, such that 
\#\normalfont
\EE_{\cD}\Biggl[ \frac{\text{SubOpt}\big(\texttt{Algo}(\cD);x\big)}{\sum_{h=1}^H\EE_{\pi^*}\Bigl[ \bigl( \phi(s_h,a_h)^\top \Lambda_h^{-1}\phi(s_h,a_h)\bigr) ^{1/2} \Biggiven s_1=x\Bigr]}  \Biggr] \geq {c},%\cdot \log(|\cA|/\xi)^{-1/2},
\label{eq:minimax_lower}
\#
where $c>0$ is an absolute constant. Here $\EE_{\pi^*}$ is taken with respect to the trajectory induced by $\pi^*$ in the underlying MDP given the fixed matrix $\Lambda_h$. Meanwhile, $\EE_{\cD}$ is taken with respect to $\PP_\cD$, where $\texttt{Algo}(\cD)$ and $\Lambda_h$ depend on $\cD$.
\label{thm:lower}
\end{theorem}

%\begin{theorem}[Information-Theoretic Lower Bound]
%{\red For any MDP $\cM$ in some class $\mathfrak{M}$ of MDPs and any fixed $\xi\in (0,1)$, there exists a dataset $\cD$ compliant with $\cM$ and some set of functions $\Gamma_h(\cdot,\cdot):\cS\times \cA\to \RR$ constructed from $\cD$, so that }
%\begin{enumerate}[(i)]
%\item there exists a learning procedure $\text{Algo}^*(\cdot)$ under the protocol of Algorithm \ref{alg:pess_greedy_general} with $\Gamma_h(\cdot,\cdot)$ as  $\xi$-uncertainty quantifier, so that
%$$
%\sup_{\cM\in \mathfrak{M}}\PP_{\cD}\Bigg(\frac{\text{SubOpt}\big(\text{Algo}^*(\cD);x\big)}{\sum_{h=1}^H\EE_{\pi^*}\bigl[ \Gamma_h(s_h,a_h)\Biggiven s_1=x\bigr]} \leq c' \Bigg) \geq 1-\xi
%$$
%for some constant $c'>0$, where $\pi^*$ is the optimal policy for $\cM$, and $\EE_{\pi^*}$ is taken with respect to the trajectory induced by $\pi^*$ in the underlying MDP given the fixed function $\Gamma_h$.
%\item For any algorithm $\text{Algo}(\cdot)$, there exists an MDP $\cM=(\cS,\cA,H,\cP,r)$ with which $\cD$ is compliant and an initial state $x \in \cS$ such that the output policy $\text{Algo}(\cD)$ satisfies 
%\#
%\EE_{\cD}\Biggl[ \frac{\text{SubOpt}\big(\text{Algo}(\cD);x\big)}{\sum_{h=1}^H\EE_{\pi^*}\bigl[ \Gamma_h(s_h,a_h)\Biggiven s_1=x\bigr]}  \Biggr] \geq {c}\cdot \log(|\cA|/\xi)^{-1/2},
%\#
%where $c>0$ and $\xi \in (0,1)$ are absolute constants. Meanwhile, $\PP_{\cD}$, $\EE_{\cD}$ is taken with respect to $\PP_\cD$, where $\text{Alg}(\cD)$ and $\Gamma_h$ depend on $\cD$.
%\end{enumerate}
%\label{thm:lower'}
%\end{theorem}

\begin{proof}[Proof of Theorem \ref{thm:lower}]
See Section \ref{sec:minimax_sketch} for a proof sketch and Appendix \ref{sec:appendix:lower} for a detailed proof.
\end{proof}

Theorem \ref{thm:lower} matches Theorem \ref{thm:regret_upper_linear} up to $\beta$ and absolute constants. Although Theorem \ref{thm:lower} only establishes the minimax optimality, Proposition \ref{prop:minimax_upper} further certifies the local optimality on the constructed set of worst-case MDPs via a more refined instantiation of the meta-algorithm (Algorithm  \ref{alg:pess_greedy_general}). See Appendix \ref{sec:minimax_refined_upper} for a detailed discussion.

% { 
% \color{blue} 
\subsection{Pessimistic Value Iteration: Reproducing Kernel Hilbert Spaces}
\label{sec:rkhs}

\revise{In this section, we study the Pessimistic Value Iteration 
in greater generality 
with
kernel function approximation, 
covering the linear setting of Algorithm \ref{alg:pess_greedy} as a 
special case. 
The algorithm we develop in this part 
slightly modifies 
the generic Algorithm~\ref{alg:pess_greedy_general} 
with a data splitting trick:  
we use distinct (and reverse-ordered) subsets 
of the offline dataset for the value iteration at each time step.  
Despite a (limited) reduction in the size of available sample, 
this modification allows us to remove the dependence on the covering number 
of the kernel function classes in the analysis of suboptimality upper bounds, 
thereby being particularly favorable if the covering number is large.}

\subsubsection{Basics of Reproducing Kernel Hilbert Space}
\label{subsec:rkhs_intro}
To simplify the notations, we let $z =(s ,a )$ denote the state-action pair 
and denote $\cZ=\cS\times \cA$ for any $h\in [H]$. 
Without loss of generality, we view $\cZ$
as a compact subset of $\RR^d$ where the dimension $d$ is fixed. 
Let $\cH$ be an RKHS of functions on $\cZ$ 
with kernel function $K\colon \cZ\times \cZ\to \RR$, 
innder product $\langle \cdot, \cdot\rangle \colon \cH\times\cH\to \RR$ 
and RKHS norm $\|\cdot\|_\cH\colon \cH\to \RR$. 
By definition of RKHS, there exists a \emph{feature mapping} 
$\phi\colon \cZ\to \cH$ such that $f(z) = \langle f,\phi(z)\rangle_\cH$
for all $f\in \cH$ and all $z\in \cZ$. 
Also, the kernel function admits the feature representation  
$K(x,y) = \langle \phi(x),\phi(y)\rangle_\cH$ for any $x,y\in \cH$. 
Throughout the section, we assume that the kernel function is uniformly bounded as
$\sup_{z\in \cZ}K(z,z)<\infty$. Without loss of generality, 
we assume $\sup_{z\in \cZ}K(z,z)\leq 1$ hence $\|\phi(z)\|_\cH\leq 1$ for all $z\in \cZ$. 

Let $\cL^2(\cZ)$ be the space of square-integrable functions on $\cZ$ 
with respect to Lebesgue measure and let $\langle \cdot,\cdot\rangle_{\cL^2}$ 
be the inner product for $\cL^2(\cZ)$. 
The kernel function $K$ induces an integral operator $T_K\colon \cL^2(\cZ) \to \cL^2(\cZ)$
defined by 
\#\label{eq:def_itg_op}
T_K f(z) = \int_{\cZ} K(z,z') \cdot f(z') \ud z',\quad \forall f\in \cL^2(\cZ). 
\#
Mercer's Theorem \citep{steinwart2008support} implies that there 
exists a countable and non-increasing sequence 
of nonnegative eigenvalues $\{\sigma_i\}_{i\geq 1}$
for the integral operator $T_K$, and the associated eigenfunctions $\{\psi_i\}_{i\geq 1}$
form an orthogonal basis of $\cL^2(\cZ)$. 
Moreover, the kernel function admits a spectral representation 
$K(z,z') = \sum_{i=1}^\infty \sigma_i \cdot \psi_i(z)\cdot \psi_i(z')$ 
for all $z,z'\in \cZ$.  
The eigenfunctions $\{\psi_i\}_{i\geq 1}$ 
also enables us to write the RKHS $\cH$ as a subset of $\cL^2(\cZ)$:
\$
\cH = \bigg\{  f\in \cL^2(\cZ)\colon \sum_{i=1}^\infty \frac{\langle f, \psi_i\rangle_{\cL^2}^2}{\sigma_i} <\infty   \bigg\},
\$
such that the $\cH$-inner product of any $f,g\in \cH$ can be represented as 
\$
\langle f, g\rangle_{\cH} = \sum_{i=1}^n \sigma_i^{-1} \cdot 
\langle f, \psi_i\rangle_{\cL^2} \cdot \langle g, \psi_i\rangle_\cL^2.
\$
Then the scaled eigenfunctions $\{\sqrt{\sigma_i}\psi_i\}_{i\geq 1}$ 
form an orthogonal basis for $\cH$, and the feature mappling $\phi$ can be 
written as 
$
\phi(z) = \sum_{i=1}^\infty \sigma_i \psi_i(z) \cdot \psi_i
$ 
for any $z\in \cZ$.  
\revise{In particular, 
assuming $\sigma_j = 0$ once 
$j>\gamma$ for some $\gamma\in \NN$,  
there exists a feature mapping $\phi\colon \cX\to \RR^\gamma$,  
and $\cH$ is a $\gamma$-dimensional RKHS. 
In this case, 
the kernel function approximation approach 
we introduce here 
recovers the linear setting in Section~\ref{sec:pes_linear} with $d=\gamma$.} 

\subsubsection{Pessimistic Value Iteration for Kernel Function Approximation with Data Splitting}
% {\color{red} (Call RKHS setting ``kernel function approximation'') Need a better motivation.}  

\revise{We propose a slight modification of 
Algorithm~\ref{alg:pess_greedy_general} where 
we use a distinct fold of trajectories
for value iteration at each step $h\in [H]$. 
The goal of such a  data-splitting trick is to remove the   dependency across $h \in [H]$ in Algorithm~\ref{alg:pess_greedy_general}. In specific, recall that $\hat\BB_h\hat{V}_{h+1}$ in Line 4 of Algorithm~\ref{alg:pess_greedy_general} is obtained via the ridge regression given in Equation \eqref{eq:wlin}, which uses the whole dataset $\cD$. As a result, for all $h \in [H]$,  $\{ x_h^\tau, a_h^\tau \}_{\tau \in [K]}$ are $\hat V_{h+1}$ are statistically dependent.  
When quantifying the uncertainty of $\hat\BB_h\hat{V}_{h+1}$, to handle such a statistical  dependency, in the proof of  Theorem \ref{thm:regret_upper_linear}, we additionally seek uniform concentration over a value function class that contains $\hat V_{h+1}$. 
The covering number of such a function class partly determines the scaling parameter $\beta$ in Theorem  \ref{thm:regret_upper_linear}. 
When the function class has a large covering number, the uniform convergence approach for handling the  dependency between $\{ x_h^\tau, a_h^\tau \}_{\tau \in [K]}$ and $\hat V_{h+1}$ would result in an excessively large $\beta$ such that the uncertainty quantifier is loose. 
To resolve this issue, in the following, we split the dataset $\cD$ into $H$ parts $\{ \cD_h \}_{h\in[H]}$ and use each $\cD_h$ in the construction of $\hat\BB_h\hat{V}_{h+1}$. Such a data splitting mechanism removes the undesirable statistical  dependency. As we will show in Proposition \ref{prop:rkhs_decay}, for linear function approximation, with data splitting, it suffices to choose $\beta = \tilde \cO(\sqrt{d} H)$ instead of $\beta = \tilde \cO(dH)$ in Theorem  \ref{thm:regret_upper_linear}, where $\tilde \cO(\cdot )$ hides logarithmic factors.}

Specifically, in the sequel, we first split the trajectories $\tau\in[K]$ into $H$ disjoint 
and equally-sized folds 
$\cD_{h}$ for all $h\in [H]$.   
Without loss of generality, we assume that $K/H\in \ZZ$.  
In specific, trajectories with  $\tau \in \cI_H  = \{1,\dots, K_0 \}$ are used for the construction of $\hat\BB_h \hat{V}_{h+1}$ at time step $h=H$. Here we define $K_0 = K/H$ for simplicity. 
For all $h \in [H]$, let $\cI_h = \{ (H-h)\cdot K_0 +1, \ldots, (H-h + 1) \cdot K_0\}$. Then trajectories with $\tau \in \cI_h$ is used to construct $\hat\BB_h \hat{V}_{h+1}$.
An illustration of the data-splitting step is in Figure~\ref{fig:split}. 

\begin{figure}[htbp]
\centering
\begin{tikzpicture}
  \node{\includegraphics[width=4in]{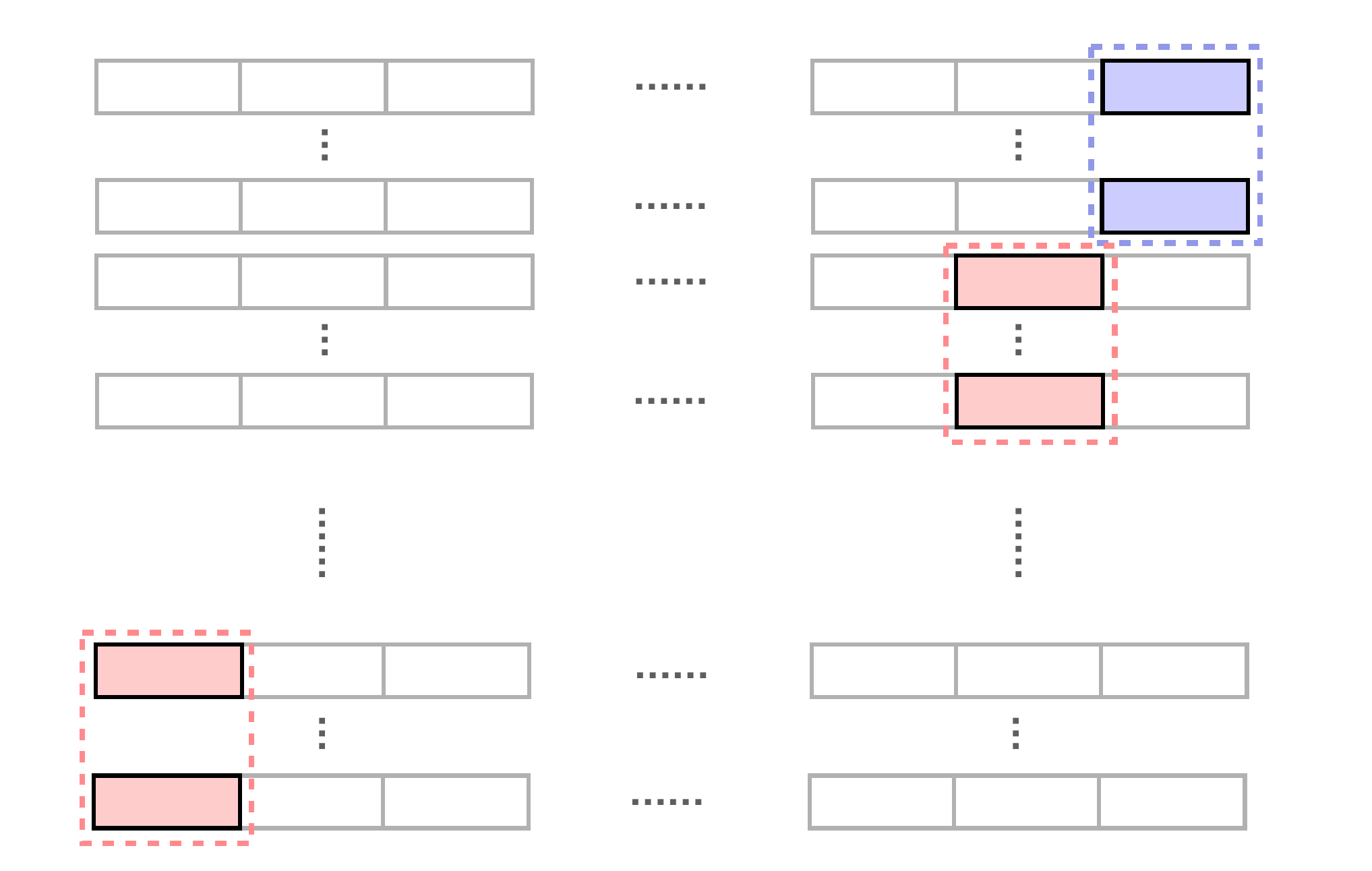}};
  \node [text=black!50] at  (-3.8, 3.3) {\footnotesize Time step };
  \node [text=black!50] at  (-3.8, 3) {\footnotesize  $h=1$};
  \node [text=black!50] at  (-2.8, 3) {\footnotesize  $h=2$};
  \node [text=black!50] at  (-1.8, 3) {\footnotesize  $h=3$};
  \node [text=black!50] at  (3.6, 3.1) {\footnotesize  $h=H$}; 
  \node [text=gray] at  (-5.8,2.55) {\footnotesize Trajectory $\tau=1$};
  \node [text=gray] at  (-5.35,1.65) {\footnotesize  $\tau=K/H$};
  \node [text=gray] at  (-5.6,1.1) {\footnotesize  $\tau=K/H+1$};
  \node [text=gray] at  (-5.4,0.22) {\footnotesize  $\tau=2K/H$};
  \node [text=blue!50] at  (5,2.1) {\small Fold $\cI_{H}$};
  \node [text=red!50] at  (2.8,-0.35) {\small Fold $\cI_{H-1}$};
  \node [text=red!50] at  (-3.7,-3.3) {\small Fold $\cI_{1}$};
  \node [text=gray] at  (-5.8,-1.75) {\footnotesize  $\tau=\frac{(H-1)K}{H} +1$};
  \node [text=gray] at  (-5.2,-2.75) {\footnotesize  $\tau=K$};
 \end{tikzpicture}
\caption{Illustration of the data-splitting step with a reverse ordering for all trajectories 
in the offline dataset. }
\label{fig:split}
\end{figure}
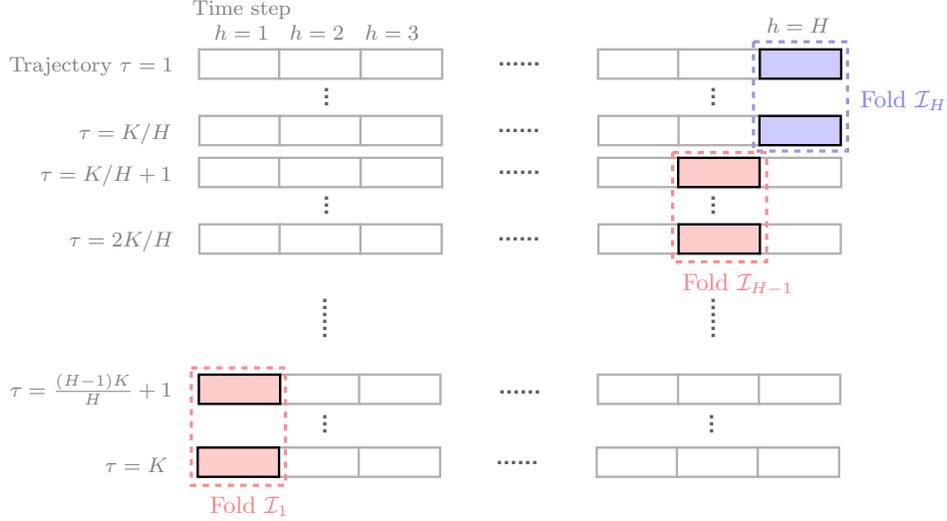
 
We then instantiate the pessimistic value iterations by constructing 
$\hat\BB_h\hat{V}_{h+1}$, $\Gamma_h$ and $\hat{V}_h$. 
To be specific,  
we define the empirical mean squared error (MSBE) as 
\$
M_h(f) = \sum_{\tau \in \cI_h}  \big( r_h^\tau + \hat{V}_{h+1}(x_{h+1}^\tau) - f(x_h^\tau,a_h^\tau) \big)^2
\$
at each step $h\in[H]$ for all $f\in \cH$, 
where only trajectories in fold $\cD_h$ are involved. 
The empirical Bellman update is obtained from 
a kernel ridge regression so that 
\$
(\hat\BB_h \hat{V}_{h+1})(z) = \hat{f}_h(z),\quad \text{where}~~ 
\hat{f}_h = \argmin_{f\in \cH} ~M_h(f) + \lambda \cdot \|f\|_\cH^2
\$
for some regularization parameter $\lambda>0$. 
Following the same arguments as in~\cite{yang2020function}, we note that $\hat{f}_h$
admits a closed-form solution
\#\label{eq:hatf_form}
\hat{f}_h(z) = k_h(z)^\top  (K_h + \lambda \cdot I)^{-1} y_h, 
\#
where we define the Gram matrix $K_h\in \RR^{K/H\times K/H}$ 
and the function $k_h\colon \cZ\to \RR^{K/H}$ as 
\#\label{eq:def_Kkh}
K_h = \big[K(z_h^\tau, z_h^{\tau'})\big]_{\tau,\tau'\in \cI_h } \in \RR^{K/H\times K/H},\quad 
k_h(z) = \big[ K(z_h^{\tau},z)\big]_{\tau\in \cI_h}^\top \in \RR^{K/H},
\#
and the entry of the response vector $y_h\in \RR^{K/H}$ corresponding to $\tau\in \cI_h$ is 
\#\label{eq:def_yh}
[y_h]_\tau = r_h^\tau + \hat{V}_{h+1}(x_{h+1}^\tau). 
\#
Moreover, we construct $\Gamma_h$ via
\#\label{eq:gamma_rkhs}
\Gamma_h(z) = \beta \cdot \lambda^{-1/2} \cdot \big( K(z,z) - k_h(z)^\top (K_h + \lambda I)^{-1} k_h(z) \big)^{1/2},
\#
where $\beta>0$ is a scaling parameter. 
Finally, we construct the pessimistic Q-function by 
\$
\hat{Q}_h(z) &= \min \big\{ \bar{Q}_h(z), H-h+1   \big\}^+,
\quad \text{where}~~
\bar{Q}_h(z) = \hat{f}_h(z) - \Gamma_h(z),\\
\hat{V}_h(x) &= \big\langle \hat{Q}_h(x,\cdot) , \hat\pi_h(\cdot\given x)\big\rangle_\cA,
\quad \text{where}~~
\hat\pi_h(\cdot\given x) = \argmax_{\pi_h}  \big\langle \hat{Q}_h(x,\cdot) , \hat\pi_h(\cdot\given x)\big\rangle_\cA.
\$

\begin{algorithm}[H]
\caption{Pessimistic Value Iteration (PEVI): RKHS Approximation with data splitting}\label{alg:pess_rkhs}
\begin{algorithmic}[1]
\STATE Input: Dataset $\cD=\{(x_h^\tau,a_h^\tau,r_h^\tau)\}_{\tau, h=1}^{K, H}$.
\STATE Initialization: Set $\hat{V}_{H+1}(\cdot) \leftarrow 0$.
\STATE Data splitting: Randomly split $[K]$ into disjoint and equal-sized folds $\{\cI_h\}_{h=1}^H$, $|\cI_h|=K/H$. 
\FOR{step $h=H,H-1,\ldots,1$}
\STATE Compute the Gram matrix $K_h\in \RR^{K/H}$, function $k_h\colon \cZ\to \RR^{K/H}$ and response $y_h\in \RR^{K/H}$ as Equations~\eqref{eq:def_Kkh} and~\eqref{eq:def_yh}. %\hfill  {//Estimation}
\hfill  {//Estimation}
\STATE Set $\Gamma_h(\cdot,\cdot) \leftarrow \beta \cdot \lambda^{-1/2} \cdot ( K(\cdot,\cdot\,;\cdot,\cdot) - k_h(\cdot,\cdot)^\top (K_h + \lambda I)^{-1} k_h(\cdot,\cdot)  )^{1/2} $. \hfill  {//Uncertainty}
\STATE Set $\overline{Q}_h(\cdot,\cdot) \leftarrow  k_h(\cdot,\cdot)^\top (K_h+\lambda I)^{-1}y_h$. \hfill  {//Pessimism} 
\STATE Set $\hat{Q}_h(\cdot,\cdot) \leftarrow \min\{\overline{Q}_h(\cdot,\cdot),H-h+1\}^+$. \hfill  {//Truncation}
\STATE Set $\hat{\pi}_h (\cdot \given \cdot) \leftarrow \argmax_{\pi_h}\langle \hat{Q}_h(\cdot, \cdot),\pi_h(\cdot\given \cdot)\rangle_{\cA}$.\hfill {//Optimization}
\STATE Set $\hat{V}_h(\cdot) \leftarrow \langle \hat{Q}_h(\cdot,\cdot),\hat\pi_h(\cdot \given \cdot)\rangle_{\cA}$.\hfill {//Evaluation} \label{alg:linear_Vhat}
\ENDFOR 
\STATE Output: $\pess(\cD) = \{\hat{\pi}_h\}_{h=1}^H$.
\end{algorithmic}
\end{algorithm}

\revise{Recall that we define the Bellman operator $\BB_h$
on any value function $V\colon \cS \to \RR$ 
in Equation~\eqref{eq:def_bellman_op}.  
We impose the following structural assumption for the kernel setting.}   

\begin{assumption}\label{assump:rkhs_bound}
Let $R_Q>0$ be some fixed constant and we define function class
$\cQ^* = \{f\in \cH\colon \|f\|_\cH \leq R_Q H\}$. Then
for any $h\in[H]$ and any $Q\colon \cS\times \cA\to [0,H]$, it holds that 
$\BB_h V\in \cQ^*$ for $V(x) = \max_{a\in \cA}Q(x,a)$. 
\end{assumption}

\revise{The above assumption states that 
the Bellman operator maps any bounded function into a bounded 
RKHS-norm ball.
When the RKHS has finite spectrum (i.e., $\sigma_j =0$ once $j>\gamma$ 
for some $\gamma\in \NN$), a sufficient condition for Assumption~\ref{assump:rkhs_bound} 
to hold 
is the linear MDP defined in Definition~\ref{assump:linear_mdp}.}

Besides the closeness assumption on the Bellman operator, 
we also define the maximal information gain~\citep{srinivas2009gaussian}
as a characterization of the complexity of $\cH$:
\#\label{eq:def_Gklambda}
G(n,\lambda) = \sup \big\{ 1/2 \cdot \log\det (I + K_\cC/\lambda) \colon 
\cC\subset \cZ,~|\cC|\leq n  \big\}.
\#
Here $K_\cC$ is the Gram matrix for the set $\cC$, defined 
similarly as Equation~\eqref{eq:def_Kkh}. 
\revise{In particular, when $\cH$ has $\gamma$-finite spectrum, $G(n,\lambda) = \cO(\gamma \cdot \log n)$ recovers the dimensionality of the linear space up to a logarithmic factor.  More importantly, information gain defined in \eqref{eq:def_Gklambda} offers a characterization of the effective dimension of   $\cH$ especially when $\cH$ is infinite-dimensional.}

% Under these regularity conditions, 
The suboptimality of the output %\texttt{Pess}($\cD$) 
of Algorithm~\ref{alg:pess_rkhs}
is characterized by the following theorem. 

\begin{theorem}\label{thm:rkhs}
% Fix any $\xi\in(0,1)$. 
Suppose Assumption~\ref{assump:rkhs_bound} holds, and 
there exists some $\lambda\geq 1+1/K$ and $B>0$ satisfying 
\#\label{eq:rkhs_B}
2\lambda R_Q^2   + 
2   G(K/H,1+1/K) + 2 K/H  \cdot \log (1+1/K)+ 4   \log\big(H /\xi\big) 
 \leq (B/H)^2.
\#
% for $R=2H\sqrt{K\cdot G(K,1+1/K)}$. 
We set $\beta=B$ in Algorithm~\ref{alg:pess_rkhs}.
Then with probability at least $1-\xi$ with repsect to $\PP_\cD$, 
it holds simultaneously 
for all $x\in \cS$ and all $h\in[H]$ that
{\normalfont  \$ 
&\text{SubOpt}\big( \pess (\cD); x \big) 
\leq 
2\sqrt{2}B \cdot \sum_{h=1}^H \EE_{\pi^*}
\Big[  \big\{\log\det\big( I+K_h(s_h,a_h)/\lambda \big) - 
\log\det(I+K_h/\lambda) \big\}^{1/2}
\Biggiven s_1=x\Big].  
\$}
Here $\EE_{\pi^*}$ is taken with respect to the trajectory induced by $\pi^*$
in the underlying MDP given the fixed Gram matrix $K_h\in \RR^{K\times K}$ 
and the fixed operator $k_h\colon \cZ\to \RR^K$ defined 
in Equation~\eqref{eq:def_Kkh}, and for any $z\in \cZ$, we define 
$K_h(z)\in \RR^{(K+1)\times(K_1)}$ as the Gram matrix for $\{z_h^\tau\}_{\tau\in[H]}\cup\{z\}$. 
\end{theorem}

\begin{proof}[Proof of Theorem~\ref{thm:rkhs}]
See Appendix~\ref{app:proof_thm_rkhs} for a detailed proof. 
\end{proof}

% {\color{red} Add a discussion. What does the theorem tell us? what is the effect of data splitting? } 
\revise{
Theorem~\ref{thm:rkhs} expresses the suboptimality upper bound 
in a generic form consisting of two parts: 
(i) a parameter $B>0$ 
that depends on the kernel function class, as well as 
(ii) an information quantity 
% To simplify notations, we 
% denote the information quantity in Theorem~\ref{thm:rkhs} as
\#\label{eq:rkhs_id}
I_\cD =  \sum_{h=1}^H \EE_{\pi^*}\Big[ \big\{\log\det\big( I+K_h(s_h,a_h)/\lambda \big) - \log\det(I+K_h/\lambda) \big\}^{1/2}\Biggiven s_1=x\Big].
\#
that only depends on the optimal policy $\pi^*$ and the offline dataset. 
In the same spirit of our preceding results, 
pessimism eliminates 
the spurious correlation (c.f.~Equation~\eqref{eq:reg_decomp})
with a properly constructed uncertainty quantifier $\{\Gamma_h\}_{h=1}^H$. 
When the RKHS has $\gamma$-finite spectrum with $\sigma_j =0$ for all $j>\gamma$, 
$\cI_\cD$ reduces to the one in Theorem~\ref{thm:regret_upper_linear} 
for the linear setting (with data splitting). 
 
}

\revise{ 
In what follows, 
we interpret 
the generic bound in Theorem~\ref{thm:rkhs} 
under specific conditions on $\cH$, 
focusing on the resulting forms of the two components. 
We would see the effect of sample splitting in our discussion. 
Firstly, the parameter $B>0$ depends on the information gain $G(K/H,1+1/K)$, 
which can be viewed as a characterization of the complexity of $\cH$. 
We provide explicit choices of 
$B$  
under various eigenvalue decay conditions of $\cH$ that 
decide such complexity. 
}

\begin{assumption}[Eigenvalue Decay of $\cH$]\label{assump:rkhs_decay}
Let $\{\sigma_j\}_{j\geq 1}$ be the  
eigenvalues induced by the integral opretaor 
$T_K$ defined in Equation~\eqref{eq:def_itg_op}
and $\{\psi_j\}_{j\geq 1}$ be the associated eigenfunctions. 
We assume that $\{\sigma_j\}_{j\geq 1}$ satisfies one of the following conditions 
for some constant $\gamma>0$. 
\begin{enumerate}[(i)]
  \item $\gamma$-finite spectrum: $\sigma_j = 0$ for all $j>\gamma$, where $\gamma$ is a positive integer.
  \item $\gamma$-exponential decay: there exists some constants $C_1,C_2>0$, $\tau\in[0,1/2)$ and $C_\psi>0$ such that $\sigma_j \leq C_1 \cdot \exp(-C_2 \cdot j^\gamma)$ and $\sup_{z\in \cZ} \sigma_j^\tau \cdot |\psi_j(z)|\leq C_\psi$ for all $j\geq 1$. 
  \item $\gamma$-polynomial decay: there exists some constants $C_1>0$, $\tau\in[0,1/2)$ and $C_\psi>0$  such that $\sigma_j \leq C_1\cdot j^{-\gamma}$ and $\sup_{z\in \cZ} \sigma_j^\tau \cdot |\psi_j(z)|\leq C_\psi$ for all $j\geq 1$, where $\gamma>1$. 
\end{enumerate} 
\end{assumption}

The $\gamma$-finite spectrum condition is satisfied by the linear MDP~\citep{jin2020provably}
with feature dimension $\gamma$, 
and Algorithm~\ref{alg:pess_rkhs} reduces to the algorithm for linear MDP 
established in preceding sections \revise{(with data splitting)}.  
Also, the exponential and polynomial decay are 
relatively mild conditions compared to those in the literature. 
We 
refer the readers to Section 4.1 of~\cite{yang2020function} 
for a detailed discussion on the eigenvalue decay conditions.  
Under the conditions in Assumption~\ref{assump:rkhs_decay}, 
Proposition~\ref{prop:rkhs_decay} 
establishes the concrete choices of $B$ 
for Theorem~\ref{thm:rkhs}.

\begin{prop}\label{prop:rkhs_decay}
Under Assumptions~\ref{assump:rkhs_bound} and~\ref{assump:rkhs_decay}, 
% we additionally assume $\eta^*>0$ and $\kappa^*>0$.  
% Then
we set $\lambda \geq  1+1/K$ and $\beta=B$ in Algorithm~\ref{alg:pess_rkhs}, where 
% satisfying Equation~\eqref{eq:rkhs_B}, where 
$$
  B = 
\begin{cases}
C \cdot H\cdot \big\{ \sqrt{\lambda} +  \sqrt{ \gamma \cdot \log(  KH/\xi)}\big\} & \qquad \gamma\text{-finite spectrum},\\
C \cdot H\cdot \big\{ \sqrt{\lambda} + \sqrt{(\log (KH/\xi))^{ 1+1/\gamma } }\big\} & \qquad \gamma\text{-exponential decay},\\
C  \cdot   H\cdot\big\{\sqrt{\lambda} + K^{\frac{d+1}{2(\gamma+d)}} H^{-\frac{d+1}{2(\gamma+d)}} \cdot \sqrt{\log(KH/\xi)} \big\} & \qquad \gamma\text{-polynomial decay}.
\end{cases}
$$
Here $C>0$ is an absolute constant that does not depend on $K$ or $H$. 
% Furthermore, 
Then with probability at least $1-\xi$ with respect to $\cD$, it holds 
that 
$$
{\normalfont \textrm{SubOpt}\big( \pess (\cD); x \big)} \leq 
\begin{cases}
  C \cdot I_\cD \cdot H \cdot \big\{\sqrt{\lambda} + \sqrt{\gamma \cdot \log( KH/\xi)} \big\} & \qquad \gamma\text{-finite spectrum},\\
  C \cdot I_\cD \cdot H \cdot \big\{\sqrt{\lambda} + \sqrt{(\log KH/\xi)^{ 1+1/\gamma }  } \big\} & \qquad \gamma\text{-exponential decay},\\
  C \cdot I_\cD \cdot  H\cdot \big\{ \sqrt{\lambda } + K^{\frac{d+1}{2(\gamma+d)}}H^{-\frac{d+1}{2(\gamma+d)}}\cdot   \sqrt{\log(KH/\xi)}\big\} & \qquad \gamma\text{-polynomial decay},
  \end{cases}
$$
for some absolute constant $C>0$ that does not depend on $K$ or $H$. 
\end{prop}

\begin{proof}[Proof of Proposition~\ref{prop:rkhs_decay}]
  See Appendix~\ref{app:proof_prop_rkhs_decay} for a detailed proof.
\end{proof}

\revise{
Under $\gamma$-finite spectrum condition, 
taking $\lambda=1+1/K$ in Algorithm~\ref{alg:pess_rkhs} 
leads to a variant of Algorithm~\ref{alg:pess_greedy} with 
sample splitting;   
setting $\lambda=1+1/K$ 
instead of $\lambda=1$ does not 
change the order of upper bounds in 
Theorem~\ref{thm:regret_upper_linear} for linear MDP. 
Firstly, 
comparing 
$B = \tilde \cO(\sqrt{\gamma}  H)$ 
in Proposition~\ref{prop:rkhs_decay} 
to
% the preceding 
$B = \tilde \cO(\gamma H)$ for linear MDP  
where $d=\gamma$ in Theorem~\ref{thm:regret_upper_linear},  
data splitting improves the upper bound by a factor of $\sqrt{\gamma}$ 
since it removes the dependence of $B$ 
on the covering number 
of the (linear) function class. 
Here $\tilde \cO(\cdot )$ hides logarithmic factors. 
On the other hand, 
in ideal settings such as the well-explored case of 
Corollary~\ref{cor:well_explore}, $\cI_\cD$ 
is of order $\tilde \cO(\sqrt{d}H\cdot |\cI_h|^{-1/2})$, 
where the reduction in sample size incurs an additional factor of $\sqrt{H}$. 
Thus, 
the data splitting approach is favorable 
if the horizon $H$ is of a smaller order than $d=\gamma$. 
In general, 
the data splitting approach 
improves sample efficiency if the kernel function class
has a covering number that is larger than $\exp(H)$. 
}

\revise{
To further understand the behavior of $\cI_\cD$ 
beyond the $\gamma$-finite spectrum setting,  
we now consider a special case where 
the offline dataset consists of 
i.i.d.~trajectories induced by some behavior policy. 
This offers a more clear illustration 
of the learning performance by certain 
population quantities that characterzie  how close the 
behavior policy is to $\pi^*$.

}

\paragraph{Offline dataset from i.i.d.~sampling.}
We study a special case where the offline dataset 
consist of i.i.d.~trajectories from some fixed 
behavior policy $\pi^b$;  
this enables us to translate $\cI_\cD$ 
into population quantities with specific choices of $\lambda$. 
The learning performance would depend on  
the ``coverage'' 
of $\pi^b$ for 
the optimal policy $\pi^*$, communicated by 
the following notion of ``effective dimension''. 

\begin{definition}[Effective dimension]
Let $K_h\in \RR^{K/H\times K/H}$ be defined in~\eqref{eq:def_Kkh} 
for all $h\in[H]$. 
Denote $\Sigma_h = \EE_{\pi^b}[\phi(z_h)\phi(z_h)^\top\given s_1=x]$, $\Sigma_h^* = \EE_{\pi^*}[\phi(z_h)\phi(z_h)^\top\given s_1=x]$, 
where $\EE_{\pi^*}$ is taken with respect to $(x_h,a_h)$ induced by 
the optimal policy $\pi^*$, and $\EE_{\pi^b}$ is 
similarly induced by the behavior policy $\pi^b$. 
We define the 
(sample) effective dimension as 
\#\label{eq:def_eff_dim_sample}
d_{\textrm{eff}}^{\text{sample}} = 
\sum_{h=1}^H \tr\big( (K_h + \lambda \cI_\cH)^{-1} \Sigma_h^*  \big)^{1/2}. 
\#
Moreover, we define the population effective dimension under $\pi^b$ as
\#\label{eq:def_eff_dim_pop}
d_{\textrm{eff}}^{\text{pop}} 
&= \sum_{h=1}^H \tr\big( (K/H\cdot \Sigma_h + \lambda  \cI_\cH)^{-1} \Sigma_h^*  \big)^{1/2}. 
\# 
\end{definition}

\begin{corollary}\label{cor:rkhs_iid}
Suppose $\cD$ consists of i.i.d.~trajectories sampled from 
behavior policy $\pi^b$, 
and Assumption~\ref{assump:rkhs_decay} holds; 
in case (iii) $\gamma$-polynomial decay, we additionally assume 
$\gamma(1-2\tau)>1$. 
In Algorithm~\ref{alg:pess_rkhs}, we set 
$B>0$ as in Proposition~\ref{prop:rkhs_decay} 
and 
$$
\lambda = 
\begin{cases}
C \cdot \gamma \cdot \log( K /\xi) & \qquad \gamma\text{-finite spectrum},\\
C \cdot \big[\log(K /\xi )\big]^{1+1/\gamma} & \qquad \gamma\text{-exponential decay},\\
C  \cdot  (K/H)^{\frac{2}{ \gamma(1-2\tau)-1}} \cdot \log(K /\xi) & \qquad \gamma\text{-polynomial decay},
\end{cases}
$$
where $C>0$ is a sufficiently large absolute constant that 
does not depend on $K$ or $H$. 
Then 
with probability at least $1-\xi$ with respect to $\cD$, 
it holds that  
\begin{equation} \label{eq:iid_rkhs_upp}
  {\normalfont \textrm{SubOpt}\big( \pess (\cD); x \big)} \leq 
\begin{cases}
C' \cdot d_{\text{eff}}^{\text{pop}} \cdot H \cdot \sqrt{ \gamma \cdot \log(KH/\xi)}) & \qquad \gamma\text{-finite spectrum},\\
C \cdot d_{\text{eff}}^{\text{pop}} \cdot H\cdot \sqrt{(\log(KH/\xi))^{1  + 1 \gamma }} & \qquad \gamma\text{-exponential decay},\\
C' \cdot d_{\text{eff}}^{\text{pop}} \cdot K^{\kappa^*} H^{\nu^*} \cdot \sqrt{\log(KH/\xi)} & \qquad \gamma\text{-polynomial decay}.
\end{cases}
\end{equation}
where $C'>0$ is an absolute constant  that does not 
depend on $K$ or $H$, and 
\$
\kappa^* = \frac{d+1}{2(\gamma +d)} + \frac{1}{ \gamma(1-2\tau)-1},\quad 
\nu^* = 1 - \frac{d+1}{2(\gamma+d)} - \frac{1}{ \gamma(1-2\tau)-1}.
\$
The same results also apply to $d_{\text{eff}}^{\text{sample}}$. 
\end{corollary}

\begin{proof}[Proof of Corollary~\ref{cor:rkhs_iid}]
See Appendix~\ref{app:subsec_cor_rkhs_iid} for a detailed proof. 
\end{proof}

\revise{  
Parallel to the linear setting, Corollary~\ref{cor:rkhs_iid} 
demonstrates the performance of our method in terms of 
$d_{\text{eff}}^{\text{pop}}$ that depends on the relationship 
between $\Sigma_h$ (from $\pi^b$) and $\Sigma_h^*$ (from $\pi^*)$. 
When $\Sigma_h$ and $\Sigma_h^*$ are 
close (i.e., $\pi^b$ covers $\pi^*$ well), 
we have $d_{\text{eff}}^{\text{pop}} \approx \tilde\cO(H^{3/2}K^{1/2})$; 
in this case, 
the suboptimality is of order $K^{-1/2}$ 
under $\gamma$-finite spectrum and $\gamma$-exponential decay, 
while for $\gamma$-polynomial decay we obtain a sublinear rate of 
$K^{\kappa^*-1/2}$. 

}

%  }

%\end{flushleft}

% !TEX root = paper.tex
%\begin{flushleft}
\section{Proof Sketch}

In this section, we sketch the proofs of the main results in Section \ref{sec:pess}. In Section \ref{sec:upper_sketch_general}, we sketch the proof of Theorem \ref{thm:regret_upper_bound_general}, which handles any general MDP. In Section \ref{sec:upper_sketch_linear}, we specialize it to the linear MDP, which is handled by Theorem \ref{thm:regret_upper_linear}. In Section \ref{sec:minimax_sketch}, we sketch the proof of Theorem \ref{thm:lower}, which establishes the information-theoretic lower bound.

%We sketch the proof of the suboptimality upper bound of PEVI  for the linear setting in Section \ref{sec:upper_sketch_linear}. Finally, in Section \ref{sec:minimax_sketch} we sketch the proof of Theorem \ref{thm:lower}, which establishes an information-theoretic lower bound of offline RL. 

\subsection{Suboptimality of PEVI: General MDP} \label{sec:upper_sketch_general}

Recall that we define the model evaluation errors $\{ \iota_h \} _{h=1}^H$ in Equation  \eqref{eq:def_iota}, which are based on the (action- and state-)value functions $\{ (\hat Q_{h}, \hat V_h) \}_{h=1}^H$ constructed by PEVI.
Also, recall that we define the $\xi$-uncertainty quantifiers $\{ \Gamma_h \}_{h=1}^H$ in Definition \ref{def:uncertainty_quantifier}. 
The key to the proof of Theorem \ref{thm:regret_upper_bound_general} is to show that for all $h \in [H]$, the constructed Q-function $\hat{Q}_h $ in Algorithm \ref{alg:pess_greedy_general} is a pessimistic estimator of the optimal Q-function $Q_h^*$. To this end, in the following lemma,  
we prove that under the event $\cE$ defined in Equation \eqref{eq:def_event_eval_err_general}, $\iota_h$ lies within $[0,2 \Gamma_h]$ in a pointwise manner for all $h \in [H]$. Recall that $\PP_{\cD}$ is the joint distribution of the data collecting process.

\begin{lemma}[Pessimism for General MDP]
Suppose that $\{\Gamma_h\}_{h=1}^H$ in Algorithm \ref{alg:pess_greedy_general} are $\xi$-uncertainty quantifiers. Under $\cE$ defined in Equation \eqref{eq:def_event_eval_err_general}, which satisfies $\PP_{\cD}(\cE)\geq 1-\xi$, we have
\label{lem:model_eval_err}
\#\label{eq:model_eval_err_bound}
0\leq \iota_h(x,a) \leq 2\Gamma_h(x,a),\quad \text{for all}~~(x,a)\in \cS\times \cA,~ h\in [H]. 
\#
%Suppose $\{ \Gamma_h\}_{h = 1 } ^ H $ in Algorithm \ref{alg:pess_greedy_general} are a $\xi$-uncertainty quantifier specified in Definition \ref{def:uncertainty_quantifier}. Then, on event $\cE$  defined in Equation \eqref{eq:def_event_eval_err_general}, the model evaluation errors $\{ \iota_h \}_{h = 1}^H $ satisfies that
%\#\label{eq:model_eval_err_bound}
%0\leq \iota_h(x,a) \leq 2\Gamma_h(x,a),\qquad \forall (x,a)\in \cS\times \cA, \forall h \in [H]. 
%\#
%Recall that $\cE$  holds with probability at least  $1-\xi$ under $\PP_{\cD}$.
\end{lemma}
\begin{proof}[Proof of Lemma \ref{lem:model_eval_err}]
See Appendix \ref{sec:app:proof_lem_model_eval} for a detailed proof.
\end{proof}

In Equation \eqref{eq:model_eval_err_bound}, the nonnegativity of $\{ \iota_h \} _{h=1}^H$ implies the pessimism of $\{ \hat{Q}_h \} _{h=1}^H$, that is, $\hat Q_h \leq Q_h^*$ in a pointwise manner for all $h\in [H] $. 
To see this, note that the definition of $\{ \iota_h \} _{h=1}^H$ in Equation \eqref{eq:def_iota} gives
\# 
Q_h^* (x,a) - \hat Q_h (x,a)  \geq (\BB_h V_{h+1} ^* ) (x,a) -  (\BB _h \hat V_{h+1} ) (x,a) = (\PP_h V_{h+1}^*) (x,a) - ( \PP_h \hat V_{h+1}) (x,a)  \label{eq:pess_recursion1},
\#
which together with Equations \eqref{eq:def_transition_op} and \eqref{eq:dp_optimal_values} further implies
\#
Q_h^* (x,a) - \hat Q_h (x,a)& \geq \EE  \bigl [  \max_{a
' \in \cA} Q_{h+1}^*(s_{h+1}, a')  -  \la  \hat Q_{h+1} (s_{h+1}, \cdot ), \hat \pi_{h+1} (\cdot \given s_{h+1}) \ra_{\cA}  \biggiven s_h = x, a_h = a \bigr ] 
\notag \\
& \geq \EE  \bigl [    \la Q_{h+1}^*(s_{h+1}, \cdot )  - \hat Q_{h+1} (s_{h+1}, \cdot ), \hat \pi_{h+1} (\cdot \given s_{h+1})  \ra_{\cA}  \biggiven s_h = x, a_h = a \bigr ]\label{eq:pess_recursion2}
\#
for all $(x, a) \in \cS \times \cA$ and $h \in [H]$. Also, note that $\hat V_{H+1} = V_{H+1}^* = 0$. 
Therefore, Equation \eqref{eq:pess_recursion1} implies $Q_{H}^* \geq \hat Q_{H} $ in a pointwise manner. 
Moreover, by recursively applying Equation \eqref{eq:pess_recursion2}, we have $Q_h^* \geq \hat Q_h$ in a pointwise manner for all $h \in [H]$. In other words, Lemma \ref{lem:model_eval_err} implies that the pessimism of $\{ \hat{Q}_h \} _{h=1}^H$ holds with probability at least $1-\xi$ as long as $\{\Gamma_h \}_{h = 1}^ H $ in Algorithm \ref{alg:pess_greedy_general} are $\xi$-uncertainty quantifiers, which serves as a sufficient condition that can be verified. Meanwhile, the upper bound of $\{ \iota_h \} _{h=1}^H$ in Equation \eqref{eq:model_eval_err_bound} controls the underestimation bias of $\{ \hat{Q}_h \} _{h=1}^H$, which arises from pessimism. 

%which is indispensable for establishing an upper bound of the  suboptimality.

Based on Lemma \ref{lem:model_eval_err},
we are ready to prove Theorem \ref{thm:regret_upper_bound_general}.

\begin{proof}[Proof of Theorem \ref{thm:regret_upper_bound_general}]
    We upper bound the three terms on the right-hand side of Equation \eqref{eq:reg_decomp} respectively. Specifically, we apply Lemma \ref{lem:dec} by setting $\hat\pi=\{\hat\pi_h\}_{h=1}^H$ as the output of Algorithm \ref{alg:pess_greedy_general}, that is, $\hat \pi = \pess(\cD)$. 
As $\hat \pi_h$ is greedy with respect to $\hat Q_h $ for all $h \in [H]$, term (iii) in Equation \eqref{eq:reg_decomp} is nonpositive. 
Therefore, we have 
\# \label{eq:decomp_pess_reg}
    \text{SubOpt}\big(\pess(\cD);x \big) & \leq \underbrace{-\sum_{h=1}^H \EE_{\hat\pi}\big[ \iota_h(s_h,a_h) \biggiven s_1=x\big]}_{\displaystyle \rm{(i)}} + \underbrace{\sum_{h=1}^H \EE_{\pi^*}\big[ \iota_h(s_h,a_h) \biggiven s_1=x \big]}_{\displaystyle\rm{(ii)}}
\#
for all $x \in \cS$, where terms (i) and (ii) characterize the spurious correlation and intrinsic uncertainty, respectively. 
To upper bound such two terms, we invoke Lemma \ref{lem:model_eval_err}, which implies  
\#\label{eq:bound_term2_decomp}
{\displaystyle\rm{(i)}} \leq 0,\quad {\displaystyle\rm{(ii)}} \leq  2\sum_{h=1}^H \EE_{\pi^*}\big[ \Gamma_h(s_h,a_h) \biggiven s_1=x \big]
\#
for all $x \in \cS$. Combining Equations  \eqref{eq:decomp_pess_reg} and \eqref{eq:bound_term2_decomp}, we obtain Equation \eqref{eq:regret_upper_general} under $\cE$ defined in Equation \eqref{eq:def_event_eval_err_general}. 
Meanwhile, by Definition \ref{def:uncertainty_quantifier}, we have $\PP_{\cD}(\cE)\geq 1-\xi$. 
Therefore, we conclude the proof of Theorem~\ref{thm:regret_upper_bound_general}. 
\iffalse 
By the suboptimality decomposition in Lemma \ref{lem:reg_decomp}, we know

where we denote $\hat\pi = \pess(\cD)$ the output policy. As $\hat\pi$ is greedy with respect to $\hat{Q}_h(s_h,\cdot)$, we have term (iii)$\leq 0$. Since $\iota_h(x,a)\geq 0$ for all $(x,a)\in \cS\times \cA$ on event $\cE$, the term (i)$\leq 0$ on the same event. By monotonicity of expectations we have the upper bound 
\begin{equation*}
\begin{split}
\text{SubOpt}\big(\pess(\cD);x \big) & \leq \sum_{h=1}^H \EE_{\pi^*}\big[ \iota_h(s_h,a_h) \biggiven s_1=x \big] \leq 2\sum_{h=1}^H \EE_{\pi^*}\big[ \Gamma_h(s_h,a_h) \biggiven s_1=x \big] 
\end{split}
\end{equation*}
on the event $\cE$ in Equation \eqref{eq:def_event_eval_err_general}, which satisfies $\PP_{\cD}(\cE)\geq 1-\xi$ since $\Gamma_h(\cdot,\cdot)$ is a $\xi$-uncertainty quantifier. This completes the proof.
\fi 
\end{proof}

\subsection{Suboptimality of PEVI: Linear MDP}\label{sec:upper_sketch_linear}
Based on Theorem \ref{thm:regret_upper_bound_general}, we are ready to prove Theorem \ref{thm:regret_upper_linear}, which is specialized to the linear MDP defined in Definition \ref{assump:linear_mdp}.

\begin{proof}[Proof of Theorem \ref{thm:regret_upper_linear}] 
It suffices 
to show that $\{\Gamma_h \} _{h = 1}^H $ specified in   Equation \eqref{eq:w05} are $\xi$-uncertainty quantifiers, which are defined in Definition \ref{def:uncertainty_quantifier}. In the following lemma, we prove that such a statement holds when the regularization parameter $\lambda > 0$ and scaling parameter $\beta > 0$ in Algorithm \ref{alg:pess_greedy} are  properly chosen.
 
\begin{lemma}[$\xi$-Uncertainty Quantifier  for Linear MDP] 
Suppose that Assumption \ref{assump:data_generate} holds and the underlying MDP is a linear MDP. In Algorithm \ref{alg:pess_greedy}, we set
\$
\lambda=1,\quad \beta = c\cdot dH\sqrt{\zeta}, \quad \text{where~~}\zeta= \log(2dHK/\xi).
\$
Here $c>0$ is an absolute constant and $\xi \in (0,1)$ is the confidence parameter. It holds that $\{\Gamma_h\}_{h=1}^H$ specified in Equation \eqref{eq:w05} are $\xi$-uncertainty quantifiers, where  $\{\hat V_{h+1}\}_{h=1}^H$ used in Equation \eqref{eq:def_event_eval_err_general} are obtained by Algorithm \ref{alg:pess_greedy}.
\label{lemma:linear_MDP_uncertainty}
\end{lemma}

\begin{proof}[Proof of Lemma \ref{lemma:linear_MDP_uncertainty}]
See Appendix \ref{sec:app:proof_prop_linear} for a detailed proof.
\end{proof}

As Lemma \ref{lemma:linear_MDP_uncertainty} proves that $\{ \Gamma_h \} _{h=1}^H $ specified in Equation \eqref{eq:w05} are $\xi$-uncertainty quantifiers, $\cE$ defined in Equation \eqref{eq:def_event_eval_err_general} satisfies $\PP_{\cD}(\cE)\geq 1-\xi$. Recall that $\PP_{\cD}$ is the joint distribution of the data collecting process. By specializing Theorem \ref{thm:regret_upper_bound_general} to the linear MDP, we have  
\$
\text{SubOpt}\big(\pess (\cD);x \big) &  \leq 2\sum_{h=1}^H\EE_{\pi^*}\big[ \Gamma_h(s_h,a_h) \biggiven s_1=x\big]   \notag \\
& =2 \beta \sum_{h=1}^H\EE_{\pi^*}\Bigl[ \bigl( \phi(s_h,a_h)^\top \Lambda_h^{-1}\phi(s_h,a_h)\bigr) ^{1/2} \Biggiven s_1=x\Bigr]
\$
for all $x \in \cS$ under $\cE$ defined in Equation \eqref{eq:def_event_eval_err_general}. Here the last equality follows from Equation \eqref{eq:w05}. Therefore, we conclude the proof of Theorem \ref{thm:regret_upper_linear}. 
\end{proof}

\subsection{Minimax Optimality of PEVI} \label{sec:minimax_sketch}

In this section, we sketch the proof of Theorem \ref{thm:lower}, which 
establishes the minimax optimality of Theorem \ref{thm:regret_upper_linear}
for the linear MDP. 
 Specifically, in Section \ref{sec:lower_minimax_elements}, we construct a class $\mathfrak{M}$ of linear MDPs and a worst-case dataset $\cD$, while in Section \ref{sec:lower_minimax_sketch}, we prove Theorem \ref{thm:lower} via the information-theoretic lower bound. 
 
%Focusing on MDPs in $\mathfrak{M}$, we provide a refined local optimality result of pessimism in Section \ref{sec:minimax_refined_upper}.

%In Section \ref{sec:lower_minimax_elements}, we construct the elements in Theorem \ref{thm:lower}, whose validity is verified in Section \ref{sec:minimax_validity_sketch}.
%A sketch of the construction of hard instances for the lower bound is provided in Section \ref{sec:lower_minimax_sketch}.

\subsubsection{Construction of a Hard Instance} \label{sec:lower_minimax_elements}

In the sequel, we construct a class $\mathfrak{M}$ of linear MDPs and a worst-case dataset $\cD$, which is compliant with 
the underlying MDP as defined in Definition \ref{def:comp}. 
 
\vskip4pt
\noindent{\bf Linear MDP:}
We define the following class of linear MDPs
\#\label{eq:define_hard_class}
\mathfrak{M}=\big\{M(p_1,p_2,p_3):p_1,p_2,p_3\in [1/4, 3/4] ~\text{with} ~ p_3 = \min\{p_1, p_2\} \big\},
\#
where $M(p_1,p_2,p_3)$ is an episodic MDP with the horizon $H\geq 2$, state space $\cS=\{x_0,x_1,x_2\}$, and action space $\cA = \{b_j\}_{j=1}^A$ with $|\cA|=A \geq 3$. In particular, we fix the initial state as $s_1 = x_0$. For the transition kernel, at the first step $h = 1$, we set 
\#\label{eq:w666}
&\cP_1(x_1\given x_0,b_1) = p_1, \quad \cP_1(x_2\given x_0,b_1) = 1-p_1 ,\notag\\
&\cP_1(x_1\given x_0,b_2) = p_2,\quad \cP_1(x_2\given x_0,b_2) = 1-p_2, \notag\\
&\cP_1(x_1\given x_0,b_j) = p_3,\quad \cP_1(x_2\given x_0,b_j) = 1-p_3,\quad \text{for all}~~ j\in \{ 3, \ldots, A\}.
\#
Meanwhile, at any subsequent step $h \in \{2, \ldots, H\}$, we set
$$
\cP_h(x_1\given x_1,a)= \cP_h(x_2\given x_2,a) = 1,\quad  \text{for all}~~ a\in \cA.
$$
In other words, $x_1, x_2 \in \cS$ are the absorbing states. Here $\cP_1(x_1\given x_0,b_1)$ abbreviates $\cP_1(s_2 = x_1\given s_1 = x_0, a_1 = b_1)$. For the reward function, we set 
\#\label{eq:w888}
&r_1 (x_0, a) = 0, \quad \text{for all}~~a\in \cA,\notag\\
&r_h(x_1, a) = 1, \quad r_h(x_2, a) = 0,\quad \text{for all}~~a\in \cA,~ h \in \{2, \ldots, H\}.
\#
See Figure \ref{fig:mdp_illu} for an illustration. Note that $M(p_1,p_2,p_3)$ is a linear MDP, which is defined in Definition \ref{assump:linear_mdp} with the dimension $d = A+2$. To see this, we set the corresponding feature map $\phi \colon \cS \times \cA  \rightarrow \RR^{d}$ as 
\#\label{eq:lower_phi}
& \phi(x_0,b_j)=(\be_j,0, 0 )\in \RR^{A+2},\quad   \text{for all}~~ j \in [A], \notag \\
& \phi(x_1,a) = (\mathbf{0}_A,1, 0)\in \RR^{A+2},\quad \text{for all}~~ a\in \cA,\notag \\
 & \phi(x_2,a) = (\mathbf{0}_A,  0, 1)\in \RR^{A+2},\quad \text{for all}~~ a\in \cA,
\#
where $\{\be_j\}_{j=1}^A$ and ${\bf 0}_A $ are the canonical basis and zero vector in $\RR^A$, respectively.  
%It can be verified that 
%each MDP  $M(p_1, p_2, p_3) \in \mathfrak{M}$ 
%is a linear MDP with feature map  $\phi$.

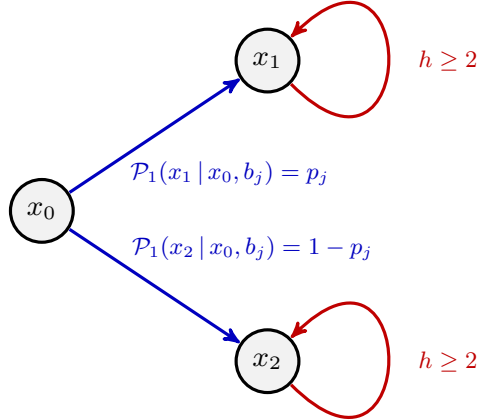
\begin{figure} [hpb]
\centering
\begin{tikzpicture}[->,>=stealth', very thick, main node/.style={circle,draw}]

\node[main node, text=black, circle, draw=black, fill=black!5, scale=1.2] (1) at  (0,0) {\small $x_0$};
\node[main node, text=black, circle, draw=black, fill=black!5, scale=1.2] (2) at  (3,2) {\small $x_1$};
\node[main node, text=black, circle, draw=black, fill=black!5, scale=1.2] (3) at  (3,-2) {\small $x_2$};

\draw[->] (1) edge [draw=blue!75!black] (2);
\draw[->] (1) edge [draw=blue!75!black] (3);
\draw[->] (2) edge [looseness=10, out= 315, in=45, draw=red!75!black] (2);
\draw[->] (3) edge [looseness=10, out= 315, in=45, draw=red!75!black] (3);

\node[text=blue!75!black] at  (2.5,0.5) {\small $\cP_1(x_1\given x_0,b_j)=p_j$};
\node[text=blue!75!black] at  (2.8,-0.5) {\small $\cP_1(x_2\given x_0,b_j)=1-p_j$};

\node[text=red!75!black] at  (5.4,2) {\small $h\geq 2$};
\node[text=red!75!black] at  (5.4,-2) {\small $h\geq 2$};

\end{tikzpicture}
\vspace*{-10mm}
\caption{An illustration of the episodic MDP $\cM = M(p_1,p_2,p_3)\in \mathfrak{M}$ with the state space~$\cS = \{ x_0, x_1, x_2\}$ and action space $\cA= \{ b_j \}_{j=1}^A$. Here we fix the initial state as $s_1 = x_0$, where the agent takes the action $a \in \cA$ and transits into the second state 
$s_2\in \{x_1,x_2\}$. 
At the first step $h=1$, the transition probability satisfies $\cP_1(x_1\given x_0,b_j)=p_j$ and $\cP_1(x_2\given x_0,b_j)=1-p_j$ for  all $j \in [A]$, 
where for notational simplicity, we define $p_j = p_3$ for all $j\in \{ 3, \ldots, A\}$. 
Also, $x_1, x_2 \in \cS$ are the absorbing states. Meanwhile, the reward function satisfies    $r_h (x_0, a) = 0$, $r_h(x_1,a)=1$,  and $r_h(x_2,a)=0$ for all $a\in \cA$ and $h \in [H]$.}
\label{fig:mdp_illu}
\end{figure}

%In this case, a policy $\pi$ is a distribution over action set $\{a_j\}_{j=1}^A$, with $\pi(a_j)$ the probability of taking action $a_j$, $j\in[A]$. 
%For simplicity write $p_j=p_3$, $j\geq 3$. 

As $x_1, x_2 \in \cS$ are the absorbing states, the optimal policy $\pi_1^*$ at the first step $h=1$ is a deterministic policy, which by Equation \eqref{eq:w888} selects the action $a \in \cA$ that induces the largest transition probability into the desired state $x_1$. In other words, at the first step $h=1$, we have
\#\label{eq:w996}
\pi_1^*(b_{j^*} \given x_0) = 1, \quad\text{where~~} j^* = \argmax _{ j \in \{1,2\}} p_j .
\#
Here we assume without loss of generality $p_1 \neq p_2$ in Equation \eqref{eq:w666}. Meanwhile, at any subsequent step $h \in \{2, \ldots, H\}$, an arbitrary policy $\pi_h$ is optimal, as the action $a \in \cA$ selected by $\pi_h$ does not affect the transition probability. Therefore, for any policy $\pi = \{\pi_h\}_{h=1}^H$,  the suboptimality of  $\pi$ for the linear MDP $\cM=M(p_1,p_2,p_3)$ takes the form  
\# \label{eq:suboptimality_hard_instance}
\text{SubOpt}(\cM,\pi;x_0) =  p_{j^*} \cdot (H-1) - \sum_{j=1}^A p_j\cdot   \pi_1(b_j\given x_0)\cdot (H-1),
\#
where for notational simplicity, we define $p_j = p_3 $ for all $j\in \{ 3, \ldots, A\}$. Here with an abuse of notation, we incorporate the explicit dependency on the underlying MDP $\cM\in \mathfrak{M}$ into the suboptimality $\text{SubOpt}(\pi;x_0)$.

\vskip4pt
\noindent{\bf Dataset:} We specify the worst-case data collecting process $\PP_{\cD}$ as follows. Given a linear MDP $\cM\in\mathfrak{M}$, the dataset $\cD=\{(x_h^\tau, a_h^\tau  , r_h^\tau )\}_{\tau, h = 1}^{K, H}$ consists of $K$ trajectories starting from the same initial state $x_0$, that is, $x_1 ^\tau = x_0$ for all $\tau \in [K]$. 
The initial  actions $\{a_1^\tau \}_{\tau = 1}^K$ 
are predetermined. The subsequent states $\{ x_h^\tau  \}_{\tau \in [K], h \geq 2} $
are sampled from the underlying MDP $\cM = M(p_1,p_2,p_3)$, while 
the subsequent actions $\{ a_h^\tau  \}_{\tau \in [K], h \geq 2}   
$ 
are arbitrarily chosen, as they do not affect the state transitions. The state transitions in the different trajectories are independent. 
The immediate reward $r_h^\tau$ satisfies 
$r_h^\tau = r_h (x_h^\tau, a_h^\tau) $. 
Note that 
such a dataset $\cD$ satisfies Assumption \ref{assump:data_generate}, that is, $\cD$ is compliant with the linear MDP $\cM \in \mathfrak{M}$.  

% Actions at step $h=1$ are  for each trajectory, so that $a_j$ is taken $n_j$ many times. Actions at steps $h\geq 2$ are arbitrary. 
% In the $\tau$-th trajectory, at each step $h\in [H]$, after taking the action $a_h^{\tau}$, the state evolves from $x_h^\tau$ to $x_{h+1}^\tau$ according to the true MDP $\cM\in \mathfrak{M}$ and receive the reward $r_h^\tau$. The transitions of all trajectories are independent. It's straightforward to see that the dataset $\cD$ is compliant with true MDP $\cM$.

We define 
\#\label{eq:w88888}
n_j = \sum_{\tau =1}^K \ind \{ a_1^\tau = b_j\},\quad\{ \kappa _j^{i}\}_{i = 1}^{n_j} = \bigl\{ r_2^\tau \colon a_1^\tau = b_j~\text{with}~ \tau \in [K] \bigr\},\quad \text{for all}~~j \in [A].
\# 
In other words, assuming that $ 1\leq \tau_1< \tau_2< \cdots < \tau_{n_j} \leq K $ are the episode indices such that $a_1^{\tau_{i}} = b_j$ for all $i \in [n_j]$, we define $\kappa_{j}^i = r_2 ^{\tau_i}$ for all $j \in [A]$.
By such a construction,  $\{ \kappa _j^{i}\}_{i,j = 1}^{n_j, A}$   are the realizations of $K$ independent Bernoulli random variables, which satisfy
\#\label{eq:bernoulli_kappa}
\EE_{\cD} [ \kappa _{j}^i ] = p_j, \quad \text{for all}~~i \in [n_j],~j \in [A].
\#
Note that knowing the value of the immediate reward $r_2^\tau $ is sufficient for determining the value of the second state $x_{2}^\tau$.
Meanwhile, recall that $x_1, x_2 \in \cS$ are the absorbing states. 
Therefore, for learning  the optimal policy $\pi^*$, the original dataset
$\cD$ contains the same information as the reduced dataset
$\cD_1 = \{(x_1^\tau, a_1^\tau, x_2^\tau, r_2^\tau) \}_{\tau=1}^K $,
where the randomness only comes from the state transition at the first step $h=1$ of each trajectory $\tau \in [K]$. 
Correspondingly, the probability of observing the dataset $\cD_1$ takes the form 
\#\label{eq:data_likelihood}
\PP_{\cD \sim \cM   }  ( \cD_1  )
& = \prod  _{\tau =1}^K \PP _{\cM} \bigl (r_2(s_2, a_2) = r_2^\tau \biggiven s_1 = x_1^\tau, a_1 = a_1^\tau \big ) \notag \\
& = \prod_{j=1}^A \Big(\prod_{i=1}^{n_j} p_j^{\kappa_{j}^i } \cdot (1-p_j)^{1-\kappa_{j}^{i} }\Big) = \prod_{j=1}^A\Big(p_j^{\sum_{i=1}^{n_j} \kappa _{j}^i} \cdot (1-p_j)^{n_j-\sum_{i=1}^{n_j}\kappa _{j}^i }\Big).
\#
Here $\PP_{\cD\sim \cM}$ denotes the randomness of the dataset $\cD$, which is compliant with the underlying MDP $\cM = M(p_1,p_2,p_3)$, while 
 $\PP_{\cM} $  denotes the randomness of the immediate rewards and state transitions. In the   
      second equality, we apply the definition of $\{ \kappa _j^{i}\}_{i = 1}^{n_j}$ in Equation \eqref{eq:w88888}.  
By such a definition, $\PP_{\cD \sim \cM} (\cD_1)$ in Equation \eqref{eq:data_likelihood} is the likelihood of the linear MDP $\cM \in \mathfrak{M}$ given the reduced dataset $\cD_1$ (or equivalently, the original dataset $\cD$, assuming that the subsequent actions $\{ a_h^\tau  \}_{\tau \in [K], h \geq 2}$ are predetermined). 
%As will see later, given the  dataset $\cD$, 
%the hardness of identifying  
%the underlying MDP $\cM$
%with statistical power captures the fundamental hardness of offline 
%RL. 

% Denote the rewards at step $h=2$ for all trajectories with action $a_j$ as $r_{j,1},\dots,r_{j,n_j}$. Since states $x_1,x_2$ are absorbing, the transitions in $\cD$ for $h\geq 2$ are the same for all MDPs $\cM\in \mathfrak{M}$. With pre-fixed actions at step $h=1$, the randomness is inherited from the transitions at step $h=1$. Therefore, the joint distribution of $\cD$ is
% $$
% \PP_{\cD }(\cD) = \prod_{j=1}^A\Big(\prod_{i=1}^{n_j} p_j^{r_{j,i}}(1-p_j)^{1-r_{j,i}}\Big) = \prod_{j=1}^A\Big(p_j^{\sum_{i=1}^{n_j} r_{j,i}}(1-p_j)^{n_j-\sum_{i=1}^{n_j}r_{j,i}}\Big).
% $$
% Here we use $\PP_{\cD\sim \cM}$ to indicate the compliance of $\cD$ to MDP $\cM=(p_1,p_2,p_3)\in \mathfrak{M}$.

\subsubsection{Information-Theoretic Lower Bound} \label{sec:lower_minimax_sketch}

The proof of Theorem \ref{thm:lower} is based on the Le Cam method \citep{le2012asymptotic,yu1997assouad}.
Specifically, we construct two linear MDPs $\cM_1, \cM_2 \in \mathfrak{M}$, where the class $\mathfrak{M}$  of linear MDPs is defined in Equation \eqref{eq:define_hard_class}. Such a construction ensures that (i) the distribution of the dataset $\cD$, which is compliant with the underlying MDP, is similar across $\cM_1, \cM_2 \in \mathfrak{M}$, and (ii) the suboptimality of any policy $\pi$, which is constructed based on the dataset, is different across $\cM_1, \cM_2 \in \mathfrak{M}$. In other words, it is hard to distinguish $\cM_1, \cM_2 \in \mathfrak{M}$ based on $\cD$, while $\pi$ obtained from $\cD$ can not achieve a desired suboptimality for $\cM_1, \cM_2 \in \mathfrak{M}$ simultaneously. Such a construction captures the fundamental hardness of offline RL for the linear MDP. 

For any $p, p^* \in [1/4, 3/4]$, where $p < p^*$, we set 
\#\label{eq:w999}
\cM_1 = M(p^*, p , p),\quad \cM_2 = M(p, p^*, p).
\#
Based on $\cD$, whose likelihood is specified in Equation \eqref{eq:data_likelihood}, we aim to test whether the underlying MDP is $\cM_1$ or $\cM_2$. The following lemma establishes a reduction from learning the optimal policy $\pi^*$ to testing the underlying MDP $\cM \in \mathfrak{M}$. Recall that for any $\ell \in \{1,2\}$, $n_\ell$ is defined in Equation \eqref{eq:w88888}.

\begin{lemma}[Reduction to Testing]
  For the dataset $\cD$ specified in Section \ref{sec:lower_minimax_elements}, the output $ \texttt{Algo}(\cD)$ of any algorithm  satisfies 
\$
&\max_{\ell \in \{ 1,2\} } \sqrt{n_\ell}\cdot \EE_{\cD\sim \cM_\ell}\Big [\text{SubOpt}\big (\cM_\ell,\texttt{Algo}(\cD);x_0 \big ) \Bigr ] \\
&\qquad \geq  \frac{\sqrt{n_1n_2}}{\sqrt{n_1}+\sqrt{n_2}} \cdot (p^*-p) \cdot (H-1) \cdot  \Big(\EE_{\cD\sim \cM_1}\bigl[1-{\pi}_1(b_1\given x_0)\bigr] + \EE_{\cD\sim \cM_2}\bigl[ {\pi}_1(b_1\given x_0)\bigr]\Big), 
\$
where $\pi=\{\pi_h\}_{h=1}^H = \texttt{Algo}(\cD)$. For any $\ell \in \{1,2\}$, $\EE_{\cD\sim \cM_\ell}$ is taken with respect to the randomness of  $\cD$, which is compliant with the underlying MDP $\cM_\ell$.
\label{lem:minimax_lower_split}
\end{lemma}

\begin{proof}[Proof of Lemma \ref{lem:minimax_lower_split}]
See Appendix \ref{sec:appendix:proof_minimax_lower_split} for a detailed proof.
\end{proof}

As specified in Equation \eqref{eq:w996}, for the underlying MDP $\cM_1$, the optimal policy $\pi_1^*$ takes the initial action $b_1$ with probability one at the initial state $x_0$, while for $\cM_2$, $\pi_1^*$ takes $b_2$ with probability one at $x_0$. We consider the following hypothesis testing problem
\#\label{eq:testing}
H_0 \colon \cM = \cM_1 \quad \text{versus} \quad H_1 \colon \cM = \cM_2
\#
based on the dataset $\cD$.
For such a problem, any test function 
$\psi $ is a binary map such that $\psi(\cD) = 0$ means the null hypothesis $H_0$ is accepted, while $\psi(\cD) = 1$ means $H_0$ is rejected. 
For the output $\pi=\{\pi_h\}_{h=1}^H = \texttt{Algo}(\cD)$ of any algorithm, we define 
\#\label{eq:define_algo_test}
\psi_{\texttt{Algo} } (\cD) =  \ind \{ a\neq  b_1 \},\quad \text{where~~}a \sim \pi_1 (\cdot \given x_0).
\# 
Correspondingly, the risk of the (randomized) test function $\psi_{\texttt{Algo}}$ takes the form 
\#\label{eq:risk_test_algo}
\textrm{Risk}(\psi_{\texttt{Algo}}) &  = \EE_{\cD \sim \cM_1} \bigl [ \ind\{ \psi_{\texttt{Algo}} (\cD) = 1 \} \bigr ] + \EE_{\cD \sim \cM_2} \bigl [ \ind\{ \psi_{\texttt{Algo}} (\cD) = 0 \}  \bigr ] \notag \\
& = \EE_{\cD\sim \cM_1}\bigl[1-{\pi}_1(b_1\given x_0)\bigr] + \EE_{\cD\sim \cM_2}\bigl[ {\pi}_1(b_1\given x_0)\bigr].
\#
Therefore, Lemma \ref{lem:minimax_lower_split} lower bounds the suboptimality of any policy $\pi=\{\pi_h\}_{h=1}^H = \texttt{Algo}(\cD)$ by the  risk of a (randomized)  test function, which is induced by $\pi$, for the corresponding hypothesis  testing problem defined in Equation 
\eqref{eq:testing}. Such an approach mirrors the Le Cam method  \citep{le2012asymptotic,yu1997assouad} for establishing the minimax optimality in statistical estimation. In particular, a careful choice of $p,p^* \in [1/4, 3/4]$ leads to the information-theoretic lower bound established in Theorem \ref{thm:lower}. See Appendix \ref{sec:appendix:lower} for a detailed proof.

%which also involves a more detailed analysis of the upper bound provided by Theorem \ref{thm:regret_upper_linear}. 
%We defer the proof to the appendix for brevity. 
%See Appendix \ref{sec:appendix:lower} for a complete proof.

% The term $\EE_{\cD\sim M_1}[1-{\pi}(a_1)] + \EE_{\cD\sim \cM_2}[ {\pi}(a_1)]$ in Lemma \ref{lem:minimax_lower_split} can be reduced to hypothesis testing problem, which relates the lower bound to some distance measure between $\PP_{\cD\sim \cM_1}$ and $\PP_{\cD\sim \cM_2}$. A 

%\end{flushleft}

%!TEX root =paper.tex

%\section{Conclusion} \label{sec:conclusion}

\newpage
\bibliographystyle{ims}
\bibliography{reference}

\newpage 

\appendix

% !TEX root = paper.tex

\section{Proofs of Suboptimality Decomposition }\label{sec:appendix_proof_decomp}
% \subsection{Proof of Suboptimality Decomposition}
%In this section,  we provide the proof of Lemma \ref{lem:reg_decomp}, which
%decomposes the suboptimality of any policy into three terms.
 % relies on an extended value difference lemma in \cite{efroni2020optimistic}.

\begin{proof}[Proof of Lemma \ref{lem:reg_decomp}]
By the definition in Equation \eqref{eq:def_regret},  the suboptimality of the policy $\hat \pi$ given any initial state $x \in \cS$ can be decomposed as
\#\label{eq:decomp_terms}
\text{SubOpt}(\hat\pi;x) = \underbrace{\big (  V_1^{\pi^*}(x) - \hat{V}_1(x) \big ) }_{\displaystyle \text{(i)}} + \underbrace{\big ( \hat{V}_1(x) - V_1^{\hat\pi}(x) \big ) }_{\displaystyle \text{(ii)}},
\#
where $\{  \hat{V}_h \}_{h = 1}^{H}$  are the estimated   value functions constructed by the meta-algorithm.  
Term (i) in Equation \eqref{eq:decomp_terms}  is  the difference between the estimated value function    $\hat{V}_1$  and the optimal    value function $V_1^{\pi^*}$, while  term (ii) is the difference between  $\hat{V}_1$ and the value function $V_1^{\hat\pi}$ of $\hat\pi$.
To further decompose terms (i) and (ii), we  utilize the following lemma, which is obtained  from \cite{cai2020provably}, to characterize the difference between an estimated  value function   and the value function of a policy.

\begin{lemma}[Extended Value Difference \citep{cai2020provably}]
	Let $\pi = \{ \pi _h \}_{h =1}^H $ and $\pi' = \{ \pi_h' \}_{ h = 1}^H  $ be any two policies and let $\{ \hat Q_h \}_{h=1}^H $ be any estimated Q-functions. 
	For all $h \in [H]$, we define the estimated value function $\hat V_h  \colon \cS\rightarrow \RR$  by setting $\hat V_h (x) = \langle \hat Q_h (x, \cdot ), \pi_h (\cdot \given x ) \rangle_{\cA}$ for all $x \in \cS$. 
	 For all $x \in \cS$, we have 
\$
	\hat{V}_1(x) - V_1^{\pi' }(x) &= \sum_{h=1}^H \EE_{\pi' }\big[ \langle \hat{Q}_h (s_h,\cdot) , \pi_h(\cdot\given s_h) - \pi'_h(\cdot\given s_h)\rangle_{\cA } \biggiven s_1=x\big]\\
	&\qquad + \sum_{h=1}^H\EE_{\pi' }\big[     \hat{Q}_h (s_h,a_h)  - (\BB_h \hat{V}_{h+1} )(s_h,a_h)  \biggiven s_1=x \big],
\$
	where $\EE_{\pi' } $ is taken with respect to the trajectory generated by $\pi'$, while $\BB_h$ is the Bellman operator defined in Equation \eqref{eq:def_bellman_op}.
	\label{lem:value_diff}
\end{lemma}

\begin{proof}
	See Section B.1 in \cite{cai2020provably} for a detailed proof. 
\end{proof}

Applying  Lemma \ref{lem:value_diff} with $\pi=\hat\pi$, $\pi'=\pi^*$, and $\{ \hat Q_h \}_{h= 1}^H $ being the estimated Q-functions constructed by the meta-algorithm, we have 
\$
\hat{V}_1(x) - V_1^{\pi^*}(x) &= \sum_{h=1}^H \EE_{\pi^*}\big[ \langle \hat{Q}_h(s_h,\cdot) , \hat\pi_h(\cdot\given s_h) - \pi^*_h(\cdot\given s_h) \rangle_{\cA} \biggiven s_1=x\big] \notag \\
& \qquad + \sum_{h=1}^H   \EE_{\pi^*}\big[      \hat{Q}_h(s_h,a_h)-  ( \BB_h \hat{V}_{h+1}) (s_h,a_h) \biggiven s_1=x\big],
\$ 
where $\EE_{\pi^*} $  is taken with respect to the trajectory generated by $\pi^*$.
By the definition of the model evaluation error 
  $\iota _h $  in Equation \eqref{eq:def_iota}, 
  we have 
\$
 V_1^{\pi^*}(x) -  \hat{V}_1(x) &= \sum_{h=1}^H \EE_{\pi^*}\big[ \langle \hat{Q}_h(s_h,\cdot) ,   \pi^*_h(\cdot\given s_h)- \hat\pi_h(\cdot\given s_h)  \rangle_{\cA} \biggiven s_1=x\big] +   \sum_{h=1}^H   \EE_{\pi^*}\big[  \iota_h  (s_h,a_h) \biggiven s_1=x\big].
\$ 
Similarly, applying Lemma \ref{lem:value_diff}  with $\pi=\pi'=\hat\pi$ and $\{ \hat Q_h \}_{h=1}^H$ being the  estimated Q-functions constructed by the meta-algorithm,   we  have 
\$
  \hat{V}_1(x)- V_1^{\hat\pi}(x)   &=  
  \sum_{h=1}^H\EE_{\hat\pi}\big[    \hat{Q}_h(s_h,a_h)- (  \BB_h \hat{V}_{h+1}) (s_h,a_h) \biggiven s_1=x\big] = - \sum_{h=1}^H\EE_{\hat\pi}\big[ \iota_h (s_h,a_h)  \biggiven s_1=x\big],
\$ 
where $\EE_{\hat\pi} $ is taken with respect to the trajectory generated by    $\hat\pi$. By Equation \eqref{eq:decomp_terms}, we conclude the proof of Lemma \ref{lem:reg_decomp}. 
\end{proof}

\section{Proofs of Pessimistic Value Iteration}

%\subsection{Suboptimality upper bound for general MDP}\label{sec:appendix:upper_bound_general}

\subsection{Proof of Lemma \ref{lem:model_eval_err}} \label{sec:app:proof_lem_model_eval}

\begin{proof}[Proof of Lemma \ref{lem:model_eval_err}]
    We first show that on the event $\cE$ defined in Equation \eqref{eq:def_event_eval_err_general}, the model evaluation errors $\{ \iota_h \}_{h=1}^H $ are 
  nonnegative. In the sequel, we assume that $\cE$ holds. Recall the construction of $\overline Q_h $ in Line \ref{alg:general_Qbar} of Algorithm~\ref{alg:pess_greedy_general} for all $h \in [H]$. 
  For all $h \in [H]$ and all $(x,a) \in \cS \times \cA$,
   if 
 $\overline{Q}_h(x,a)<0$, we have 
 \$
 \hat Q_h (x,a)  = \min \{  \overline Q_h (x,a) , H - h + 1\} ^{+} = 0.
 \$ 
 By the definition of $\iota_h$ in Equation \eqref{eq:def_iota}, we have 
 \$ %\label{eq:elem_bound1.1}
 \iota_h(x,a)= (\BB_h\hat{V}_{h+1})(x,a) - \hat{Q}_h(x,a) = (\BB_h\hat{V}_{h+1})(x,a)\geq 0,
 \$
 as $r_h$ and $\hat{V}_{h+1}$ are nonnegative.
 Otherwise, if $\overline {Q}_h (x,a) \geq 0$, we have
 \$ %\label{eq:elem_bound1}
 \hat Q_h (x,a)  = \min \{  \overline Q_h (x,a) , H - h + 1\} ^{+} \leq  \overline Q_h (x,a).
 \$
As $\{ \Gamma_h \}_{h = 1}^H $ are $\xi$-uncertainty quantifiers, which are defined in Definition \ref{def:uncertainty_quantifier}, we have 
\begin{equation*}
\begin{split}
\iota_h(x,a) &
\geq (\BB_h\hat{V}_{h+1}) (x,a) - \overline{Q}_h(x,a) =  (\BB_h\hat{V}_{h+1}) (x,a) - (\hat{\BB}_h\hat{V}_{h+1}) (x,a) + \Gamma_{h}(x,a) \geq 0.
\end{split}
\end{equation*}
Here the last inequality follows from the definition of  $\cE$ in Equation \eqref{eq:def_event_eval_err_general}. Therefore, we conclude the proof of $\iota_h(x,a)\geq 0$ for all $h\in [H]$ and all $(x,a)\in \cS\times\cA$ on $\cE$.

It remains to establish  the upper bound in Equation \eqref{eq:model_eval_err_bound}. For all $h\in [H]$ and all $(x,a) \in \cS \times \cA$, combining the definition of event $\cE$ in Equation \eqref{eq:def_event_eval_err_general} as well as the construction of $\overline Q_h$ in Line \ref{alg:general_Qbar} of Algorithm~\ref{alg:pess_greedy_general} gives
 \$ %\label{eq:elem_bound2}
 \overline{Q}_h(x,a)= (\hat\BB_h\hat{V}_{h+1}) (x,a)- \Gamma_h(x,a)\leq (\BB_h\hat{V}_{h+1} ) (x,a)\leq H-h+1, 
 \$
 where the first inequality follows from the triangle inequality, while the second inequality follows from the fact that $r_h \in [0,1]$ and $\hat V_{h+1}\in [0,H-h]$.  
Hence, we have %combining Equations \eqref{eq:elem_bound1} and \eqref{eq:elem_bound2}, 
\$
\hat Q_h (x,a) = \min\{ \overline  Q_h (x,a), H-h+1\}^+= \max \{ \overline  Q_h (x,a), 0  \} \geq \overline Q_h(x,a),
\$
which by the definition of $\iota_h$ in Equation \eqref{eq:def_iota} implies
\begin{equation*}
\begin{split}
\iota_h(x,a) &=  (\BB_h\hat{V}_{h+1} ) (x,a) - \hat{Q}_h(x,a)\leq (\BB_h\hat{V}_{h+1}) (x,a) - \overline{Q}_h(x,a) \\
&= (\BB_h\hat{V}_{h+1} ) (x,a) - (\hat\BB_h\hat{V}_{h+1} ) (x,a) +\Gamma_h(x,a) \leq 2\Gamma_h(x,a).
\end{split}
\end{equation*}
Here the last inequality follows from the definition of  $\cE$ in Equation \eqref{eq:def_event_eval_err_general}. Therefore, we complete the proof of $\iota_h(x,a)\leq 2\Gamma_h(x,a)$ for all $h\in [H]$ and all $(x,a)\in \cS\times\cA$ on $\cE$.

In summary, we conclude that on $\cE$,
$$
0\leq \iota_h(x,a) \leq 2\Gamma_h(x,a),\qquad \forall (x,a)\in \cS\times \cA, ~\forall h \in [H]. 
$$
Therefore, we conclude the proof of Lemma \ref{lem:model_eval_err}.
\end{proof}

%In the following we prove Proposition \ref{prop:linear_MDP_uncertainty}, which establishes the validity of $\Gamma_h(\cdot,\cdot)$ specified in Equation \eqref{eq:eval_step_weight} as a $\xi$-uncertainty function. The result is basically based on a uniform concentration property of the ridge-tyep least-square value iteration.

\subsection{Proof of Lemma \ref{lemma:linear_MDP_uncertainty}} \label{sec:app:proof_prop_linear}

\begin{proof}[Proof of Lemma  \ref{lemma:linear_MDP_uncertainty}]
%\begin{flushleft}
It suffices to show that under Assumption \ref{assump:data_generate}, the event $\cE$ defined in Equation \eqref{eq:def_event_eval_err_general} satisfies $\PP_{\cD}(\cE)\geq 1-\xi$  with the $\xi$-uncertainty quantifiers $\{\Gamma_h\}_{h=1}^H$ defined in Equation \eqref{eq:w05}.   
To this end, we upper bound the difference between $ (\BB_h \hat V_{h+1}) (x,a)$ and $(\hat \BB_h \hat V_{h+1}) (x,a)$ for all $h \in [H]$ and all $(x,a) \in \cS \times \cA$, where the Bellman operator $\BB_h$ is defined in Equation  \eqref{eq:def_bellman_op}, the estimated Bellman operator $\hat \BB_h$ is defined in Equation \eqref{eq:wlin}, and the estimated value function $\hat V_{h+1}$ is constructed in Line \ref{alg:linear_Vhat} of Algorithm \ref{alg:pess_greedy}.  
 
For any function $V\colon \cS \rightarrow [ 0, H]$, Definition \ref{assump:linear_mdp} ensures that $\PP_h V$ and $\BB_h V$ are linear in the feature map $\phi$ for all $h\in [H]$.
To see this, note that Equation \eqref{eq:w07} implies
\$%\label{eq:pv_linear}
(\PP _h V)(x,a) = \Big\la \phi(x,a) , \int_{\cS} V(x') \mu_h (   x') \ud x' \Big \ra , \qquad \forall (x,a) \in \cS \times \cA, ~\forall h \in [H].
\$
Also, Equation  \eqref{eq:w07} ensures that the expected reward is linear in $\phi$ for all $h\in [H]$, which implies
\#\label{eq:bv_linear}
(\BB _h V)(x,a) = \la \phi(x,a) ,\theta_h\ra + \Big\la \phi(x,a) , \int_{\cS} V(x') \mu_h (   x') \ud x' \Big \ra, \qquad \forall (x,a) \in \cS \times \cA, ~\forall h \in [H].
\#
Hence, there exists an unknown vector  $w_h \in \RR^d$ such that 
\#\label{eq:www}
(\BB_h \hat V_{h+1}) (x, a ) = \phi(x, a)^\top w_h, \qquad \forall (x,a) \in \cS \times \cA, ~\forall h \in [H].
\# 
Recall the definition of $\hat w_h$ in Equation \eqref{eq:w18} and the construction of $\hat \BB_h \hat V_{h+1}$ in Equation \eqref{eq:wlin}.
The following lemma upper bounds the norms of $w_h$ and $\hat w_h$, respectively.

\begin{lemma}[Bounded Weights of Value Functions] Let $V_{\max} > 0$ be an absolute constant. For any function $V:\cS\to [0,V_{\max}]$ and any $h \in [H]$, we have 
  \$
  \|w_h\|\leq (1+V_{\max})\sqrt{d},\qquad \|\hat{w}_h\|\leq H\sqrt{Kd/\lambda},
  \$
  where $w_h$ and $\hat w_h$ are defined in Equations \eqref{eq:www} and \eqref{eq:w18}, respectively.
  \label{lem:bound_weight_of_bellman}
  \end{lemma}
  \begin{proof}[Proof of Lemma \ref{lem:bound_weight_of_bellman}]
    For all $h \in [H]$, Equations \eqref{eq:bv_linear} and \eqref{eq:www} imply
    \$
 w_h  =  \theta_h + \int_{\cS} V(x') \mu_h(  x') \ud x' .
\$
By the triangle inequality and the fact that $\|  \mu_h(\cS ) \|\leq \sqrt{d}$ in Definition \ref{assump:linear_mdp} with the notation $\|\mu_h(\cS)\| = \int_{\cS}\|\mu_h(x')\|\ud x'$, we have 
\#\label{eq:w_norm}
  \|w_h\|   &\leq \|\theta_h\| + \Big\|\int_{\cS} V(x')\mu_h(  x') \ud x' \Big\| \leq  \|\theta_h\| + \int_{\cS} \|V(x')\mu_h(  x') \|\ud x'  \notag \\
  &\leq \sqrt{d} + V_{\max} \cdot  \|  \mu_h(\cS ) \|  \leq (1 + V_{\max}) \sqrt{d},
\# 
  where the third inequality follows from 
  the fact that  $V\in [0,V_{\max}]$. 
  
  Meanwhile,  by the definition of  $\hat{w}_{h}$ in  Equation \eqref{eq:w18} and the triangle inequality, we have 
  \begin{equation*}
  \begin{split}
  \|\hat{w}_h\| &= \Big\|     \Lambda_h  ^{-1} \Big( \sum_{\tau=1}^{K} \phi(x_h^\tau,a_h^\tau) \cdot \big(r_h^\tau + \hat{V}_{h+1}(x_{h+1}^\tau)\big)  \Big)\Big\| \\
  &\leq \sum_{\tau=1}^K \big\|    \Lambda_h ^{-1}   \phi(x_h^\tau,a_h^\tau)  \cdot \big(r_h^\tau + \hat{V}_{h+1}(x_{h+1}^\tau)\big)\big\|.
  \end{split}
  \end{equation*}
Note that $|r_h^\tau + \hat{V}_{h+1}(x_{h+1}^\tau) | \leq  H $, which follows from the fact that $r_h^\tau\in [0,1]$ and $\hat{V}_{h+1}\in[0, H-1]$ by Line \ref{alg:linear_Vhat} of Algorithm \ref{alg:pess_greedy}. 
Also, note that $\Lambda_h  \succeq \lambda \cdot I$, which follows from the definition of $\Lambda_h$ in Equation \eqref{eq:w18}.
Hence, we have 
  \$
  \|\hat{w}_h\| 
  \leq H \cdot \sum_{\tau=1}^K \|    \Lambda_h ^{-1} \phi(x_h^\tau,a_h^\tau) \| &= H \cdot \sum_{\tau=1}^K \sqrt{\phi(x_h^\tau,a_h^\tau)^\top  \Lambda_h ^{-1/2} \Lambda_h^{-1}  \Lambda_h ^{-1/2} \phi(x_h^\tau,a_h^\tau) }\\
  &\leq \frac{H}{\sqrt{\lambda}} \cdot \sum_{\tau=1}^K  \sqrt{ \phi(x_h^\tau, a_h^\tau )^\top  \Lambda _h  ^{-1} \phi(x_h^\tau, a_h^\tau ) } ,
\$ 
  where the last inequality follows from the fact that  $\|\Lambda_h^{-1}\|_{\oper}\leq \lambda^{-1}$. Here  $\| \cdot \|_{\oper}$ denotes   the matrix operator norm.
  By the Cauchy-Schwarz inequality, we have 
\# \label{eq:hatw_norm}
&  \|\hat{w}_h\| 
 \leq H  \sqrt{K/\lambda} \cdot \sqrt{ \sum_{\tau=1}^K\phi(x_h^\tau,a_h^\tau)^\top  \Lambda_h^{-1}\phi(x_h^\tau,a_h^\tau) } =  H  \sqrt{K/\lambda} \cdot \sqrt{\tr\Bigl (\Lambda_h^{-1} \sum_{\tau=1}^K\phi(x_h^\tau,a_h^\tau) \phi(x_h^\tau,a_h^\tau)^\top \Bigr) } \notag \\
   & \qquad =  H  \sqrt{K/\lambda} \cdot \sqrt{\tr\bigl (\Lambda_h^{-1} ( \Lambda_h  - \lambda \cdot I )\bigr ) } \leq H  \sqrt{K/\lambda} \cdot \sqrt{\tr  (\Lambda_h^{-1}   \Lambda_h  ) } = H \sqrt{ K d / \lambda } , 
 \#
  where the second equality follows from the definition of $\Lambda_h $ in Equation \eqref{eq:w18}.

  Therefore, combining Equations \eqref{eq:w_norm} and \eqref{eq:hatw_norm}, we conclude the proof of Lemma \ref{lem:bound_weight_of_bellman}. 
  \end{proof}
%\end{flushleft}

%\begin{flushleft}
We upper bound the difference between 
$\BB_h \hat V_{h+1} $ and $\hat \BB_h \hat V_{h+1}$. 
For all $h\in [H]$ and all
  $(x,a) \in \cS \times \cA$,  we have  
  \# 
    &(\BB_h\hat{V}_{h+1})(x,a) - (\hat{\BB}_h\hat{V}_{h+1})(x,a) =  \phi (x,a)^\top ( w_h - \hat w_h ) \notag \\
    &\qquad = \phi (x,a)^\top   w_h - \phi(x,a)^\top \Lambda_h^{-1} \Big( \sum_{\tau=1}^{K} \phi(x_h^\tau,a_h^\tau) \cdot \bigl (r_h^\tau + \hat{V}_{h+1}(x_{h+1}^\tau) \bigr ) \Big) \notag  \\
    & \qquad = \underbrace{\phi(x,a)^\top w_h - \phi(x,a)^\top \Lambda_h^{-1} \Big( \sum_{\tau=1}^{K} \phi(x_h^\tau,a_h^\tau) \cdot (\BB_h\hat{V}_{h+1}) (x_{h}^\tau,a_h^\tau)  \Big)}_{\displaystyle \text{(i)}} \label{eq:term1_diff}\\
    &\qquad  \qquad -  \underbrace{\phi (x,a)^\top \Lambda_h^{-1} \Big( \sum_{\tau=1}^{K} \phi(x_h^\tau,a_h^\tau) \cdot \big( r_h^\tau + \hat{V}_{h+1}(x_{h+1}^\tau) - (\BB_h\hat{V}_{h+1}) (x_{h}^\tau,a_h^\tau) \big ) \Big)}_{\displaystyle\text{(ii)}}.\notag%\label{eq:term2_diff}
  \# 
  Here the first equality follows from the definition of the Bellman operator $\BB_h$ in Equation  \eqref{eq:def_bellman_op}, the decomposition of $\BB_h$ in Equation \eqref{eq:bv_linear}, and the definition of the estimated Bellman operator $\hat \BB_h$ in Equation \eqref{eq:wlin}, while 
  the second equality follows from the definition of $\hat w_h$ in Equation \eqref{eq:w18}. 
By the triangle inequality, we have 
\$
\bigl | (\BB_h\hat{V}_{h+1})(x,a) - (\hat{\BB}_h\hat{V}_{h+1})(x,a) \bigr | \leq  |\text{(i)}| + |\text{(ii)}|. 
\$

In the sequel, we upper bound terms (i) and (ii) respectively. 
By the construction of the estimated value function $\hat{V}_{h+1}$ in Line \ref{alg:linear_Vhat} of Algorithm \ref{alg:pess_greedy}, we have $\hat{V}_{h+1}\in [0,H-1]$.
By Lemma \ref{lem:bound_weight_of_bellman}, we have $\|w_h\|\leq H\sqrt{d}$.
Hence, term (i) defined in Equation \eqref{eq:term1_diff} is upper bounded by 
\#\label{eq:zzz888}
 |\text{(i)} | 
&= \bigg| \phi(x,a)^\top w_h - \phi(x,a)^\top \Lambda_h^{-1} \Big( \sum_{\tau=1}^{K} \phi(x_h^\tau,a_h^\tau) \phi(x_h^\tau,a_h^\tau)^\top w_h  \Big)\bigg| \notag\\
&= \big| \phi(x,a)^\top w_h - \phi (x,a)^\top \Lambda_h^{-1}(\Lambda_h -  \lambda\cdot I)w_h \big|  = \lambda \cdot \big| \phi(x,a)^\top \Lambda_h^{-1} w_h   \big| \notag \\
&  \leq \lambda \cdot  \|w_h \|_{ \Lambda_h^{-1}}\cdot  \|\phi(x,a) \|_{ \Lambda_h^{-1}} \leq H\sqrt{d \lambda } \cdot  \sqrt{\phi(x,a)^\top  \Lambda_h  ^{-1}\phi(x,a)}.
\#
Here the second  equality follows from the definition of $\Lambda_h$ in Equation \eqref{eq:w18}. Also, the first inequality follows from the Cauchy-Schwarz inequality, while the last inequality follows from the fact that $$\|w_h\|_{\Lambda_h^{-1}} =  \sqrt{ w_h ^\top \Lambda_h^{-1} w_h } \leq  \|  \Lambda_h ^{-1}  \|_{\oper} ^{1/2} \cdot \| w_h \| \leq H \sqrt{d/ \lambda}.$$
Here $\| \cdot \|_{\oper}$ denotes the matrix operator norm and we use the fact that $\|  \Lambda_h ^{-1}  \|_{\oper} \leq \lambda^{-1}$.

It remains  to upper bound term (ii). 
For notational simplicity, for any $h \in [H]$, any $\tau \in [K]$, and any function $V \colon \cS \rightarrow [0, H]$, 
we define the random variable
\#\label{eq:def_eps_h^tau(V)}
\epsilon_h^\tau (V) = r_h ^\tau + V(x_{h+1}^\tau ) - (\BB_h V) (x_h^\tau ,a_h ^\tau ) . 
\#
By the Cauchy-Schwarz inequality, term (ii) defined in Equation \eqref{eq:term1_diff} is upper bounded by
\# \label{eq:define_term3} 
  |\text{(ii)} | &= \bigg|  \phi (x,a)^\top \Lambda_h^{-1} \Big( \sum_{\tau=1}^{K} \phi(x_h^\tau,a_h^\tau) \cdot \epsilon_h^\tau(\hat{V}_{h+1}) \Big)    \bigg| \notag \\
  & \leq \Big\|  \sum_{\tau=1}^{K} \phi(x_h^\tau,a_h^\tau) \cdot \epsilon_h^\tau(\hat{V}_{h+1}) \Big\|_{\Lambda_h^{-1}} \cdot \|\phi(x,a) \|_{ \Lambda_h^{-1}}\notag\\
  &= \underbrace{\Big\|  \sum_{\tau=1}^{K} \phi(x_h^\tau,a_h^\tau) \cdot \epsilon_h^\tau(\hat{V}_{h+1}) \Big\|_{\Lambda_h^{-1}}}_{\displaystyle \text{(iii)} } \cdot \sqrt{\phi(x,a)^\top \Lambda_h^{-1}\phi(x,a)}.
  \# 
  
  %defined in Equation \eqref{eq:define_term3}
In the sequel, we upper bound term (iii) via concentration inequalities. An obstacle is that $\hat V_{h+1}$ 
depends on $\{(x_h^\tau,a_h^\tau)\}_{\tau=1}^K$ via $\{(x_{h'}^\tau,a_{h'}^\tau)\}_{\tau\in[K],h'>h}$, as it is constructed based on the dataset $\cD$. To this end, we resort to uniform concentration inequalities to upper bound
\#
\sup_{V \in \cV_{h+1}(R,B,\lambda)} \Big\|  \sum_{\tau=1}^{K} \phi(x_h^\tau,a_h^\tau) \cdot \epsilon_h^\tau(V) \Big\|\notag
\#
for each $h\in [H]$, where it holds that $\hat{V}_{h+1} \in \cV_{h+1}(R,B,\lambda)$. Here for all $h\in [H]$, we define the function class 
\#\label{eq:form_function_class}
&\cV_{h}(R,B,\lambda) =\big\{ V_h(x;\theta,\beta,\Sigma)\colon \cS\to [0,H]~\text{with}~\|\theta\|\leq R, \beta\in [0,B], \Sigma \succeq \lambda\cdot I   \big\},\notag \\
&\text{where~~}V_h(x;\theta,\beta,\Sigma) = \max_{a\in \cA}  \Bigl\{ \min\bigl \{ \phi(x,a)^\top \theta - \beta\cdot \sqrt{ \phi(x,a)^\top \Sigma^{-1}\phi(x,a) },H  - h + 1\bigr \}^+ \Bigr\}.
\#

% of random variables of the form $\|  \sum_{\tau=1}^{K} \phi(x_h^\tau,a_h^\tau) \cdot \epsilon_h^\tau( V) \|_{\Lambda_h^{-1}}$, where $V$ is any fixed value function in a function class to be decided that contains $\hat V_{h+1}$. 
%Specifically, for any $h \in [H+1 ]$ and  any constants $ R, B, \lambda  > 0$, we define   $\cV_h   (R, B, \lambda)$ as the function class that contains all functions $V\colon \cS \rightarrow [0,H]$ of the form  

%  where $\| \theta \| \leq R$, $\beta \in [0, B]$, and $\Sigma \succeq \lambda \cdot I$. 
For all $\varepsilon >0$ and all $h\in [H]$,  let $ \cN_{h} (\varepsilon; R, B, \lambda)$ be the minimal 
$\varepsilon$-cover  of 
$\cV_h   (R, B, \lambda)$ with respect to the supremum norm. 
In other words, for any function $V \in \cV_h   (R, B, \lambda)$, there exists a function $V^\dagger  \in  \cN_{h} (\varepsilon; R, B, \lambda)$ such that 
\$
\sup_{x\in \cS} \bigl| V(x) - V^\dagger (x) \bigr| \leq \varepsilon.
\$ 
Meanwhile, among all $\varepsilon$-covers of 
$\cV_h   (R, B, \lambda)$ defined by such a property, we choose $ \cN_{h} (\varepsilon; R, B, \lambda)$ as the one with the minimal cardinality.

By Lemma \ref{lem:bound_weight_of_bellman}, we have $\| \hat w_h \| \leq H \sqrt{Kd / \lambda}$. Hence, for all $h\in [H]$, we have
\#
\hat V_{h+1} \in \cV_{h+1} (R_0, B_0 , \lambda),\qquad \text{where~~}R_0=H\sqrt{Kd/\lambda},~B_0=2\beta.\notag
\#
Here $\lambda>0$ is the regularization parameter and $\beta>0$ is the scaling parameter, which are specified in Algorithm \ref{alg:pess_greedy}.
For notational simplicity, we use $\cV_{h+1} $ and $\cN_{h+1} (\varepsilon)$ to denote $\cV_{h+1} (R_0, B_0 , \lambda)$
and $\cN_{h+1} (\varepsilon; R_0, B_0, \lambda)$, respectively.
As it holds that $\hat{V}_{h+1}\in \cV_{h+1}$ and $\cN_{h+1} (\varepsilon)$ is an $\varepsilon$-cover of $\cV_{h+1}$, 
there exists a function $V^\dagger_{h+1}  \in \cN_{h+1} (\varepsilon)$ such that 
\#\label{eq:bound_sup_norm_diff}
\sup_{x\in \cS} \big| \hat V_{h+1} (x) - V^\dagger_{h+1} (x)\big| \leq \varepsilon.
\#
Hence, given $V^\dagger_{h+1} $ and $\hat{V}_{h+1}$, the monotonicity of conditional expectations implies 
\#\label{eq:bound_cond_exp}
& \big|(\PP_hV^\dagger_{h+1} )(x,a) - (\PP_h\hat{V}_{h+1})(x,a)\big|\\
&\qquad=\Big| \EE  \big[V^\dagger_{h+1} (s_{h+1})\biggiven s_h =  x, a_h = a  \big] - \EE  \big[ \hat V_{h+1} (s_{h+1})\biggiven s_h = x,a_h = a  \big]  \Big| \notag \\
&\qquad \leq  \EE \Big[ \big|V^\dagger_{h+1} (s_{h+1})- \hat V_{h+1} (s_{h+1})\big| \Biggiven s_h = x,a_h = a  \Big]   \leq \varepsilon,\qquad \forall (x,a)\in \cS\times \cA,~\forall h\in [H].\notag
\#
Here the conditional expectation is induced by the transition kernel $\cP_h(\cdot \given x,a)$. 
Combining Equation \eqref{eq:bound_cond_exp} and the definition of the Bellman operator $\BB_h$ in Equation \eqref{eq:def_bellman_op}, we have 
\# \label{eq:bound_bellman_update_diff}
\big| (\BB_h V^\dagger_{h+1} )(x,a) - (\BB_h  \hat V_{h+1}) (x,a)  \big| \leq \varepsilon,\qquad \forall (x,a)\in \cS\times \cA,~\forall h\in [H].
\#
By the triangle inequality, Equations \eqref{eq:bound_sup_norm_diff} and \eqref{eq:bound_bellman_update_diff} imply
\#\label{eq:cover_2eps}
\Big| \bigl (r_h (x,a)  + \hat V_{h+1} (x') - (\BB_h  \hat V_{h+1} )(x,a) \bigr ) -   \bigl (r_h(x,a) + V^\dagger_{h+1} (x') - (\BB_h V^\dagger_{h+1}  )(x,a)\big )  \Big| \leq 2\varepsilon   
\#
for all $h\in [H]$ and all $ (x,a,x')\in \cS\times \cA\times \cS$. Setting $(x,a,x')=(x_h^\tau,a_h^\tau,x_{h+1}^\tau)$ in Equation \eqref{eq:cover_2eps}, we have
\#\label{eq:error_2eps}
\bigl| \epsilon_h^\tau(\hat V_{h+1} )-\epsilon_h^\tau(V^\dagger_{h+1}  )\bigr| \leq 2\varepsilon, \qquad \forall \tau\in [K],~\forall h\in [H].
\#
Also, recall the definition of term (iii) in Equation \eqref{eq:define_term3}. 
By the Cauchy-Schwarz inequality, for any two vectors $a,b\in \RR^d$ and any positive definite matrix $\Lambda\in \RR^{d\times d}_+$, it holds that $\|a+b\|_{\Lambda}^2 \leq 2\cdot \|a\|_{\Lambda}^2 + 2\cdot \|b\|_\Lambda^2$. Hence, for all $h\in[H]$, we have 
\#\label{eq:bound_term3_1}
  |\textrm{(iii)}|^2 \leq   2\cdot \Big\| \sum_{\tau=1}^{K} \phi(x_h^\tau,a_h^\tau) \cdot \epsilon_h^\tau(V^\dagger_{h+1}) \Big\|_{\Lambda_h^{-1}}^2 + 2\cdot \Big\| \sum_{\tau=1}^{K} \phi(x_h^\tau,a_h^\tau) \cdot \big(\epsilon_h^\tau(\hat V_{h+1} )-\epsilon_h^\tau(V^\dagger_{h+1})\big) \Big\|_{\Lambda_h^{-1}}^2. 
\#
The second term on the right-hand side of Equation \eqref{eq:bound_term3_1} is upper bounded by 
\$
&2\cdot \Big\| \sum_{\tau=1}^{K} \phi(x_h^\tau,a_h^\tau) \cdot \big(\epsilon_h^\tau(\hat V_{h+1})-\epsilon_h^\tau(V^\dagger_{h+1})\big) \Big\|_{\Lambda_h^{-1}}^2 \\
&\qquad  =  2\cdot \sum_{\tau,\tau'=1}^K \phi(x_h^\tau,a_h^\tau)^\top   \Lambda_h  ^{-1} \phi(x_h^{\tau'},a_h^{\tau'}) \cdot  \big(\epsilon_h^\tau(\hat V_{h+1} )-\epsilon_h^\tau(V^\dagger_{h+1})\big) \cdot \big(\epsilon_h^{\tau'}( \hat V_{h+1})-\epsilon_h^{\tau'}(V^\dagger_{h+1})\big)\\
&\qquad \leq 8\varepsilon^2 \cdot \sum_{\tau,\tau'=1}^K \bigl | \phi(x_h^\tau,a_h^\tau)^\top \Lambda_h^{-1} \phi(x_h^{\tau'},a_h^{\tau'}) \big |  \leq 8\varepsilon^2 \cdot \sum_{\tau,\tau'=1}^K  \| \phi(x_h^\tau,a_h^\tau)  \| \cdot   \| \phi(x_h^{\tau'} ,a_h^{\tau'} )  \| \cdot \|  \Lambda_h ^{-1} \|_{\oper} ,
\$
where the first inequality follows from Equation \eqref{eq:error_2eps}. 
As it holds that $\Lambda _h \succeq \lambda \cdot I$ by the definition of $\Lambda_h$ in Equation \eqref{eq:w18} and $\| \phi(x,a) \| \leq 1$ for all $(x,a) \in \cS \times \cA$ by Definition \ref{assump:linear_mdp}, for all $h\in [H]$, we have
\#\label{eq:bound_term3_2}
2\cdot \Big\| \sum_{\tau=1}^{K} \phi(x_h^\tau,a_h^\tau) \cdot \big(\epsilon_h^\tau(\hat V_{h+1})-\epsilon_h^\tau(V^\dagger_{h+1})\big) \Big\|_{\Lambda_h^{-1}}^2 \leq   8 \varepsilon^2  K^2 / \lambda. 
\#
Combining Equations \eqref{eq:bound_term3_1} and \eqref{eq:bound_term3_2}, for all $h\in[H]$, we have 
\#\label{eq:bound_term3_3}
|\textrm{(iii)}|^2 \leq  2 \cdot  \sup_{V \in \cN_{h+1} (\varepsilon)}    \Big\| \sum_{\tau=1}^{K} \phi(x_h^\tau,a_h^\tau) \cdot \epsilon_h^\tau(V) \Big\|_{\Lambda_h^{-1}}^2   +  8 \varepsilon^2  K^2 / \lambda. 
\#
Note that the right-hand side of Equation \eqref{eq:bound_term3_3} does not involve the estimated value functions $  \hat Q_{h}$ and $ \hat V_{h+1}  $, which are constructed  based on the dataset $\cD$. Hence, it allows  us to upper bound the first term via uniform concentration inequalities. 
We utilize the 
following lemma to characterize the first term for any fixed function $V \in \cN_{h+1} (\varepsilon)$. Recall the definition of $\epsilon_h^\tau(V)$ in Equation \eqref{eq:def_eps_h^tau(V)}. Also recall that $\PP_{\cD}$ is the joint distribution of the data collecting process.

%\begin{flushleft}

\begin{lemma}[Concentration of Self-Normalized Processes] \label{lem:concen_bellman_eval_singleV}
  Let $V :\cS\to [0,H -1]$ be any fixed function.
  %For any $\tau \in [K]$, let 
% $\epsilon_h^\tau (V)=r_h^\tau  + V(x_{h+1}^\tau ) - (\BB_h V)(x_h^\tau ,a_h^\tau)$,
% where $(x_h^\tau, a_h^\tau, r_h^\tau, x_{h+1}^\tau )$ comes from the dataset $\cD$ satisfying Assumption \ref
Under Assumption \ref{assump:data_generate}, for any fixed $h\in [H]$ and any $\delta  \in (0,1)$, 
we have 
\$%\label{eq:concentration_fix_V}
\PP_{\cD} \biggl(  \Big\|   \sum_{\tau=1}^{K} \phi(x_h^\tau,a_h^\tau) \cdot \epsilon_h^\tau(V) \Big\|_{\Lambda_h^{-1}}^2 > H^2\cdot \bigl (  2  \cdot \log(1/ \delta ) + d\cdot \log(1+K/\lambda)\big) 
\biggr) \leq \delta.
\$
%\iffalse 
% be the Bellman evaluation error random variable, where $r_h$ is the reward at step $h$, and $(x_h,a_h)$ is the state-action pair at step $h$. Also denote the realizations $\epsilon_h^\tau(V) = r_h^\tau + V(x_{h+1}^\tau) - [\BB V](x_h^\tau,a_h^\tau)$, $\tau=1,\dots,K$. For any fixed level $\xi\in (0,1)$, define the event 
%  $$
%  \cE_h^0 = \bigg\{   \Big\|   \sum_{\tau=1}^{K} \phi(x_h^\tau,a_h^\tau) \cdot \epsilon_h^\tau(V) \Big\|_{\Lambda_h^{-1}}^2 
%  \leq (H-h+1)^2\Big(\log(\frac{1}{\xi})+d\log(1+K/\lambda)\Big)    \bigg\}.
%  $$
%  Then under Assumptions \ref{assump:linear_mdp} and \ref{assump:data_generate}, we have $\PP_{\cD}(\cE_h^0)\geq 1-\xi$. 
%  \fi 
  \end{lemma}

  \begin{proof}[Proof of Lemma \ref{lem:concen_bellman_eval_singleV}]
%To simplify the notation, 
%for any fixed $h\in [H]$ and value function $V\colon \cS \rightarrow [ 0, H]$  and  any $\tau \in [K]$, we let 
%$\phi_h ^\tau$ and $\epsilon_h  ^ \tau$ denote 
%$\phi(x_h^\tau,a_h^\tau)$ and $ \epsilon_{h}^\tau  (V)  $, respectively. 
For the fixed $h\in [H]$ and all $\tau \in \{ 0, \ldots , K \}$, we define the $\sigma$-algebra 
  $$
  \cF_{h,\tau   } = \sigma\big( \{ (x_h^j,a_h^j)  \} _{j=1}^{(\tau+1) \wedge K }\cup \{(r_{h}^j, x_{h+1}^j) \} _{j=1}^{\tau}      \big),
  $$
  where $\sigma(\cdot)$ denotes the $\sigma$-algebra generated by a set of random variables and $(\tau + 1) \wedge K$ denotes $\min \{ \tau + 1, K \}$. For all $\tau \in [K]$, we have $\phi (x_h^\tau, a_h^\tau) \in \cF_{h, \tau - 1 }$, as $(x_h^\tau, a_h^\tau)$ is $\cF_{h, \tau-1}$-measurable.  
  Also, for the fixed function $V\colon \cS\to [0,H-1]$ and all $\tau \in [K]$, we have 
  \$
  \epsilon_h^\tau (V) =r_h^\tau  + V(x_{h+1}^\tau ) - (\BB_h V)(x_h^\tau ,a_h^\tau)\in \cF_{h, \tau },
  \$
  as $(r_h^\tau, x_{h+1}^\tau)$ is $\cF_{h, \tau}$-measurable. Hence,  
  $\{ \epsilon_{h}^\tau (V) \}_{\tau= 1}^K $ is a stochastic process adapted to the filtration $\{\cF_{h, \tau}\}_{\tau=0}^K $. By Assumption \ref{assump:data_generate}, we have %under the assumption that $\cD$ is compliant with the underlying MDP, Equation  \eqref{eq:assump_data_generate} further implies that for any $\tau \in [K]$,
  \begin{equation*}
  \begin{split}
  \EE_{\cD} \big[\epsilon_h^\tau (V) \biggiven \cF_{h,\tau-1}\big] &= \EE_{\cD} \bigl[r_h^\tau + V(x_{h+1}^\tau) \biggiven \{ (x_h^j,a_h^j) \} _{j=1}^{\tau},  \{(r_{h}^j, x_{h+1}^j) \} _{j=1}^{\tau-1} \bigr]  - (\BB_h V) (x_h^\tau , a_h^\tau )\\
  & = \EE \bigl [ r_h(s_h, a_h) + V(s_{h+1}) \biggiven s_h = x_h^\tau , a_h = a_h^\tau  \bigr ] -  (\BB_h V) (x_h^\tau , a_h^\tau ) = 0,
  \end{split}
  \end{equation*}
  where the second equality follows from Equation \eqref{eq:assump_data_generate} and the last equality follows from the definition of the Bellman operator $\BB_h$ in Equation \eqref{eq:def_bellman_op}.
  Here $\EE_{\cD}$ is taken with respect to $\PP_\cD$, while $\EE$ is taken with respect to the immediate reward and next state in the underlying MDP.% the expectation after the  second equality is with respect to the reward and transition   of the underlying MDP. 
  
Moreover, as it holds that $r_h^\tau\in [0,1]$ and $V\in [0,H-1]$, we have $r_h^\tau + V(x_{h+1}^\tau)\in [0,H]$. Meanwhile, we have $(\BB_h V)(x_h^\tau,a_h^\tau)\in [0,H]$, which implies $|\epsilon_h^\tau (V) | \leq   H$. 
  Hence, for the fixed $h\in [H]$ and all $\tau \in [K]$, the random variable $\epsilon_h^\tau (V) $ defined in Equation \eqref{eq:def_eps_h^tau(V)} is mean-zero and $H$-sub-Gaussian conditioning on $\cF_{h, \tau - 1}$.

We invoke Lemma \ref{lem:concen_self_normalized} with $M_0=\lambda \cdot I$ and $M_k  = \lambda \cdot  I + \sum_{\tau =1}^k \phi(x_h^\tau,a_h^\tau)\ \phi(x_h^\tau,a_h^\tau)^\top$
for all $k \in [K]$. For the fixed function $V\colon \cS\to [0,H-1]$ and fixed $h\in [H]$, we have
%with respect to the randomness of $\cD$, we have 
\$%\label{eq:prob_tail_fixv}
 \PP_{\cD}  \bigg( \Big\|   \sum_{\tau=1}^{K} \phi(x_h^\tau,a_h^\tau) \cdot \epsilon_h^\tau(V)  \Big\|_{\Lambda_h^{-1}}^2  >   2 H^2 \cdot  \log \Big(  \frac{\det(\Lambda_h)^{1/2}}{\delta  \cdot \det( \lambda \cdot I) ^{1/2} }  \Big)   \bigg ) \leq   \delta
  \$
  for all $\delta \in (0,1)$. Here we use the fact that $M_K=\Lambda_h$. Note that  $\|\phi(x,a)\|\leq 1$ for all $(x,a )\in \cS\times \cA$ by Definition \ref{assump:linear_mdp}.
  We have %Then 
   %the operator norm of $\Lambda_h$ can be bounded  by
  \begin{equation*}
  \|\Lambda_h\|_{\oper}  =   \Big\|\lambda \cdot  I + \sum_{\tau=1}^K \phi(x_h^\tau,a_h^\tau)\phi(x_h^\tau,a_h^\tau)^\top \Big\| _{\oper} \leq  \lambda  +  \sum_{\tau = 1} ^K  \| \phi(x_h^\tau,a_h^\tau)\phi(x_h^\tau,a_h^\tau)^\top  \|_{\oper} \leq \lambda  + K,
  \end{equation*}
  where $\|\cdot\|_{\oper}$ denotes the matrix operator norm.
  Hence, it holds that
  $\det(\Lambda_h)\leq (\lambda+K)^d$ and $\det(\lambda \cdot I) = \lambda ^d $, which implies
  \#   
& \PP_{\cD}  \bigg( \Big\|   \sum_{\tau=1}^{K} \phi(x_h^\tau,a_h^\tau) \cdot \epsilon_h^\tau(V)  \Big\|_{\Lambda_h^{-1}}^2  >  H^2\cdot \bigl (  2  \cdot \log(1/ \delta ) + d\cdot \log(1+K/\lambda)\big) 
\biggr)  \notag \\
&\qquad \leq \PP_{\cD}  \bigg( \Big\|   \sum_{\tau=1}^{K} \phi(x_h^\tau,a_h^\tau) \cdot \epsilon_h^\tau(V)  \Big\|_{\Lambda_h^{-1}}^2  >   2 H^2 \cdot  \log \Big(  \frac{\det(\Lambda_h)^{1/2}}{\delta  \cdot \det( \lambda \cdot I) ^{1/2} }  \Big)   \bigg ) \leq   \delta. \notag
  \#
%which gives Equation \eqref{eq:concentration_fix_V}. 
Therefore, we conclude the proof of Lemma \ref{lem:concen_bellman_eval_singleV}. 
  \end{proof}

Applying Lemma \ref{lem:concen_bellman_eval_singleV} and the union bound, for any fixed $h \in [H]$, 
we have 
\$
 & \PP_{\cD} \bigg( \sup_{V \in \cN_{h+1} (\varepsilon)}     \Big\| \sum_{\tau=1}^{K} \phi(x_h^\tau,a_h^\tau) \cdot \epsilon_h^\tau(V)  \Big\|_{\Lambda_h^{-1}}^2 >  H^ 2  \cdot \bigl (  2 \cdot \log(1/ \delta ) + d \cdot \log(1+K/\lambda)\big)   \biggr )  \leq   \delta \cdot | \cN_{h+1}(\varepsilon ) | .
\$
For all $\xi \in(0, 1)$ and all $\varepsilon > 0$, we set  $\delta =\xi/(H\cdot  | \cN_{h+1}(\varepsilon)| )$.  
Hence, for any fixed $h \in [H]$,
  it holds that
\# \label{eq:bound_term3_4}
&  \sup_{V \in \cN_{h+1} (\varepsilon)}  \Big\| \sum_{\tau=1}^{K} \phi(x_h^\tau,a_h^\tau) \cdot \epsilon_h^\tau(V)  \Big\|_{\Lambda_h^{-1}}^2 \notag \\
 &\qquad  \leq H^2 \cdot \Bigl ( 2\cdot \log \bigl (H \cdot  | \cN_{h+1} (\varepsilon ) | / \xi  \bigr )   + d \cdot \log (1 + K / \lambda )  \Bigr ) 
\#
with probability at least $1 - \xi / H$, which is taken with respect to $\PP_{\cD}$.
Combining Equations \eqref{eq:bound_term3_3} and \eqref{eq:bound_term3_4}, we have 
\#\label{eq:bound_term3_5}
& \PP_{\cD} \bigg(  \bigcap_{h\in [H]} \biggl \{   \Big\|  \sum_{\tau=1}^{K} \phi(x_h^\tau,a_h^\tau) \cdot \epsilon_h^\tau(\hat{V}_{h+1})  \Big\|_{\Lambda_h^{-1}}^2  \\
& \qquad\qquad  \qquad \leq  2 H^2 \cdot \Bigl ( 2 \cdot \log \bigl (H \cdot  | \cN_{h+1} (\varepsilon ) | / \xi  \bigr )   + d \cdot \log (1 + K / \lambda )  \Bigr )  + 8 \varepsilon^2 K^2 / \lambda \biggr \} \biggr ) \geq 1 - \xi, \notag
 \#
 which follows from the union bound. 
 
It remains to choose a proper $\varepsilon > 0$ and upper bound the $\varepsilon$-covering number $| \cN_{h+1} (\varepsilon)| $. 
In the sequel, we set 
$\varepsilon = dH / K$ and $\lambda = 1$. 
By Equation \eqref{eq:bound_term3_5}, for all $h\in[H]$, it holds that
%with probability at least $1 - \xi$ with respect to $\PP_{\cD}$,
%we have 
\#\label{eq:bound_term3_6}
\Big\|  \sum_{\tau=1}^{K} \phi(x_h^\tau,a_h^\tau) \cdot \epsilon_h^\tau(\hat{V}_{h+1}) \Big\|_{\Lambda_h^{-1}}^2  \leq 2H^2 \cdot  
 \Bigl ( 2 \cdot \log \bigl (H \cdot  | \cN_{h+1} (\varepsilon ) | / \xi  \bigr )   + d \cdot \log (1 + K ) + 4 d^2  \Bigr )   
\#
 with probability at least $1-\xi$, which is taken with respect to $\PP_\cD$.
To upper bound $|\cN_{h+1}(\varepsilon)|$, 
 we utilize the 
  following lemma, which is obtained from \cite{jin2020provably}. Recall the definition of the function class $\cV_{h}(R,B,\lambda)$ in Equation \eqref{eq:form_function_class}. Also, recall that $\cN_{h}(\varepsilon;R,B,\lambda)$ is the minimal $\varepsilon$-cover of $\cV_h(R,B,\lambda)$ with respect to the supremum norm.

\begin{lemma}[$\varepsilon$-Covering Number \citep{jin2020provably}]
	%Recall that we define $\cV_h(R, B, \lambda)$ as the function  class 
	%that contains all value functions of the form in Equation \eqref{eq:form_function_class} with $(\theta, \beta, \Sigma)$ satisfying $\| \theta \| \leq R$, $\beta \in [0, B]$, and $\Sigma \succeq \lambda \cdot I$.  
	%Suppose the feature map $\phi $ is bounded such that $\| \phi(x,a )\| \leq 1$ for all $(x,a) \in \cS \times \cA$. 
	For all $h\in [H]$ and all $\varepsilon > 0$, %let $\cN_{h}(\varepsilon; R, B, \lambda)$ denote the minimal $\epsilon$-cover of $\cV_h(R, B, \lambda)$ with respect to the supremum norm. 
	we have 	\label{lem:covering_num}
	$$
	\log | \cN_h (  \varepsilon; R, B, \lambda)  | \leq d \cdot \log (1+ 4 R /  \varepsilon  ) + d^2  \cdot  \log\bigl(1+ 8 d^{1/2} B^2 / ( \varepsilon^2\lambda) \bigr). 	$$
 %Here the supremum norm between two value functions $V, V'$ is defined as $\sup_{x\in\cS} | V(x) - V'(x) | $.
\end{lemma}

\begin{proof}[Proof of Lemma \ref{lem:covering_num}]
	See Lemma D.6 in \cite{jin2020provably} for a detailed proof. 
\end{proof}

Recall that 
\$
\hat V_{h+1} \in \cV_{h+1} (R_0, B_0, \lambda),\qquad \text{where}~~  R_0 = H \sqrt{ Kd/\lambda},~  B_0 = 2\beta,~ \lambda = 1 ,~ \beta = c \cdot d H \sqrt{\zeta}.
\$ 
Here $c>0$ is an absolute constant, $\xi\in (0,1)$ is the confidence parameter, and $\zeta = \log (2 d H K / \xi) $ is specified in Algorithm \ref{alg:pess_greedy}. 
Recall that $\cN_{h+1}(\varepsilon) = \cN_{h+1}(\varepsilon; R_0,B_0,\lambda) $ is the minimal $\varepsilon$-cover of $\cV_{h+1}=\cV_{h+1}(R_0,B_0,\lambda)$ with respect to the supremum norm. Applying Lemma \ref{lem:covering_num} with $\varepsilon = d H / K$,
we have 
\#\label{eq:apply_cov_num}
\log  | \cN_{h+1}(\varepsilon) | &  \leq  d \cdot \log  ( 1 + 4 d^{-1/2}K^{3/2}  ) + d^2 \cdot \log  ( 1 + 32 c^2\cdot d^{1/2}K^2\zeta  )\notag \\
&  \leq  d \cdot \log  ( 1 + 4 d^{1/2}K^2  ) + d^2 \cdot \log  ( 1 + 32 c^2 \cdot d^{1/2}K^2\zeta  ).
\#
As it holds that $\zeta >1$, we set $c\geq 1$ to ensure that the second term on 
the right-hand side of Equation \eqref{eq:apply_cov_num} is the dominating term, where $32 c^2 \cdot d^{1/2}K^2\zeta \geq 1$. 
Hence, we have
\#\label{eq:bound_term4_7}
\log  | \cN_{h+1}(\varepsilon) |  & \leq   2 d^2 \cdot \log  ( 1 + 32c^2 \cdot d^{1/2 }  K^2  \zeta    )  \leq  2 d^2 \cdot \log  ( 64c^2 \cdot d^{1/2 }  K^2  \zeta    ).
\#
By Equations \eqref{eq:bound_term3_6} and \eqref{eq:bound_term4_7}, for all $h\in [H]$, it holds that 
\#\label{eq:www888}
& \Big\|  \sum_{\tau=1}^{K} \phi(x_h^\tau,a_h^\tau) \cdot \epsilon_h^\tau(\hat{V}_{h+1}) \Big\|_{\Lambda_h^{-1}} ^2\notag
\\&\qquad \leq 2 H^2 \cdot  
\bigl ( 2 \cdot \log (H/ \xi  ) + 4d^2 \cdot \log  ( 64c^2\cdot d^{1/2 }  K^2  \zeta    ) + d\cdot \log(1+K) + 4 d^2  \bigr ) 
\#
 with probability at least $1-\xi$, which is taken with respect to $\PP_\cD$. Note that $\log(1+K)\leq\log(2K)\leq \zeta$ and $\log\zeta\leq \zeta$. Hence, we have 
\$
&2 \cdot \log (H/ \xi  ) + 4d^2 \cdot \log  ( d^{1/2 }  K^2  \zeta    ) +d\cdot \log(1+K) +4d^2\\
&\quad \leq 2d^2\cdot \log(dHK^4/\xi)+d\zeta +8d^2\zeta\leq 18d^2 \zeta.
\$
As it holds that $\zeta> 1$ and $\log \zeta \leq \zeta$, Equation \eqref{eq:www888} implies
\#\label{eq:bound_term3_8}
\Big\|  \sum_{\tau=1}^{K} \phi(x_h^\tau,a_h^\tau) \cdot \epsilon_h^\tau(\hat{V}_{h+1}) \Big\|_{\Lambda_h^{-1}} ^2
% \leq 2 H^2 \cdot  \bigl ( 8d^2\zeta + 4d^2\zeta + 4d^2\log(64c^2) + 4d^2  \bigr ) 
 \leq d^2H^2\zeta \cdot \big( 36 + 8\cdot \log(64c^2) \big).
\#
%Here the second inequality follows from the fact that $\log(H/\xi)\leq \zeta$, $\log(\zeta)\leq \zeta$ and $\log(d^{1/2}K^2)\leq 3\zeta$, while the third inequality follows from the fact that $\zeta>1$.
% second inequality follows from the facts that $\log \zeta \leq  \zeta$ and   $$2  \log (H/ \xi  ) + d \cdot \log (1 + K ) \leq 2d \cdot  \log \bigl (H \cdot (1+ K) / \xi \bigr ) \leq 2 d \cdot  \log   (2 dH K/ \xi   ) = 2d \cdot \zeta, $$ 
%which together imply
%$$
% 2 \log (H/ \xi  ) +   4d^2 \cdot \log \zeta+ d \cdot \log (1 + K ) \leq 6 d^2 \cdot \zeta .
% $$
We set $c\geq 1$ to be sufficiently large, which ensures that $36+8\cdot \log(64c^2)\leq c^2/4$ on the right-hand side of Equation \eqref{eq:bound_term3_8}. By Equations \eqref{eq:define_term3} and \eqref{eq:bound_term3_8}, for all $h \in [H]$, it holds that  
 \#\label{eq:rrr888}
 | \text{(ii)}| \leq c/2 \cdot d H \sqrt{ \zeta } \cdot \sqrt{ \phi(x,a) ^\top  \Lambda_h  ^{-1} \phi(x,a) }  =  \beta /2 \cdot \sqrt{ \phi(x,a) ^\top  \Lambda_h  ^{-1} \phi(x,a) } 
 \#
with probability at least $1 - \xi$, which is taken with respect to $\PP_{\cD}$.
%
% By Equation \eqref{eq:bound_term3_8}, we can choose the absolute constant $c$ to be sufficiently large such that $36+8\log(64c^2)\leq c^2/4$. For example, we may set $c=22$. In that case, the last term in Equation  \eqref{eq:bound_term3_8} is upper bounded by $ (c^2 / 4) \cdot d^2 H^2 \zeta$.  
% Thus, combining Equations \eqref{eq:define_term3} and \eqref{eq:bound_term3_8}, it holds that
% \$
% | \text{(ii)}| \leq c/2 \cdot d H \sqrt{ \zeta } \cdot \sqrt{ \phi(x,a) ^\top  \Lambda_h  ^{-1} \phi(x,a) }  =  \beta /2 \cdot \sqrt{ \phi(x,a) ^\top  \Lambda_h  ^{-1} \phi(x,a) } 
% \$
% for all $h \in [H]$ with probability at least $1 - \xi$ with respect to $\PP_{\cD}$. 
 
 By Equations \eqref{eq:w05}, \eqref{eq:term1_diff}, \eqref{eq:zzz888}, and \eqref{eq:rrr888}, for all $h \in [H]$ and all $(x,a) \in \cS\times \cA$, it holds that 
  \$
\bigl |  (\BB_h \hat V_{h+1} ) (x,a) - (\hat\BB_h \hat V_{h+1} ) (x,a) \bigr |  \leq (  H \sqrt{d} +  \beta /2 ) \cdot \sqrt{ \phi(x,a) ^\top \Lambda_h  ^{-1} \phi(x,a) }  \leq \Gamma_h (x,a)
 \$
 with probability at least $1 - \xi$, which is taken with respect to $\PP_{\cD}$. In other words, $\{ \Gamma_h\}_{h=1}^H$ defined in Equation \eqref{eq:w05}  are  $\xi$-uncertainty quantifiers.  
Therefore, we conclude the proof of Lemma \ref{lemma:linear_MDP_uncertainty}.
\end{proof}
\subsection{Proof of Corollary \ref{cor:well_explore_optimal}} \label{sec:appendix:optimal-explore}

\begin{proof}[Proof of Corollary \ref{cor:well_explore_optimal}]
By the Cauchy-Schwarz inequality, we have
\#\label{eq:bound_eigen}
&\EE_{\pi^*}\Bigl[ \bigl(\phi(s_h,a_h)^\top \Lambda_h^{-1}\phi(s_h,a_h)\bigr)^{1/2} \Biggiven s_1=x\Bigr]\notag \\
&\qquad = \EE_{\pi^*}\Bigl[ \sqrt{\tr\big(\phi(s_h,a_h)^\top \Lambda_h^{-1}\phi(s_h,a_h)\big)} \Biggiven s_1=x\Bigr]\notag \\
&\qquad = \EE_{\pi^*}\Bigl[ \sqrt{\tr\big(\phi(s_h,a_h)\phi(s_h,a_h)^\top \Lambda_h^{-1}\big)} \Biggiven s_1=x\Bigr]\notag \\
&\qquad \leq  \sqrt{\tr\Big(\EE_{\pi^*}\big[\phi(s_h,a_h)\phi(s_h,a_h)^\top \biggiven s_1=x\big]\Lambda_h^{-1}\Big)}
\#
for all $x\in \cS$ and all $h\in [H]$. 
We define the event
%\#\label{eq:opt_explore}
%\cE_1 = \bigcap_{h=1}^H \Big\{ \Lambda_h \geq I + c_0 \cdot K \cdot \EE_{\pi^*}\bigl[\phi(s_h,a_h)\phi(s_h,a_h)^\top\biggiven s_1=x\bigr]\text{ for all }x\in \cS   \Big\},
%\#
\#\label{eq:subopt_bound}
\cE^{\ddagger} = \bigg\{\text{SubOpt}\big(\linpess(\cD);x \big) \leq 2 \beta \sum_{h=1}^H\EE_{\pi^*}\Bigl[ \bigl(\phi(s_h,a_h)^\top \Lambda_h^{-1}\phi(s_h,a_h)\bigr)^{1/2} \Biggiven s_1=x\Bigr]\text{ for all }x\in \cS\bigg\}.
\#
For notational simplicity, we define 
\$
\Sigma_h(x) = \EE_{\pi^*}\big[\phi(s_h,a_h)\phi(s_h,a_h)^\top \biggiven s_1=x\big]
\$ for all $x\in \cS$ and all $h\in [H]$.
On the event $\cE^\dagger \cap \cE^\ddagger$, where $\cE^\dagger$ and $\cE^\ddagger$ are defined in Equations \eqref{eq:event_opt_explore} and \eqref{eq:subopt_bound}, respectively, we have
\$
\text{SubOpt}\big(\linpess(\cD);x \big)&\leq 2 \beta \sum_{h=1}^H\EE_{\pi^*}\Bigl[ \bigl(\phi(s_h,a_h)^\top \Lambda_h^{-1}\phi(s_h,a_h)\bigr)^{1/2} \Biggiven s_1=x\Bigr]\\
&\leq 2 \beta \sum_{h=1}^H \sqrt{\tr\Big(\Sigma_h(x)\cdot  \big(I + c^\dagger \cdot K \cdot \Sigma_h(x) \big)^{-1}\Big)}\\
& =2 \beta \sum_{h=1}^H \sqrt{\sum_{j=1}^d \frac{\lambda_{h,j}(x)}{1+c^\dagger\cdot K \cdot \lambda_{h,j}(x)}}.
\$
Here $\{\lambda_{h,j}(x)\}_{j=1}^d$ are the eigenvalues of $\Sigma_h(x)$ for all $x\in \cS$ and all $h\in [H]$, the first inequality follows from the definition of $\cE^\ddagger$ in Equation \eqref{eq:subopt_bound}, and the second inequality follows from Equation \eqref{eq:bound_eigen} and the definition of $\cE^\dagger$ in Equation \eqref{eq:event_opt_explore}.
Meanwhile, by Definition \ref{assump:linear_mdp}, we have $\|\phi(s,a)\|\leq 1$ for all $(x,a)\in \cS \times \cA$. By Jensen's inequality, we have
\$
\|\Sigma_h(x)\|_{\oper} \leq \EE_{\pi^*}\big[ \|\phi(s_h,a_h)\phi(s_h,a_h)^\top \|_{\oper}\biggiven s_1=x  \big] \leq 1
\$
for all $x\in \cS$ and all $h\in [H]$. As $\Sigma_h(x)$ is positive semidefinite, we have
$ \lambda_{h,j}(x) \in [0,1]$ for all $x\in \cS$, all $h\in [H]$, and all $j\in [d]$. Hence, on $\cE^\dagger \cap \cE^\ddagger$, we have
\$
\text{SubOpt}\big(\linpess(\cD);x \big)&\leq 2 \beta \sum_{h=1}^H \sqrt{\sum_{j=1}^d \frac{\lambda_{h,j}(x)}{1+c^\dagger\cdot K \cdot \lambda_{h,j}(x)}} \\
&\leq 2 \beta \sum_{h=1}^H \sqrt{\sum_{j=1}^d \frac{1}{1+c^\dagger\cdot K}} \leq c'  \cdot d^{3/2} H^2 K^{-1/2} \sqrt{\zeta}
\$
for all $x\in \cS$, where the second inequality follows from the fact that $\lambda_{h,j}(x) \in [0,1]$ for all $x\in \cS$, all $h\in [H]$, and all $j\in [d]$, while the third inequality follows from the choice of the scaling parameter $\beta > 0$ in Corollary \ref{cor:well_explore_optimal}. Here we define 
the absolute constant $c'=2c/\sqrt{c^\dagger}>0$, where $c^\dagger>0$ is the absolute constant used in Equation \eqref{eq:event_opt_explore} and $c>0$ is the absolute constant used in Theorem \ref{thm:regret_upper_linear}.
By the condition in Corollary \ref{cor:well_explore_optimal}, we have $\PP_{\cD}(\cE^\dagger)\geq 1-\xi/2$. 
Also, by Theorem \ref{thm:regret_upper_linear}, we have $\PP_{\cD}(\cE^\ddagger)\geq 1-\xi/2$. 
Hence, by the union bound, we have $\PP_{\cD}(\cE^\dagger \cap \cE^\ddagger)\geq 1-\xi$, which yields Equation \eqref{eq:event_opt_explore_d}.

In particular, if $\rank(\Sigma_h(x))\leq r$ for all $x\in \cS$ and all $h\in [H]$, on $\cE^\dagger \cap \cE^\ddagger$, which satisfies $\PP_{\cD}(\cE^\dagger \cap \cE^\ddagger)\geq 1-\xi$, we have
\$
\text{SubOpt}\big(\linpess(\cD);x \big)&\leq 2 \beta \sum_{h=1}^H \sqrt{\sum_{j=1}^d \frac{\lambda_{h,j}(x)}{1+c^\dagger\cdot K \cdot \lambda_{h,j}(x)}} \\
&\leq 2 \beta \sum_{h=1}^H \sqrt{\sum_{j=1}^r \frac{1}{1+c^\dagger\cdot K}} \leq c''  \cdot d  H^2 K^{-1/2}\sqrt{\zeta},
\$
where $c'' = 2c\sqrt{r/c^\dagger}>0$ is an absolute constant. Here the second inequality follows from the fact that $\lambda_{h,j}(x) \in [0,1]$ and $\rank(\Sigma_h(x))\leq r$ for all $x\in \cS$, all $h\in [H]$, and all $j\in [d]$, while the third inequality follows from the choice of $\beta$ in Corollary \ref{cor:well_explore_optimal}. Hence, we obtain Equation \eqref{eq:event_opt_explore_r}. Therefore, we conclude the proof of Corollary \ref{cor:well_explore_optimal}.
\end{proof}
%\end{flushleft}

\subsection{Proof of Corollary \ref{cor:well_explore}}\label{sec:appendix:well-explored}
\begin{proof}[Proof of Corollary \ref{cor:well_explore}]
%\begin{flushleft}
%	We show that the $\xi$-uncertainty quantifier $\{ \Gamma_h\}_{h = 1 }^H $ for linear MDP is uniformly bounded when the dataset $\cD$ is well-explored. 

For all $h \in [H]$ and all $\tau \in [K]$, 
we define the random matrices
\#\label{eq:def_A_tau}
&Z_h=\sum_{\tau=1}^K  A_h^\tau,\qquad
A_h^\tau = \phi(x_h^\tau,a_h^\tau)\phi(x_h^\tau,a_h^\tau)^\top - \Sigma_h,\notag \\
&\text{where}~~\Sigma_h = \EE_{\bar{\pi}}\big[\phi(s_h,a_h)\phi(s_h,a_h)^\top\big].
\#
For all $h\in[H]$ and all $\tau \in [K]$, %the definition of $A_h^\tau$ and $\Sigma_h$ in 
Equation \eqref{eq:def_A_tau} implies $\EE_{\bar\pi}[A_h^\tau]=0$. Here $\EE_{\bar\pi}$ is taken with respect to the trajectory induced by the fixed behavior policy $\bar\pi$ in the underlying MDP.
% Here we omit the dependency of $h$ in $Z$ and $A_{\tau}$ for notational simplicity. 
As the $K$  trajectories in the dataset $\cD$ are $\iidtext$,
for all $h\in [H]$,  $\{ (x_h^\tau,a_h^\tau, r_h^\tau )\} _{\tau=1}^K$ are also $\iidtext$. Hence, for all $h\in[H]$, $\{A_h^\tau\}_{\tau=1}^K$ are $\iidtext$ and centered.

By Definition \ref{assump:linear_mdp}, we  have $\| \phi(x,a)\| \leq 1$ for all $(x,a) \in \cS\times \cA$. 
By 
 Jensen's inequality, we have   $$\|\Sigma_h\|_{\oper} \leq \E_{\bar\pi}\big[ \|\phi(s_h,a_h)\phi(s_h,a_h)^\top \|_{\oper} \big]\leq 1.$$ 
For any vector $v\in \RR^d$ with $\| v \| = 1$, the triangle inequality implies 
$$
\|A_h^\tau  v\| \leq \|\phi(x_h^\tau,a_h^\tau)\phi(x_h^\tau,a_h^\tau)^\top v \| + \|\Sigma_h v\| \leq \|v\| + \|\Sigma_h\|_{\oper} \cdot\|v\|
\leq 2.
$$
Hence, for all $h\in[H]$ and all $\tau \in [K]$, we have 
\$
\| A_h^{\tau} \|_{\oper} \leq 2,\qquad\| A_h^\tau  (A_h^\tau) ^\top\|_{\oper}\leq \|A_h^\tau \|_{\oper} \cdot \|(A_h^\tau) ^\top\|_{\oper} \leq 4.
\$
As $\{ A_h^{\tau }\}_{\tau =1}^K $ are $\iidtext$ and centered, for all $h\in[H]$, we have 
\$ %\label{eq:bound_ZZt_op}
 \|\EE_{\bar\pi}[Z_hZ_h^\top] \|_{\oper} &= \Big \|\sum_{\tau=1}^K \EE_{\bar\pi}[A_h^\tau (A_h^\tau)^\top]\Big\|_{\oper }\notag \\
&  =  K\cdot  \| \EE_{\bar\pi}[A_h^1 (A_h^1)^\top] \|_{\oper } \leq K \cdot  \EE_{\bar\pi}\big[ \|A_h^1 (A_h^1)^\top \|_{\oper} \big] \leq 4K,
\$
where the 
 first  inequality  follows from Jensen's inequality.
  Similarly, for all $h\in[H]$ and all $\tau \in [K]$, as it holds that  
   \$
    \| (A_h^\tau)  ^\top A_h^\tau   \|_{\oper}\leq  \|(A_h^\tau)^\top  \|_{\oper} \cdot \|A_h^\tau  \|_{\oper} \leq 4,
    \$ 
  we have
  \$
   \|\E_{\bar\pi}[Z_h^\top Z_h] \|_{\oper}
  \leq 4K.
  \$
   Applying Lemma \ref{lem:matrix_concentration} to  $Z_h$ defined in Equation \eqref{eq:def_A_tau}, for any fixed $h\in[H]$ and any $t \geq 0$, we have
\#\label{eq:apply_mat_bern1}
\PP_{\cD} \big(\|Z_h\|_{\oper}  > t\big)&= \PP_{\cD}\bigg(\Big\|\sum_{\tau=1}^K   A_h^{\tau }\Big\|_{\oper} > t \bigg)\leq 2d\cdot\exp\Bigl( -\frac{t^2/2}{4K + 2t/3} \Bigr).
\#
For all $\xi\in (0,1)$, we set  $t=\sqrt{10K \cdot \log(4d H /\xi)}$.
By Equation \eqref{eq:apply_mat_bern1}, when $K$ is sufficiently large so that $K \geq 5\cdot \log (4d H /\xi)$, 
we have 
 $2t/3\leq K$. Hence,  for the fixed $h\in [H]$, we have
% By Equation \eqref{eq:apply_mat_bern1}, for any fixed $h \in [H]$,  $ \| Z_h \|_{\oper} \leq t$ with probability at least 
\#\label{eq:apply_mat_bern1+}
 \PP_{\cD} \big(\|Z_h\|_{\oper}  \leq t\big) &\geq 1 - 2d \cdot \exp\bigl (-t^2/(8K+4t/3) \bigr )\notag\\
 & \geq 1 - 2d \cdot \exp \bigl(-t^2 / (10 K) \bigr ) = 1- \xi / (2 H).
 \#
 %with respect to $\PP_{\cD}$. 
 By Equation  \eqref{eq:apply_mat_bern1+} and the union bound, for all $h\in[H]$, it holds that 
 \# \label{eq:upper_cor1}
 \|Z_h/K \|_{\oper} =  \Big\|\frac{1}{K}\sum_{\tau=1}^K  \phi(x_h^\tau,a_h^\tau)\phi(x_h^\tau,a_h^\tau)^\top - \Sigma_h\Big\|_{\oper} \leq \sqrt{10/K\cdot\log(4dH/\xi)} 
 \#
with probability at least $1-\xi/2$, which is taken with respect to $\PP_{\cD}$. 

By the definition of $Z_h$ in Equation \eqref{eq:def_A_tau}, we have
 \# \label{eq:relation_zh_sigmah}
 Z_h = \sum_{\tau = 1}^K \phi(x_h^\tau, a_h^\tau) \phi(x_h^\tau, a_h^\tau) ^\top - K \cdot \Sigma_h =(\Lambda_h - \lambda \cdot I ) - K \cdot \Sigma_h. 
 \#
Recall that
  there exists an absolute constant $\underline{c}>0$ such that  $\lambda_{\min}(\Sigma_h)\geq \underline{c}/d$,
  which implies 
  $\| \Sigma_h^{-1} \|_{\oper} \leq d/  \underline{c}$. 
  %Note that we have 
 %$ \sqrt{10\log(4dH/\xi)/ K} \leq \underline{c} /2$. 
  By Equations \eqref{eq:upper_cor1} and \eqref{eq:relation_zh_sigmah}, 
  when $K$ is sufficiently large so that $K  \geq 40d/\underline{c} \cdot \log (4d H / \xi)$, 
  for all $h\in[H]$, it holds that
  \$
  \lambda_{\min }  (  \Lambda_h / K  ) &=  \lambda_{\min }  (\Sigma_h + \lambda/K\cdot I + Z_h/K)\\
  &\geq \lambda _{\min} ( \Sigma_h ) - \|Z_h/K\|_{\oper} \geq  \underline{c}/d - \sqrt{ 10/K\cdot \log (4d H / \xi) } \geq \underline{c} /2.
  \$
 Hence, for all $h\in[H]$, it holds that
 $$
 \|  \Lambda _h ^{-1}  \|_{\oper} \leq \bigl  ( K\cdot  \lambda _{\min } ( \Lambda_h / K) \bigr )^{-1}
  \leq 2  d / ( K \cdot \underline{c})  $$ 
  with probability at least $1-\xi/2$ with respect to $\PP_\cD$, which implies
 \begin{equation}
 \sqrt{\phi(x,a)^\top \Lambda_h^{-1} \phi(x,a)} \leq \| \phi (x,a) \| \cdot \|  \Lambda_h  ^{-1}  \| _{\oper}^{1/2}  \leq  c'' \sqrt{d/K} ,\qquad \forall (x,a)\in \cS\times \cA,~\forall h\in [H].
 \label{eq:def_c2_bound}
 \end{equation}
 Here we define the absolute constant $c'' =\sqrt{ 2 / \underline {c}} $ 
 and use the fact that $\| \phi (x,a) 
\|\leq 1$ for all $(x,a)\in \cS\times \cA$ in Definition \ref{assump:linear_mdp}.

We define the event 
\$%\label{eq:define_event_ce2} 
\cE_1^* = \Big\{ \sqrt{\phi(x,a)^\top   \Lambda_h ^{-1} \phi(x,a)} \leq  c'' \sqrt{d/K} ~\textrm{for all}~(x,a)\in \cS\times \cA \text{ and all } h\in[H]  \Big\}. 
\$
By Equation \eqref{eq:def_c2_bound}, 
we have $\PP_{\cD} (\cE_1^*) \geq 1 - \xi /2$ for $K  \geq 40d/\underline{c} \cdot \log (4d H / \xi)$.
Also, we define the event 
\$%\label{eq:define_event_ce1}
\cE_2^* = \bigg\{ \text{SubOpt}(\hat\pi;x)\leq  2\beta\cdot \sum_{h=1}^H \EE_{\pi^*}\Big[\sqrt{\phi(s_h,a_h)^\top \Lambda_h^{-1}\phi(s_h,a_h)} \Biggiven s_1=x \Big] ~\textrm{for all}~ x\in \cS   \bigg\}.
\$
Here we set $\beta = c\cdot d H \sqrt{ \log (4dHK / \xi) }$, where $c>0$ is the same absolute constant as in Theorem \ref{thm:regret_upper_linear}. By Theorem \ref{thm:regret_upper_linear},
we have $\PP_{\cD}(\cE_2^*)\geq 1-\xi/2$.
Hence, when $K$ is sufficiently large so that $K  \geq 40/\underline{c} \cdot \log (4d H / \xi)$, 
on the event $\cE ^* = \cE_1^*\cap \cE_2^*$, 
we have 
\$
  \text{SubOpt}(\hat\pi;x)\leq  2\beta \cdot H \cdot c''  \sqrt{d/K}  = c' \cdot d^{3/2}H^2 \sqrt{\log(4dHK/\xi)/ K },\qquad \forall x\in \cS.   
\$
By the union bound, we have $\PP_{\cD}(\cE^*)\geq 1-\xi$ with $c' = 2 c \cdot c''$, where $c'' = \sqrt{2 / \underline{c}}$ and $c>0$ is the same absolute constant as in Theorem \ref{thm:regret_upper_linear}.  
Therefore, we conclude the  proof of Corollary \ref{cor:well_explore}.
%\end{flushleft} 
\end{proof}

\iffalse 
%In addition, recall that $\| \Sigma_h \|_{\oper} \leq 1$. 
%Equation \eqref{eq:upper_cor1} 
%implies that 
% $\| \Lambda _h / K \|_{\oper}  \leq (1 + \lambda / K) + \underline{c} / 2 \leq \underline{c} / 2  + 2$. 
% Meanwhile, for any two symmetric matrices $A, B \in \RR^{d\times d}$ such that both $A$ and $A+ B$ are invertible, we have 
%$$
%\| A^{-1} - (A+ B)^{-1} \|_{\oper} = \|  A^{-1} B (A+ B)^{-1} \|_{\oper} \leq \| A ^{-1} \|_{\oper} \cdot \| (A+ B)^{-1} \|_{\oper} \cdot \| B \|_{\oper}. 
%$$ 
% Applying this inequality with $A = \Lambda_h / K$ and $A+ B = \Sigma_h + \lambda / K \cdot I$, we have  
%\#\label{eq:upper_cor2}
%  K\cdot \|\Lambda_h^{-1}\|_{\oper} 
%
%  \leq \|(\Sigma_h +\lambda  / K \cdot I )^{-1}\|_{\oper}+ \| (\Sigma_h +\lambda /K\cdot I )^{-1} \|_{\oper} \cdot \|  \sqrt{\frac{10\log(4dH/\xi)}{K}} \leq c_1
%\#
%for some constant $c_1>0$. This indicates 
%\begin{equation}
%\sqrt{\phi(x,a)^\top \Lambda_h^{-1} \phi(x,a)} \leq \frac{c_2}{\sqrt{K}},\quad \forall (x,a)\in \cS\times \cA
%\label{eq:def_c2_bound}
%\end{equation}
%for constant $c_2=\sqrt{c_1}$ since $\|\phi(x,a)\|\leq 1$ for all $(x,a)\in \cS\times \cA$. Now that i.i.d. trajectories satisfies Assumption \ref{assump:data_generate}, Theorem \ref{thm:regret_upper_linear} also applies in this case. Therefore there exists some constant $c'>0$ so that taking $\lambda=1$ and $\beta=c'\cdot dH\sqrt{\log(4dHK/\xi)}$ in Algorithm \ref{alg:pess_greedy}, the event 
\fi

\section{Proofs of Minimax Optimality}

%In this section, 
%we prove the results of the minimax optimality. 
%Section \ref{sec:appendix:proof_minimax_lower_split} is devoted to the proof of Lemma \ref{lem:minimax_lower_split}, which is used to prove 
%Theorem \ref{thm:lower} in Section \ref{sec:appendix:lower}.
%The locally refined upper bound is established in Section \ref{sec:minimax_refined_upper}.

\subsection{Proof of Lemma \ref{lem:minimax_lower_split}} \label{sec:appendix:proof_minimax_lower_split}

\begin{proof}[Proof of Lemma \ref{lem:minimax_lower_split}]
%Consider two MDPs $\cM_1:=M(p^*,p,p_3)$ and $\cM_2:=M(p,p^*,p_3)$ for some $p_3<p<p^*$, where $p_3,p,p^*\in [0,1]$ is to be decided. Let $\cD$ be the data collecting process as described in Section \ref{sec:lower_minimax_elements}, with action $a_j$ taken $n_j$ many times, $j\in [A]$. %We assume $\frac{n_1}{n_2}\in [c^{-1},c]$ for some constant $c>0$ and place no restrictions on others.
We consider two linear MDPs $\cM_1 = M (p^*, p, p)$ and $\cM_2 = M(p, p^*, p)$ in the class $\mathfrak{M}$ defined in Equation \eqref{eq:define_hard_class}. 
As we have $p ^* > p$, 
by Equations \eqref{eq:w666} and \eqref{eq:w888},
the optimal policy for $\cM_1$ satisfies $\pi_1^{*,1}(a_1\given x_0) = \ind \{a_1=b_1\}$, which always chooses the action $b_1$ at the first step $h=1$, while the optimal policy for $\cM_2$ satisfies $\pi_1^{*,2}(a_1\given x_0) = \ind\{a_1=b_2\}$, which always chooses the action $b_2$ at the first step $h=1$. 
Given the dataset $\cD$, we denote by $\pi =\{\pi_h\}_{h=1}^H = \texttt{Algo}(\cD)$ the output of any offline RL algorithm.
Recall that $\sum_{j=1}^A \pi_1(b_j\given x_0)=1$.
By Equation \eqref{eq:suboptimality_hard_instance}, 
 the suboptimality of $\pi$ for $\cM_1$ is
\#\label{eq:lower_M1}
\text{SubOpt}(\cM_1,\pi;x_0) & =  \Bigl ( p^*   - p^* \cdot \pi_1(b_1\given x_0)   - \sum_{j=2}^A p \cdot  {\pi}_1(b_j\given x_0) \Bigr ) \cdot (H-1)  \notag \\
& =  ( p^* - p ) \cdot  \bigl (1-{\pi}_1(b_1 \given x_0) \bigr )\cdot (H-1) .
\#
Similarly, the suboptimality of $\pi $ for $\cM_2$ is 
\# \label{eq:lower_M2}
\text{SubOpt}(\cM_2,\pi;x_0) =  (p^*-p) \cdot  \big (1-{\pi}_1(b_2\given x_0) \big) \cdot (H-1).
\#
Recall that we define $n_{j} = \sum_{\tau = 1}^K \ind\{ a_1^\tau  = b_j \}$ for all $j \in [A]$.
Combining Equations \eqref{eq:lower_M1} and \eqref{eq:lower_M2}, 
we have
\$
&\max_{\ell \in \{ 1,2\} }\  \sqrt{n_{\ell}}\cdot \EE_{\cD\sim \cM_{\ell} }\Big [\text{SubOpt} \big (\cM_{\ell},\texttt{Algo}(\cD);x_0 \big ) \Big ]  \\
& \qquad \geq \frac{\sqrt{n_1 n_2}}{\sqrt{n_1}+\sqrt{n_2}} \cdot \biggl (  \EE_{\cD\sim \cM_1} \Bigl [\text{SubOpt}\bigl(\cM_1,\texttt{Algo}(\cD);x_0\bigr) \Big ] +  \EE_{\cD\sim \cM_2} \Bigl [\text{SubOpt} \big (\cM_2,\texttt{Algo}(\cD);x_0 \big ) \Big ] \biggr ) \\
& \qquad = \frac{\sqrt{n_1n_2}}{\sqrt{n_1}+\sqrt{n_2}} \cdot  (p^*-p) \cdot (H-1)\cdot  \Big(\EE_{\cD\sim \cM_1} \big[1-{\pi}_1(b_1\given x_0) \big] + \EE_{\cD\sim \cM_2} \big[1- {\pi}_1(b_2\given x_0)\big]\Big)\\
&\qquad \geq \frac{\sqrt{n_1n_2}}{\sqrt{n_1}+\sqrt{n_2}} \cdot (p^*-p) \cdot (H-1) \cdot  \Big(\EE_{\cD\sim \cM_1}\big[1-{\pi}_1(b_1\given x_0)\big] + \EE_{\cD\sim \cM_2}\big[ {\pi}_1(b_1\given x_0)\big]\Big),
\$
where $\EE_{\cD\sim \cM_\ell}$ is the expectation taken with respect to the randomness of  $\cD$, which is compliant with the underlying MDP $\cM_\ell$ for all $\ell\in\{1,2\}$. % when the underlying  MDP is  $\cM_\ell$. 
Here the first inequality follows from the fact that $\max \{ x, y\} \geq a\cdot x + (1- a)\cdot y$, for all $a\in [0, 1] $ and all $x, y \geq 0$.
Therefore, we conclude the proof of Lemma \ref{lem:minimax_lower_split}. 
\end{proof}
%\begin{flushleft}

\subsection{Suboptimality of PEVI on $\mathfrak{M}$}
In this section, we establish the suboptimality of PEVI for the linear MDPs in the class $\mathfrak{M}$. %Since MDPs in $\mathfrak{M}$ are special cases of linear MDP, a direct application of Theorem \ref{thm:regret_upper_linear} provides an upper bound on the suboptimality. 
We consider any linear MDP $\cM=M(p_1,p_2,p_3)\in \mathfrak{M}$ and the dataset $\cD = \{ (x_h^\tau, a_h^\tau, r_h^\tau)\}_{  \tau, h =1}^{ K, H }$ compliant with $\cM$, which is constructed in Section \ref{sec:lower_minimax_elements}.  Recall that $n_j=\sum_{\tau=1}^K \ind \{a_1^\tau = b_j\}$ for all $j\in [A]$ and $j^*=\argmax_{j\in [A]} p_j $. We define $m_j = \sum_{\tau =1}^K \ind \{ x_2^\tau = x_j\}$ for all $j \in \{ 1, 2\}$.

\begin{lemma}[Suboptimality of PEVI]
Suppose that Assumption \ref{assump:data_generate} holds and the underlying MDP is $\cM\in \mathfrak{M}$. In Algorithm \ref{alg:pess_greedy}, we set $\lambda=1$ and $\beta = c\cdot dH\sqrt{\log(4dHK/\xi)}$. Here $c>0$ is an absolute constant and $\xi\in(0,1)$ is the confidence parameter, which are specified in Theorem \ref{thm:regret_upper_linear}.
We have
\#\label{eq:upper_expression}
&\sum_{h=1}^H \EE_{\pi^*}\Bigl[ \bigl ( \phi(s_h,a_h)^\top \Lambda_h^{-1}\phi(s_h,a_h)\bigr) ^{1/2} \Biggiven s_1=x_0\Bigr] \notag \\
&\qquad = \frac{1}{\sqrt{1 + n_{j^*}}} +(H-1)
\cdot \Bigl(  \frac{p_{j^*}}{\sqrt{1 + m_1}} +  \frac{1 - p_{j^*} }{\sqrt{ 1+ m_2 }} \Bigr),
\#
where $\EE_{\pi^*}$ is taken with respect to the trajectory induced by $\pi^*$ in $\cM$.
%Furthermore, for any constant $\delta\in (0,1)$, when $K$ is sufficiently large such that $K \geq 32 \cdot  \log (8/\delta)$, 
%we have $m_1,m_2\geq K/8$ with probability at least $1-\delta$ with respect to $\PP_{\cD}$. 
When $K$ is sufficiently large so that $K\geq 32\cdot \log(8/\xi)$, $\linpess(\cD)$ in Algorithm \ref{alg:pess_greedy} satisfies% that
\#\label{eq:hard_instance_upper_final}
\text{SubOpt}\big(\cM ,\linpess(\cD);x_0\big) %& \leq \frac{2\beta}{\sqrt{ n_{j^*}}} +2\beta(H-1)
%\cdot \Bigl(  \frac{3/4}{\sqrt{  m_1}} +  \frac{3/4 }{\sqrt{  m_2 }} \Bigr) \notag   \\
%& \leq   2 \beta  / \sqrt{n_{j^*} } + 9\beta (H-1) / \sqrt{ K } 
 \leq 9 \beta H /\sqrt{ n_{j^*}}
\#
with probability at least $1-\xi$, which is taken with respect to $\PP_{\cD}$.
\label{lem:subopt_pevi_hard}
\end{lemma}

\begin{proof}[Proof of Lemma \ref{lem:subopt_pevi_hard}]
Recall that $x_1^\tau=x_0$ for all $\tau\in [K]$.
By the definition of $\Lambda_h$ in Equation \eqref{eq:w18}, we have% the closed form
\#\label{eq:define_lambda1}
\Lambda_1 = \lambda \cdot I + \sum_{\tau=1}^K \phi(x_0,a_1^\tau)\phi(x_0,a_1^\tau)^\top = \diag(\lambda + n_1,\dots,\lambda +n_A,\lambda, \lambda ) \in \RR^{(A+2 )\times (A+2 )},
\#
where the second equality follows from the definition of $\phi$ in Equation \eqref{eq:lower_phi}.
Since $x_1,x_2 \in \cS$ are the absorbing
 states, 
for all $h \in\{2,\dots,H\}$, we have 
\#\label{eq:define_lambda2}
\Lambda_h = \lambda \cdot I + \sum_{\tau=1}^K \phi(x_h^\tau,a_h^\tau)\phi(x_h^\tau,a_h^\tau)^\top = \diag( \lambda,\dots,\lambda ,\lambda + m_1, \lambda + m_2)\in \RR^{(A+2 )\times (A+2 )},
\# 
where the second equality follows from the definition of $\phi$ in Equation \eqref{eq:lower_phi}.
%%Recall that $j^* = \argmax_{j \in \{1,2\}} \{ p_j\}$. 
%Under the optimal policy $\pi^*$ on the MDP $\cM$, 
Also, we have 
\$
\PP_{\pi^*} (s_2 = x_1) = p_{j^*},\qquad\PP_{\pi^*} (s_2 = x_2) = 1 - p_{j^*},
\$
where $\PP_{\pi^*}$ is taken with respect to the trajectory induced by $\pi^*$ in $\cM$.
Combining Equations \eqref{eq:define_lambda1} and \eqref{eq:define_lambda2}, 
we have 
\# \label{eq:lower_bound_bonus}
&  \EE_{\pi^*}\Bigl[ \bigl ( \phi(s_h,a_h)^\top \Lambda_h^{-1}\phi(s_h,a_h)\bigr) ^{1/2} \Biggiven s_1=x_0\Bigr] \notag \\
 & \qquad 
  = \begin{cases} 
(1 + n_{j^*})^{-1/2},& h=1, \\
p_{j^*} \cdot (1 + m_1)^{-1/2} + (1- p_{j^*}) \cdot (1 + m_2 )^{-1/2},& h\in\{2,\dots,H\},
 \end{cases} 
\# 
which yields Equation \eqref{eq:upper_expression}. Here we use the definition of $\phi$ in Equation \eqref{eq:lower_phi} and the regularization parameter $\lambda=1$ in Algorithm \ref{alg:pess_greedy}.
% Under the assumption that   $\cD$  is compliant with the underlying MDP  $\cM$,    Theorem \ref{thm:regret_upper_linear}  shows that the suboptimality of PEVI is bounded by 
%\begin{equation}
%\text{SubOpt}\big(\cM ,\linpess(\cD);x_0\big) \leq \frac{2\beta}{\sqrt{1 + n_{j^*}}} +2\beta(H-1)
%\cdot \Bigl(  \frac{p_{j^*}}{\sqrt{1 + m_1}} +  \frac{1 - p_{j^*} }{\sqrt{ 1+ m_2 }} \Bigr)  ,
%\label{eq:minimax_upper_linear}
%\end{equation}
%with probability at least $1-\xi / 2 $ with respect to $\PP_{\cD}$. Here 
%  the 
% parameters   $ \lambda$ and $ \beta$ are    specified as in Theorem \ref{thm:regret_upper_linear}. 
% Specifically,  we set $$\lambda = 1 , \qquad  \beta = c \cdot (A+2 ) \cdot H \sqrt{\log \bigl(4 (A+2) H K  / \xi \bigr)}$$
% with $c$ being an absolute constant. 

%Notice that $p_1, p_2$ and $p_3$ are all bounded in $[1/4, 3/4]$.
In the sequel, we lower bound $m_1$ and $m_2$ via concentration inequalities. 
By the construction of $\cD$ in Section \ref{sec:lower_minimax_elements}, for all $\tau \in [K]$ and all $j\in[A]$, given the action $a_1^\tau = b_j$, $\ind \{ x_2^\tau = x_1 \}$  is a Bernoulli random variable with the success probability $p_j$. As $p_1, p_2,p_3\in[1/4, 3/4]$, we have
\#\label{eq:mean_m1}
\EE_{\cD} [ m_1] = \sum_{\tau =1}^K \EE_{\cD}  \bigl [ \ind \{ x_2^\tau = x_1 \} \bigr ]
 = \sum_{j=1}^A p_j \cdot n_j \geq 1/4 \cdot \sum_{j=1}^A n_j = K/4 .
\#
Given the  actions $\{a_1^\tau\}_{\tau=1}^K$, $m_1$ is a sum of $K$ independent Bernoulli random variables.
By Hoeffding's inequality, for all $\xi>0$, it holds that
 \#\label{eq:apply_hoeffding}
  \bigl | m _1 - \EE_{\cD} [ m_1] \bigr | \leq \sqrt{ K /2   \cdot \log (8/ \xi)   }
 \#
 with probability at least $1- \xi /4 $, which is taken with respect to $\PP_{\cD}$.
By Equations \eqref{eq:mean_m1} and \eqref{eq:apply_hoeffding}, it holds that
\#\label{eq:apply_hoeffding1}
m_1 \geq K / 4 - \sqrt{K  /2   \cdot  \log (8/ \xi)}
\#
with probability at least $1 -  \xi / 4$, which is taken with respect to $\PP_{\cD}$. 
Similarly, we have
\$
\EE_{\cD} [ m_2 ] = \sum_{\tau = 1}^ K \EE_{\cD} \bigl [ \ind \{ x_2^\tau = x_2 \} \bigr ]  = \sum _{j = 1 }^ A ( 1- p_j) \cdot n_j \geq 1/4\cdot \sum_{j=1}^A  n_j\geq K / 4. 
\$
By Hoeffding's inequality, it holds that
\#\label{eq:apply_hoeffding2}
m_2 \geq K/4 -  \sqrt{ K/2 \cdot \log (8/\xi ) }
\#
with probability at least $1 - \xi/ 4$, which is taken with 
respect to $\PP_{\cD}$.
We define the event
\#\label{eq:def_event_K/8}
\overline \cE   = \{ m_1\geq    K/ 8,~ m_2 \geq  K / 8   \}.
\#
Combining Equations \eqref{eq:apply_hoeffding1} and \eqref{eq:apply_hoeffding2}, by the union bound,
when $K$ is sufficiently large so that $K \geq 32 \cdot  \log (8/\xi)$, we have 
$\PP_{\cD}(\overline{\cE})\geq 1-\xi/2$.

Meanwhile, by Theorem \ref{thm:regret_upper_linear} with the regularization parameter $\lambda=1$ and the confidence parameter $\xi/2$, it holds that
\begin{equation}
\text{SubOpt}\big(\cM ,\linpess(\cD);x_0\big) \leq \frac{2\beta}{\sqrt{1 + n_{j^*}}} +2\beta(H-1)
\cdot \Bigl(  \frac{p_{j^*}}{\sqrt{1 + m_1}} +  \frac{1 - p_{j^*} }{\sqrt{ 1+ m_2 }} \Bigr) 
\label{eq:minimax_upper_linear}
\end{equation}
with probability at least $1-\xi/2$, which is taken with respect to $\PP_{\cD}$. 
%where we set $\lambda=1$ and $\beta = c\cdot dH\sqrt{\log(4dHK/\xi)}$ with the same absolute constant $c>0$ as in Theorem \ref{thm:regret_upper_linear}.
By the union bound, on the two events defined in Equations \eqref{eq:def_event_K/8} and \eqref{eq:minimax_upper_linear}, respectively, it holds that
\#
\text{SubOpt}\big(\cM ,\linpess(\cD);x_0\big) & \leq \frac{2\beta}{\sqrt{ n_{j^*}}} +2\beta(H-1)
\cdot \Bigl(  \frac{3/4}{\sqrt{  m_1}} +  \frac{3/4 }{\sqrt{  m_2 }} \Bigr) \notag   \\
& \leq   2 \beta  / \sqrt{n_{j^*} } + 9\beta (H-1) / \sqrt{ K }  \leq 9 \beta H /\sqrt{ n_{j^*}}\notag
\#
with probability at least $1-\xi$, which is taken with respect to $\PP_{\cD}$.
Here the first inequality follows from the fact that $p_{j^*}\in [1/4,3/4]$, the second inequality follows from the fact that $m_1,m_2\geq K/8$ on $\overline{\cE}$ defined in Equation \eqref{eq:def_event_K/8}, and the last inequality follows from the fact that $n_{j^*}\leq K$.
Therefore, we conclude the proof of Lemma \ref{lem:subopt_pevi_hard}.
\end{proof}

\subsection{Proof of Theorem \ref{thm:lower}} \label{sec:appendix:lower}

\begin{proof}[Proof of Theorem \ref{thm:lower}]
We consider two linear MDPs $\cM_1=M(p^*,p,p)$ and $\cM_2= M(p,p^*,p)$ in the class $\mathfrak{M}$ and
the dataset $\cD$ compliant with $\cM_1$ or $\cM_2$, which is constructed in Section \ref{sec:lower_minimax_elements}.
We additionally assume that $n_1,n_2\geq 4$ and $1/ \bar{c}\leq n_1/n_2\leq \bar{c}$ for an absolute constant $\bar{c}>0$.
For the policy $\pi = \{\pi_h\}_{h=1}^H =\texttt{Algo} (\cD)$ constructed by any offline RL algorithm, 
recall %the definition of 
the  test function $\psi_{\texttt{Algo}} (\cD)$ defined in Equation \eqref{eq:define_algo_test}, which is constructed 
for the hypothesis testing problem defined in Equation \eqref{eq:testing}. 
By Equation \eqref{eq:risk_test_algo}, we have
\#\label{eq:proof_minimax_lower_reduce_test}
   &\EE_{\cD\sim \cM_1}\big[1- {\pi}_1(b_1\given x_0)\big] + \EE_{\cD\sim \cM_2}\big[{\pi}_1(b_1\given x_0)\big] \notag \\
   &\qquad  =  \EE_{\cD \sim \cM_1} \bigl [ \ind    \{  \psi_{\texttt{Algo}} (\cD) = 1  \} \bigr ] + \EE_{\cD \sim \cM_2} \bigl [ \ind   \{ \psi_{\texttt{Algo}} (\cD) = 0  \}   \bigr ] \notag \\
&\qquad  \geq    1-\mathrm{TV} (\PP_{\cD\sim \cM_1}  ,  \PP_{\cD\sim \cM_2}) \geq 1- \sqrt{ \text{KL}(\PP_{\cD\sim \cM_1}\,\|\,\PP_{\cD\sim \cM_2}) / 2 },
\#
where the first inequality  follows from the definition of the  total variation distance, while the second inequality follows from Pinsker's inequality. 
Here for each $\ell\in\{1,2\}$, $\PP_{\cD\sim \cM_\ell}$ is taken with respect to the randomness of $\cD$ when it is compliant with $\cM_\ell$. Also, we use $\mathrm{TV} $ and $\mathrm{KL} $ to denote the total variation distance and the Kullback-Leibler (KL-)divergence, respectively.

Recall the mapping from the rewards $\{ r_2^\tau \}_{\tau = 1}^K$ to 
the relabeled rewards 
$\{ \kappa_j^i \}_{i,j=1}^{n_j,A}$ defined in Equation \eqref{eq:w88888}.
 Also, recall that 
the actions $\{ a_h^\tau \}_{\tau \in [K], h \geq 2}$
are chosen arbitrarily but  fixed in the dataset $\cD$. Since $x_1,x_2\in \cS$ are the absorbing states, we have $\PP_{\cD\sim \cM_\ell}(\cD) = \PP_{\cD\sim \cM_\ell}(\cD_1)$ for all $\ell\in\{1,2\}$, where %we denote 
 $\cD_1 = \{(x_1^\tau, a_1^\tau, x_2^\tau, r_2^\tau) \}_{\tau = 1}^K$ is the reduced dataset.
 By Equation \eqref{eq:data_likelihood}, the probabilities of observing $\cD_1$ in $\cM_1$ and $\cM_2$ take the form
\#
\PP_{\cD\sim \cM_1}(\cD_1 )
 & =( p^*)^{\sum_{i=1}^{n_1}\kappa_1^i }\cdot (1-p^*)^{n_1-\sum_{i=1}^{n_1}\kappa _1^i}  \cdot  \prod_{\stackrel{j\in[A]}{j\neq 1}} \Big(p^{\sum_{i=1}^{n_j} \kappa _j^i}\cdot (1-p )^{n_j-\sum_{i=1}^{n_j}\kappa_j^i }\Big), \notag\\%\label{eq:compute_likelihood1}\\
 \PP_{\cD\sim \cM_2}(\cD_1)&  = ( p^*)^{\sum_{i=1}^{n_2}\kappa_2^i }\cdot (1-p^*)^{n_2-\sum_{i=1}^{n_2}\kappa _2^i }  \cdot  \prod_{\stackrel{j\in[A]}{j\neq 2}} \Big(p^{\sum_{i=1}^{n_j} \kappa _j^i }\cdot (1-p )^{n_j-\sum_{i=1}^{n_j}\kappa_j ^i}\Big), \label{eq:compute_likelihood2}
\# 
respectively.
Here we use the fact that $p_1=p^*$ and $p_2=p$ in $\cM_1=M(p^*,p,p)$, while $p_1=p$ and $p_2=p^*$ in $\cM_2=M(p,p^*,p)$, where $p^*>p$.
By Equation \eqref{eq:compute_likelihood2}, we have%the KL divergence between  $\PP_{\cD\sim \cM_1}$ and $\PP_{\cD\sim \cM_2}$ is 
\$
\text{KL}(\PP_{\cD\sim \cM_1}\,\|\,\PP_{\cD\sim \cM_2}) &= \EE_{\cD\sim \cM_1}\bigg[ \Big( \sum_{i=1}^{n_1}\kappa _1^i  - \sum_{i=1}^{n_2}\kappa _2^i  \Big)\cdot \log \frac{p^*\cdot (1-p)}{p\cdot (1-p^*)} + (n_1-n_2)\cdot \log \frac{1-p^*}{1-p}\bigg]\\
&= (n_1p^*-n_2p)\cdot \log \frac{p^*\cdot(1-p)}{p\cdot(1-p^*)}+ (n_1-n_2)\cdot \log \frac{1-p^*}{1-p}\\
&= n_1\cdot \Big(p^*\cdot \log\frac{p^*}{p} +(1-p^*)\cdot \log \frac{1-p^*}{1-p} \Big) +n_2\cdot \Big(p \cdot  \log\frac{p}{p^*} +(1-p)\cdot \log \frac{1-p}{1-p^*} \Big),
\$
where the second equality follows from Equation \eqref{eq:bernoulli_kappa}. 
Note that for all $x\in (-1,1)$, it holds that $\log(1+x)\leq x$.
Hence, when $p^*-p < \min\{p, 1-p\}$, we have
\$
p^*\cdot \log\frac{p^*}{p} +(1-p^*)\cdot \log \frac{1-p^*}{1-p} &= p^*\cdot \log\Big(1+\frac{p^*-p}{p}  \Big)+(1-p^*)\cdot \log\Big(1+\frac{p-p^*}{1-p}  \Big)\\
&\leq p^* \cdot \frac{p^*-p}{p} + (1-p^*)\cdot \frac{p-p^*}{1-p} = \frac{(p^*-p)^2}{p\cdot(1-p)}.
\$ 
Similarly, when 
 $p^*-p < \min\{p^*, 1-p^*\}$, we have
 \$
 p \cdot \log\frac{p}{p^*} +(1-p) \cdot \log \frac{1-p}{1-p^* }  \leq  \frac{(p^*-p)^2}{p^* \cdot(1-p^*)}.
 \$ 
Recall that  $n_1,n_2\geq 4$ and $ 1/ \bar{c} \leq n_1 / n_2 \leq \bar{c}$ for an absolute constant $\bar{c}>0$. We set
\# \label{eq:set_diff_p} 
p^* =\frac{1}{2} + \frac{1}{8}\cdot \sqrt{\frac{3}{2\cdot(n_1+n_2)}},\qquad p =\frac{1}{2} - \frac{1}{8}\cdot \sqrt{\frac{3}{2\cdot(n_1+n_2)}}
\# 
such that $p^*,p\in[1/4,3/4]$, $0\leq p^*-p\leq 1/4$, and $p^*-p < \min\{p, 1-p,p^*, 1-p^*\}$.
Hence, the KL-divergence is upper bounded as
\#\label{eq:kl_bound}
\text{KL}(\PP_{\cD\sim \cM_1}\,\|\,\PP_{\cD\sim \cM_2}) &\leq \frac{n_1\cdot (p^*-p)^2}{p\cdot(1-p)} + \frac{n_2\cdot (p^*-p)^2}{p^*\cdot(1-p^*)} \notag\\
&\leq 16/3\cdot (n_1+n_2) \cdot (p^*-p)^2 \leq 1/2, 
\#
  where the second inequality follows from the fact that $p,p^*\in [1/4, 3/4]$ and the last inequality follows from  Equation \eqref{eq:set_diff_p}.
By Equations \eqref{eq:proof_minimax_lower_reduce_test} and \eqref{eq:kl_bound}, we have
\#\label{eq:kl_bound_final}
&\EE_{\cD\sim \cM_1}\big[1- {\pi}_1(b_1\given x_0)\big] + \EE_{\cD\sim \cM_2}\big[{\pi}_1(b_1\given x_0)\big]  \geq 1- \sqrt{ \text{KL}(\PP_{\cD\sim \cM_1}\,\|\,\PP_{\cD\sim \cM_2}) / 2 } \geq 1/2.
\#
Combining Equations \eqref{eq:set_diff_p} and \eqref{eq:kl_bound_final}, for the output $\pi = \{\pi_h\}_{h=1}^H = \texttt{Algo}(\cD)$  of any offline RL algorithm, we have
\# \label{eq:bound_exp_subopt}
&\max_{\ell \in \{ 1,2\} }  \sqrt{n_\ell}\cdot \EE_{\cD\sim \cM_\ell} \Bigl  [\text{SubOpt} \bigl (\cM_\ell,\texttt{Algo}(\cD);x_0 \big) \Big ]  \notag \\
&\qquad \geq \frac{\sqrt{n_1n_2}}{\sqrt{n_1}+\sqrt{n_2}} \cdot  (p^*-p) \cdot (H-1) \cdot  \Big(\EE_{\cD\sim \cM_1}\big[1- {\pi}_1(b_1\given x_0)\big] + \EE_{\cD\sim \cM_2}\big[{\pi}_1(b_1\given x_0)\big]\Big) \notag   \\
&\qquad \geq \frac{\sqrt{n_1 n_2}}{\sqrt{n_1}+\sqrt{n_2}} \cdot \frac{1}{4}\cdot \sqrt{\frac{3}{2\cdot(n_1+n_2)}} \cdot \frac{ H-1} {2}  \notag \\
& \qquad =  \frac{ \sqrt{ n_1 / n_2 }  }{( \sqrt{n_1/ n_2 } + 1) \cdot \sqrt{ 1 + n_1/ n_2 }} \cdot \frac{ \sqrt{3}}{ 8\sqrt{2}} \cdot (H-1)\notag\\
&\qquad\geq  C'\cdot(H-1)
\#
 for an absolute constant 
 \# \label{eq:constant_Cprime}
C' = \frac{\sqrt{3}}{8\sqrt{2}}\cdot \frac{1}{ (\sqrt{\bar{c}} + 1 ) \cdot \sqrt{  \bar{c} \cdot (\bar{c}+1)}  }  > 0.
\#
Here the first inequality follows from Lemma \ref{lem:minimax_lower_split} and the last inequality follows from the fact that $1/\bar{c}\leq n_1/n_2\leq \bar{c}$ for an absolute constant $\bar{c}>0$.

By the definition of $\cM_1$ and $\cM_2$ in Equation \eqref{eq:w999}, at the first step $h = 1$, the optimal policy $\pi^{*,1}$ for $\cM_1$ always chooses the action $b_1$, while the optimal policy $\pi^{*,2}$ for $\cM_2$ always chooses the action $b_2$.
Recall that $n_j = \sum_{\tau=1}^K \ind\{a_1^\tau=b_j\}$ for all $j\in [A]$ and $m_j = \sum_{\tau=1}^K \ind\{x_2^\tau=x_j\}$ for all $j\in \{1,2\}$. Also, recall that $j^*=\argmax_{j\in [A]}p_j=1$ for $\cM_1$, while $j^*=\argmax_{j\in[A]}p_j = 2$ for $\cM_2$.  Hence, we have $n_{j^*}=n_\ell$ in $\cM_\ell$ for all $\ell\in\{1,2\}$.
 By  Lemma \ref{lem:subopt_pevi_hard}, for all $\ell \in \{ 1, 2\}$, we have 
\$
& \sum_{h=1}^H \EE_{\pi^{*,\ell}, \cM_{\ell} }\Bigl[ \bigl ( \phi(s_h,a_h)^\top \Lambda_h^{-1}\phi(s_h,a_h)\bigr) ^{1/2} \Biggiven s_1=x_0\Bigr]  = 
\frac{1}{\sqrt{1 + n_{\ell }}} + (H-1) \cdot  \Bigl(  \frac{p^*}{\sqrt{1 + m_1}} +  \frac{ 1 - p^* }{\sqrt{ 1+ m_2 }} \Bigr),   
\$ 
where $\EE_{\pi^{*,\ell},\cM_{\ell}}$ is taken with respect to the trajectory induced by  $\pi^{*,\ell}$ in $\cM_{\ell}$ for all $\ell\in\{1,2\}$.
Given the actions $\{a_1^\tau\}_{\tau=1}^K$, $m_1$ and $m_2$ are two sums of $K$ independent Bernoulli random variables. 
We define the event
\#\label{eq:lower_bound_event} 
\overline \cE   = \{ m_1\geq    K/ 8,~ m_2 \geq  K / 8   \}.
\#
%In the sequel, 
%we define an event $\overline \cE  $ as 
%\#\label{eq:lower_bound_event} 
%\overline \cE   = \{ m_1\geq    K/ 8, m_2 \geq  K / 8   \}.
%\#
%Combining Equations \eqref{eq:apply_hoeffding1} and \eqref{eq:apply_hoeffding2}, 
On $\overline{\cE}$, we have
\#\label{eq:info_quantity_upper}
& \sum_{h=1}^H \EE_{\pi^{*,\ell}, \cM_{\ell} }\Bigl[ \bigl( \phi(s_h,a_h)^\top \Lambda_h^{-1}\phi(s_h,a_h)\bigr) ^{1/2} \Biggiven s_1=x_0\Bigr] \notag \\
& \qquad  \leq  
\frac{1}{\sqrt{  n_{\ell }}} + (H-1) \cdot  \Bigl(  \frac{3/4 }{\sqrt{  m_1}} +  \frac{ 3/4 }{\sqrt{  m_2 }} \Bigr) \leq 6 (H-1) / \sqrt{ n_{\ell} } .
\#
Hence,  we have 
\# \label{eq:lower_bound_final1}
 &\max_{\ell \in \{1, 2\}}  \EE_{\cD\sim \cM_\ell}\Bigg[ \frac{\text{SubOpt}\big(\cM_{\ell},\texttt{Algo}(\cD);x_0\big)}{\sum_{h=1}^H\EE_{\pi^{*,\ell},\cM_\ell}\Bigl[ \bigl(\phi(s_h,a_h)^\top \Lambda_h^{-1}\phi(s_h,a_h)\bigr) ^{1/2} \Biggiven s_1=x_0\Bigr]}  \Bigg] \notag \\
 & \qquad  \geq \max_{\ell \in \{1, 2\}}   \EE_{\cD\sim \cM_\ell}\Bigg[ \frac{\text{SubOpt}\big(\cM_{\ell},\texttt{Algo}(\cD);x_0\big)}{\sum_{h=1}^H\EE_{\pi^{*,\ell},\cM_\ell}\Bigl[ \bigl( \phi(s_h,a_h)^\top \Lambda_h^{-1}\phi(s_h,a_h)\bigr) ^{1/2} \Biggiven s_1=x_0\Bigr]}  \cdot \ind _{\overline \cE }\Bigg] \notag \\
 & \qquad  \geq  \max_{\ell \in \{1, 2\}}  \biggl \{  \frac{ \sqrt{n_{\ell} }} {6(H-1)} \cdot \EE_{\cD\sim \cM_\ell}\Big[  \text{SubOpt}\big(\cM_{\ell},\texttt{Algo}(\cD);x_0\big)   \cdot \ind _{\overline \cE }\Big]  \biggr \} ,
\#
where the last inequality follows from Equation \eqref{eq:info_quantity_upper} and the definition of $\overline \cE$ in Equation  \eqref{eq:lower_bound_event}.
We use $\overline \cE^c$ to denote the complement of $\overline{\cE }$. 
By Equation \eqref{eq:w888}, we have $r_h\in [0,1]$ for all $h\in[H]$ and $r_1(x_0,a)=0$ for all $a\in \cA$. Hence, the suboptimality of any policy is upper bounded by $H-1$. For all $\ell\in\{1,2\}$, we have  
\#\label{eq:lower_bound_final2}
& \sqrt{n_{\ell} } \cdot \EE_{\cD\sim \cM_\ell}\Big[  \text{SubOpt}\big(\cM_{\ell},\texttt{Algo}(\cD);x_0\big)   \cdot \ind _{\overline \cE }\Big] \notag \\
& \qquad =\sqrt{n_{\ell} } \cdot  \EE_{\cD\sim \cM_\ell}\Big[  \text{SubOpt}\big(\cM_{\ell},\texttt{Algo}(\cD);x_0\big)  \Big]  - \sqrt{n_{\ell} } \cdot  \EE_{\cD\sim \cM_\ell}\Big[  \text{SubOpt}\big(\cM_{\ell},\texttt{Algo}(\cD);x_0\big) \cdot \ind _{\overline \cE ^c } \Big]   \notag \\
& \qquad \geq \sqrt{n_{\ell} } \cdot  \EE_{\cD\sim \cM_\ell}\Big[  \text{SubOpt}\big(\cM_{\ell},\texttt{Algo}(\cD);x_0\big)  \Big]  - (H-1) \cdot \sqrt{ n_{\ell}} \cdot \PP_{\cD\sim \cM_\ell}(\overline \cE^{c} ). 
\# 
We invoke the same argument as in Equations \eqref{eq:apply_hoeffding1} and \eqref{eq:apply_hoeffding2}.
For all $\delta>0$ and all $\ell\in\{1,2\}$, as it holds that $p,p^*\geq 1/4$, 
when $K$ is sufficiently large so that $K \geq 32 \cdot  \log (4/\delta)$, we have $\PP_{\cD\sim \cM_\ell}(\overline{\cE})\geq 1-\delta$.
Hence, setting $\delta=1/K^2$, when $K$ is sufficiently large so that $K \geq 32 \cdot \log (4 K^2) = 64\cdot \log(2K)$,
we have $\PP_{\cD\sim \cM_\ell} ( \overline \cE) \geq 1 - 1/ K^2$ for all $\ell \in \{ 1, 2\}$. By Equation \eqref{eq:lower_bound_final2}, for all $\ell\in \{1,2\}$, we have
\# \label{eq:lower_bound_final3}
&\sqrt{n_{\ell} } \cdot \EE_{\cD\sim \cM_\ell}\Big[  \text{SubOpt}\big(\cM_{\ell},\texttt{Algo}(\cD);x_0\big)   \cdot \ind _{\overline \cE }\Big] \notag \\
&\qquad  \geq \sqrt{n_{\ell} } \cdot  \EE_{\cD\sim \cM_\ell}\Big[  \text{SubOpt}\big(\cM_{\ell},\texttt{Algo}(\cD);x_0\big)  \Big] - (H-1)\cdot \sqrt{ n_{\ell}}/K^2.
\#
%when $K$ is sufficiently large so that $K\geq 64\log(2K)$.

By Equations \eqref{eq:bound_exp_subopt}, \eqref{eq:lower_bound_final1}, and \eqref{eq:lower_bound_final3}, when $K$ is sufficiently large so that $K\geq 64\cdot \log(2K)$ and $K\geq 2/C'$, we have
\# \label{eq:lower_bound_final4}
 &\max_{\ell \in \{1, 2\}}  \EE_{\cD\sim \cM_\ell}\Bigg[ \frac{\text{SubOpt}\big(\cM_{\ell},\texttt{Algo}(\cD);x_0\big)}{\sum_{h=1}^H\EE_{\pi^{*,\ell},\cM_\ell}\Bigl[ \bigl(\phi(s_h,a_h)^\top \Lambda_h^{-1}\phi(s_h,a_h)\bigr) ^{1/2} \Biggiven s_1=x_0\Bigr]}  \Bigg] \notag \\
 & \qquad \geq \max_{\ell \in \{1, 2\}}  \biggl \{  \frac{ \sqrt{n_{\ell} }} {6(H-1)} \cdot \EE_{\cD\sim \cM_\ell}\Big[  \text{SubOpt}\big(\cM_{\ell},\texttt{Algo}(\cD);x_0\big)   \cdot \ind _{\overline \cE }\Big]  \bigg \}\notag \\
 & \qquad  \geq  \max_{\ell \in \{1, 2\}}  \biggl \{  \frac{ \sqrt{n_{\ell} }} {6(H-1)} \cdot \EE_{\cD\sim \cM_\ell}\Big[  \text{SubOpt}\big(\cM_{\ell},\texttt{Algo}(\cD);x_0\big)  \Big]  - \frac{n_\ell}{6K^2} \bigg \}\notag \\
 &\qquad \geq  \frac{1}{6(H-1)}\cdot \max_{\ell \in \{1, 2\}}     \sqrt{n_{\ell} }\cdot \EE_{\cD\sim \cM_\ell}\Big[  \text{SubOpt}\big(\cM_{\ell},\texttt{Algo}(\cD);x_0\big)  \Big]  - \frac{1}{6K}\notag \\
  &\qquad \geq  C'/6 - 1/(6K)\notag\\
  &\qquad\geq C'/12.
\#
Here $C' > 0$ is the absolute constant defined in Equation \eqref{eq:constant_Cprime}, 
the first inequality follows from Equation \eqref{eq:lower_bound_final1}, the second inequality follows from Equation \eqref{eq:lower_bound_final3}, the third inequality follows from the fact that $n_\ell\leq K$ for all $\ell\in\{1,2\}$, and the fourth inequality follows from Equation \eqref{eq:bound_exp_subopt}. As $\mathfrak{M}$ defined in Equation \eqref{eq:define_hard_class} is a subclass of linear MDPs, 
Equation \eqref{eq:lower_bound_final4} implies 
the lower bound given in Equation \eqref{eq:minimax_lower} with $c = C'/ 12$ for $K$ sufficiently large  so that $K\geq 64\cdot \log(2K)$ and $K\geq 2/C'$.
 Therefore, we conclude the proof of  Theorem  \ref{thm:lower}.
\end{proof}
%\end{flushleft}

\subsection{Locally Refined Upper Bounds}\label{sec:minimax_refined_upper}

Since $\mathfrak{M}$ is a class of linear MDPs, 
  Theorem \ref{thm:regret_upper_linear} yields an upper bound on the suboptimality of $\linpess(\cD)$ constructed by Algorithm \ref{alg:pess_greedy}, which  is minimax optimal up to $\beta$ and absolute constant.
Focusing on $\mathfrak{M}$, with a different choice of the $\xi$-uncertainty 
quantifier tailored for $\cM$, 
PEVI achieves a more refined local minimax optimality. 

Specifically, for  any $\cM  = M(p_1, p_2, p_3) \in \mathfrak{M}$, recall that $x_1,x_2\in\cS$ are the absorbing states. 
By the  construction of the reward function in Equation \eqref{eq:w888}, for any function $V$ and any $h \geq 2$, 
the Bellman operator $\BB_h$ defined in Equation \eqref{eq:def_bellman_op} takes the form
\$
(\BB_h V) (x_1, a) = r_h( x_1, a) + V(x_1) = V(x_1 ) + 1 ,  \qquad (\BB_h V) (x_2, a ) = r_h (x_2, a) + V(x_2) = V(x_2),
\$
for all $a \in \cA$. 
Recall that the initial state is fixed to $x_0$. 
For any $ j \in [A]$, we have 
\#\label{eq:bellman_first_step} 
(\BB_1 V) (x_0, b_j) = p_j  \cdot V(x_1) + (1- p_j) \cdot V(x_2).
\#
Based on the dataset $\cD = \{ (x_h^\tau, a_h^\tau, r_h^\tau )\}_{\tau, h =1}^{K,H}$ that is compliant with $\cM$,
we construct the estimated Bellman operator $\hat B_h $ and value function $\hat V_h$ for all $h\in[H]$ as follows.  
To begin with, we define $\hat V_{H+1}$ as a zero function. For all $h \geq 2$,  we define 
\#\label{eq:def_hatV_h_refined}
 \hat V_h (x_1) = H-h+1,\quad \hat V_{h} (x_2) = 0.
 \#
For all $h\geq 2$ and all $a \in \cA$, we define the estimated Bellman update $\hat \BB_h\hat V_{h+1}$ as
\begin{equation}
\begin{split}
&( \hat\BB_h \hat{V}_{h+1} )(x_1,a)= (\BB_h \hat{V}_{h+1}) (x_1,a)  = \hat V_{h+1} (x_1) + 1= H- h+1, \\
 &(\hat\BB_h \hat{V}_{h+1}) (x_2,a) = (\BB_h \hat{V}_{h+1}) (x_2,a)= \hat{V}_{h+1} (x_2)=0. 
\end{split}
\label{eq:minimax_fitted_value_later}
\end{equation} 
For 
  $h=1$ %we define the estimated Bellman update by
  %replacing  Equation \eqref{eq:bellman_first_step} with its empirical estimator.
  %Specifically, since $r_1(x_0, a) = 0$ for all $a \in \cA$, by Equation \eqref{eq:minimax_fitted_value_later},
  and all $j \in [A]$ with $n_j>0$, we define 
\#
(\hat\BB_1 \hat{V}_{2} )(x_0,   b  _j) &= \frac{1}{n_j}\sum_{\tau=1}^K\ind \{a_1^\tau=b_j\} \cdot  \hat{V}_2(x_2^\tau) \notag \\
& = \frac{H-1}{n_j}\sum_{\tau=1}^K\ind \{(a_1^\tau, x_2^\tau) =(b_j, x_1) \}   =  \frac{H-1}{n_j}\sum_{\tau=1}^K\ind \{a_1^\tau=b _j\} \cdot  r_2^\tau.
\label{eq:minimax_fitted_value_1}
\# 
For all $j\in [A]$ with $n_j = 0$, we simply set $(\hat \BB_1 \hat V_2 )(x_0, b_j) = 0$. 
 Furthermore, for any $\xi > 0$, 
 we define 
 \begin{equation}
 \Gamma_1^{\xi}(x_0,b_j) = (H-1) \cdot \sqrt{ \log(2A/\xi)  \big / (1+n_j) } , \quad \forall  j \in [A],  \qquad \Gamma_h^\xi(\cdot,\cdot)\equiv 0,\quad \forall h\geq 2.
 \label{eq:gamma_minimax}
 \end{equation}
 Thus, employing the empirical Bellman update $\hat \BB_h \hat V_{h+1}$ given in Equations \eqref{eq:minimax_fitted_value_later} and \eqref{eq:minimax_fitted_value_1}, and function $\Gamma_h^\xi$ defined in Equation \eqref{eq:gamma_minimax} for all $h\in[H]$, we obtain an instantiation of PEVI specified in Algorithm~\ref{alg:pess_greedy_general}.
 We use $\pess^*(\cD)$ to denote the output policy, whose suboptimality is established in the 
  following proposition.

\begin{prop}[Local Optimality of PEVI]
For any $\cM  \in \mathfrak{M}$ and any dataset $\cD$ that is compliant with $\cM $,   
we assume that $n_{j^*} = \sum_{\tau =1}^K \ind\{ a_1^\tau = b_j^*\} \geq 1$ for the optimal action $b_{j^*}$, where $j^*=\argmax_{j\in\{1,2\}}\{p_j\}$.
Then the   following statements hold: (i) $\{\Gamma_h^\xi \}_{h=1}^H$ defined in Equation \eqref{eq:gamma_minimax}   are    $\xi$-uncertainty quantifiers satisfying Equation \eqref{eq:def_event_eval_err_general}; (ii) we have
$$
\text{SubOpt}\big(\pess^*(\cD);x_0\big) \leq c  \cdot H   \sqrt{\log(A/\xi)}\cdot\sum_{h=1}^H  \EE_{\pi^*}\Bigl[ \bigl ( \phi(s_h,a_h)^\top \Lambda_h^{-1}\phi(s_h,a_h)\bigr) ^{1/2} \Biggiven s_1=x_0\Bigr] 
$$
with probability at least least $1-\xi$, which is taken with respect to $\PP_{\cD}$.
   Here $c>0$ is an absolute constant and $\Lambda_h$ is defined in Equation \eqref{eq:w18}  with   $\lambda=1$. $\EE_{\pi^*}$ is taken with respect to the trajectory induced by the optimal policy $\pi^*$ in the underlying MDP $\cM $.
\label{prop:minimax_upper}
\end{prop}

We remark that based on the $\xi$-uncertainty quantifiers tailored for linear MDPs in $\mathfrak{M}$, 
Proposition \ref{prop:minimax_upper} establishes a tighter upper bound than that in Theorem \ref{thm:regret_upper_linear}.
Specifically, Equation \eqref{eq:hard_instance_upper_final} shows that directly applying Theorem \ref{thm:regret_upper_linear} yields  
 an $\tilde\cO(  H ^2 A /\sqrt{n_{j^*}  } ) $ suboptimality upper bound, 
 where $\tilde \cO(\cdot )$ omits logarithmic terms and absolute constants. 
In contrast, as shown in Equation \eqref{eq:loc_upper1},
$\pess^*(\cD)$ achieves an improved $\tilde\cO(  H   /\sqrt{n_{j^*}  } )$ suboptimality upper bound. 
Thus, neglecting logarithmic terms and absolute constants, 
although both being minimax optimal algorithms, 
$\pess^*(\cD)$ is superior over $\pess(\cD)$ given in Algorithm \ref{alg:pess_greedy} by a factor of $HA$, and $\pess^*(\cD)$ is minimax optimal up to a factor of $H$.

\begin{proof}[Proof of Proposition \ref{prop:minimax_upper}]
	The proof consists of two steps. 
	In the first step, we prove   that 
	$\{\Gamma_h^\xi \}_{h=1}^H$  given in Equation \eqref{eq:gamma_minimax} are $\xi$-uncertainty quantifiers. 
	In the second step, we apply Theorem \ref{thm:regret_upper_bound_general} and establish the upper bound. 
	
	\vspace{4pt} 
	{\noindent \bf Step (i).}
	In the sequel, we show that $\{\Gamma_h^\xi\}_{h= 1}^H $ are  $\xi$-uncertainty quantifiers. 
	For all $h\geq 2$, by the definition of estimated Bellman updates $\hat{\BB}_h\hat{V}_{h+1}$ in Equation \eqref{eq:minimax_fitted_value_later}  and the definition of $\Gamma_h$ in Equation \eqref{eq:gamma_minimax},
	we have 
	\#\label{eq:uncertainty_nj=0}
 \big |	(\hat \BB_h \hat V_{h+1}) (x,a) - ( \BB_h \hat V_{h+1}) (x, a) \bigr | = 0 = \Gamma_h^{\xi} (x,a) 
	\#
	for all $(x,a) \in \cS\times \cA$. 
	%Thus, it suffices to  focus on  $h = 1$. 

	Now we fix $h=1$. Recall that we define $n_j = \sum_{\tau = 1}^K \ind\{ a_1^\tau = b_j\}$ for all $j\in[A]$. 
	For any $j\in[A]$ such that $n_j=0$, recall the definition that $ (\hat \BB_1 \hat V_2) (x_0, b_j) = 0$. 
	By Equation \eqref{eq:bellman_first_step},  we have 
 $(\BB_1\hat V_2 )(x_0, b_j ) =  p_j \cdot (H-1)$.
 By the definition of $\Gamma_1^\xi$ in Equation \eqref{eq:gamma_minimax}, we have
 \$
 \Gamma_1^{\xi} (x_0, b_j ) = (H-1) \cdot \sqrt{ 2 \log (2 A / \xi) } \geq  H - 1  \geq (\BB_1\hat V_2 )(x_0, b_j ) 
 \$
 for all $j\in[A]$.
 Thus, we have 
 \$
 \big| (\hat\BB_1 \hat{V}_{2} ) (x_0,b_j) - (\BB_1 \hat{V}_{2}) (x_0,b_j) \big| \leq \Gamma_1^{\xi}(x_0,b_j)
 \$
 for all $j\in[A]$.
 Recall the mapping from the rewards $\{r_2^\tau\}_{\tau=1}^K$ to the relabeled rewards $\{ \kappa _j^i\}_{i,j=1}^{n_j,A}$ defined in Equation \eqref{eq:w88888}. For any $j\in[A]$ such that $n_j \geq 1$, we
  consider the $\sigma$-algebras
  \$
  \cF_{\tau}^j=\sigma \big(\{\kappa_{j}^i \}_{i=1}^{\tau} \big),\quad \tau\in [n_j].
  \$
	Since $\cD$ is compliant with $\cM$, by Equation \eqref{eq:bernoulli_kappa}, 
		 $\{ \kappa_{j}^i  - p_j\}_{i=1}^{n_j} $ is a martingale difference sequence adapted to filtration $\{\cF_{i}\}_{i=1}^{n_j}$. 
		 Applying  Azuma-Hoeffding's inequality,  we have
	\#\label{eq:uncertainty_lower}
& 	\PP_{\cD}\bigg(\Big|\frac{1}{n_j}\sum_{i=1}^{n_j} (\kappa_j ^i  - p_j ) \Big| \geq \sqrt{   \log ( 2A / \xi ) / ( 1+ n_j )}    \bigg)  \notag \\
&\qquad \leq  2\exp\bigl(-2n_j / (1+ n_j) \cdot \log ( 2A /\xi)  \bigr ) \leq    2\cdot (2A/\xi)^{-2n_j/(1+n_j)} \leq \xi / A,
	\#
	where the last inequality follows from the fact that $n_j \geq 1$.
  By Equation \eqref{eq:bellman_first_step}, we have
  \#\label{eq:b1v2_refined}
(  \BB_1 \hat V_2) (x_0, b_j ) = p_j \cdot \hat{V}_2(x_1) + (1-p_j) \cdot \hat{V}(x_2) = p_j\cdot (H-1),
  \#
  where the second equality follows from the definition of $\hat{V}_h$ for all $h\geq 2$ in Equation \eqref{eq:def_hatV_h_refined}.
  %Combining Equations \eqref{eq:minimax_fitted_value_1} and \eqref{eq:b1v2_refined}, 
	%Combining Equations  \eqref{eq:minimax_fitted_value_later},  \eqref{eq:minimax_fitted_value_1}, \eqref{eq:gamma_minimax}
	%and \eqref{eq:uncertainty_lower},
	Thus, for any fixed $j \in [A]$ such that $n_j \geq 1$, 
	we have 
	\$
	\big | ( \hat \BB_1 \hat V_2) (x_0, b_j ) -  (  \BB_1 \hat V_2) (x_0, b_j )  \big|  = ( H - 1) \cdot \Big|\frac{1}{n_j}\sum_{i=1}^{n_j} (\kappa_j ^i  - p_j ) \Big|     \leq \Gamma_h^\xi (x, b_j)
	\$
	with probability at least $1- \xi / A$, which is with respect to $\PP_{\cD}$. Here the first equality follows from Equations \eqref{eq:minimax_fitted_value_1} and \eqref{eq:b1v2_refined} and the inequality follows from Equation \eqref{eq:uncertainty_lower}. 
	Combining Equation \eqref{eq:uncertainty_nj=0} and taking the union bound over all  $j\in[A]$ such that $n_j\geq 1$,
  %That is, for the empirical Bellman operators given in Equations  \eqref{eq:minimax_fitted_value_later} and \eqref{eq:minimax_fitted_value_1}, 
  we have 
   \$
  \PP_{\cD}\Big( \big|(\hat\BB_h \hat{V}_{h+1})(x,a) - (\BB_h \hat{V}_{h+1}) (x,a) \big| \leq \Gamma_h^{\xi}(x,a),~~\forall (x,a)\in \cS\times \cA, \forall h\in[H]\Big) \geq 1-\xi. 
  %\label{eq:gamma_valid_1-xi}
  \$
	Thus, $\{\Gamma_h^\xi\}_{h=1}^H$ defined in Equation \eqref{eq:gamma_minimax}  are  $\xi$-uncertainty quantifiers. 
	
	\vspace{4pt}
	{\noindent \bf Step (ii).} 
	In the sequel, we apply Theorem \ref{thm:regret_upper_bound_general} to 
	$\pess^*(\cD)$ and establish the suboptimality upper bound in Proposition \ref{prop:minimax_upper}. 
	Specifically, by 
	Theorem \ref{thm:regret_upper_bound_general}, it holds that 
	\# \label{eq:loc_upper1}
	\text{SubOpt}\big(\pess^*(\cD);x_0\big) \leq 2 \sum_{h=1}^H \EE_{\pi^*}\big[\Gamma_h^\xi(s_h,a_h)\biggiven s_1=x_0\big] = 2(H-1) \cdot \sqrt{\frac{ \log(2A/\xi)}{1+n_{j^*}}},
	\#
  with probability at least $1- \xi$, which is with respect to $\PP_{\cD}$.
	Here the last equality follows from Equation \eqref{eq:gamma_minimax}. 
  By Equation \eqref{eq:lower_bound_bonus}, %with $\lambda = 1$,
	%, whenever $n_{j^*}\geq 1$,
	 it  holds that
	\# \label{eq:loc_upper2}
	\sum_{h=1}^H \EE_{\pi^*}\Bigl[ \bigl( \phi(s_h,a_h)^\top \Lambda_h^{-1}\phi(s_h,a_h)\bigr)^{1/2} \Biggiven s_1=x_0\Bigr] \geq  \frac{1}{\sqrt{1+n_{j^*}}} .
	\#
	Note that $\log(2A/\xi)=\log 2 + \log(A/\xi) \leq 2\log(A/\xi)$ for $A\geq 2$. 
	Combining Equations \eqref{eq:loc_upper1} and \eqref{eq:loc_upper2}, we have
	\# 
	\text{SubOpt}\big(\pess^*(\cD);x_0\big) \leq c\cdot H \sqrt{\log(A/\xi)} \cdot \sum_{h=1}^H \EE_{\pi^*}\Bigl[ \bigl ( \phi(s_h,a_h)^\top \Lambda_h^{-1}\phi(s_h,a_h)\bigr) ^{1/2} \Biggiven s_1=x_0\Bigr]\notag
	\#
	with probability at least $1 - \xi$, which is taken with respect to $\PP_{\cD}$.
	Here $c>0$ is an absolute constant that can be chosen as $c=4$. 
Therefore, we 
	 conclude the proof of Proposition \ref{prop:minimax_upper}.
\end{proof}

% { \color{blue}  
\section{Proofs of PEVI with Kernel Function Approximation}

\subsection{Proof of Theorem~\ref{thm:rkhs}} \label{app:proof_thm_rkhs}
In this part, we provide the proof of Theorem~\ref{thm:rkhs}
and the related supporting lemmas. 

\begin{proof}[Proof of Theorem~\ref{thm:rkhs}]
By Lemma~\ref{lem:rkhs_uc}, we know that $\{\Gamma_h\}_{h\in[H]}$ 
constructed in Algorithm~\ref{alg:pess_rkhs}
is a $\xi$-uncertainty quantifier. 
Applying the general conclusion of Theorem~,  
with probability at least $1-\xi$ with repsect to $\PP_\cD$, 
it holds simultaneously 
for all $x\in \cS$ and all $h\in[H]$ that
\$
 \textrm{SubOpt}\big( \pess (\cD); x \big)  
 \leq 
2  \sum_{h=1}^H \EE_{\pi^*}\Big[ \Gamma_h(s_h,a_h)\Biggiven s_1=x\Big]
\$
for $\{\Gamma_h\}_{h\in[H]}$ defined in Equation~\eqref{eq:gamma_rkhs}. 
Here $\EE_{\pi^*}$ is taken with respect to the trajectory induced by $\pi^*$
in the underlying MDP given the fixed Gram matrix $K_h\in \RR^{N_h\times N_h}$ 
and the fixed operator $k_h\colon \cZ\to \RR^{N_h}$ defined 
in Equation~\eqref{eq:def_Kkh}. 

We now proceed to obtain an upper bound of  $\Gamma_h(s_h,a_h)$ 
for all $h\in[H]$ to yield the simpler expression in Theorem~\ref{thm:rkhs}. 
In the following, we adopt an equivalent set of notations for ease of presentation. 
We formally write the inner product in $\cH$ as 
$\langle f,f'\rangle_\cH = f^\top f' = f'^\top f$ for any $f,f'\in \cH$, 
so that 
$f(z) = \langle \phi(z),f\rangle_\cH = f^\top \phi(z)$
for any $f\in \cH$ and any $z\in \cZ$. 
Moreover, letting $\cI_h = \{\tau_{h,1},\dots,\tau_{h,N_h}\}\subset [K]$, 
we denote the operators $\Phi_h\colon \cH\to \RR^K$ 
and $\Lambda_h\colon \cH \to \cH$ as 
\#\label{eq:rkhs_def_Phi}
\Phi_h = \begin{pmatrix}
  \phi(z_h^{\tau_{h,1}})^\top \\
  \vdots \\
  \phi(z_h^{\tau_{h,N_h}})^\top
\end{pmatrix},\quad 
\Lambda_h = \lambda\cdot I_\cH + \sum_{\tau \in \cI_h} \phi(z_h^\tau)\phi(z_h^\tau)^\top  
= \lambda\cdot I_\cH + \Phi_h^\top \Phi_h,
\#
where $I_\cH$ is the identity mapping in $\cH$
and all the formal matrix multiplications follow the same rules 
as those for real-valued matrix. 
In this way, these operators are well-defined. 
Also, $\Lambda_h$ 
is a self-adjoint operator eigenvalues 
no smaller than $\lambda$, in the sense that 
$\langle f, \Lambda g\rangle = \langle \Lambda f,g\rangle$
for any $f,g\in \cH$. 
Therefore, there exists a positive definite operator $\Lambda^{1/2}_h$ 
whose eigenvalues are no smaller than $\lambda^{1/2}$ 
and $\Lambda_h = \Lambda_h^{1/2} \Lambda_h^{1/2}$. 
We denote the inverse of $\Lambda_h^{1/2}$ as $\Lambda_h^{-1/2}$, so that 
$\Lambda_h^{-1} = \Lambda_h^{-1/2}\Lambda_h^{-1/2}$
and $\|\Lambda_h^{-1/2}\|_\cH\leq \lambda^{-1/2}$. 
For any $z\in \cZ$, we denote 
$
\Lambda_h(z) = \Lambda_h + \phi(z)\phi(z)^\top.
$
Then we have 
\$
\Lambda_h(z) = \Lambda_h^{1/2}\big(  I_\cH + \Lambda_h^{-1/2} \phi(z)\phi(z)^\top \Lambda_h^{-1/2}\big)\Lambda_h^{1/2},
\$
which indicates 
\$
\log\det\big(\Lambda_h(z)\big) &= \log\det(\Lambda_h) + \log\det\big(  I_\cH + \Lambda_h^{-1/2} \phi(z)\phi(z)^\top \Lambda_h^{-1/2}\big) \\
&= \log\det(\Lambda_h) + \log \big(1+  \phi(z)^\top \Lambda_h^{-1 }\phi(z)\big).
\$
Therefore, since $\phi(z)^\top \Lambda_h^{-1}\phi(z)\leq 1$ for $\lambda\geq 1$, we have 
\#\label{eq:rewrite_penalty}
\phi(z)^\top \Lambda_h^{-1 }\phi(z) &\leq 2 \log \big(1+  \phi(z)^\top \Lambda_h^{-1 }\phi(z)\big)\notag \\
& = 2 \cdot \big[\log\det\big(\Lambda_h(z)\big) - \log\det(\Lambda_h) \big]\notag \\
&= 2 \cdot \big[\log\det\big( I+K_h(z)/\lambda \big) - \log\det(I+K_h/\lambda )\big],
\#
where the last equality follows from the fact that 
\$
\det(\Lambda_h) = \det(\lambda I + K_h),\quad \det(\Lambda_h(z)) = \det\big(\lambda I + K_h(z)\big),
\$
and we define $K_h(z)$ as the Gram matrix for $\{z_h^\tau\}_{\tau\in[K]}\cup\{z\}$. On the other hand, under these notations $k_h(z) = \Phi_h \phi(z)$ for any $z\in \cZ$
and $k_h(\cdot)$ defined in Equation~\eqref{eq:def_Kkh}, hence 
the penalty function $\Gamma_h$ defined in Equation~\eqref{eq:gamma_rkhs}
can be written as  
\#\label{eq:rewrite_gamma}
\Gamma_h(x,a) &= \beta \cdot \lambda^{-1/2} \cdot 
\big( \phi(x,a)^\top \phi(x,a) - \phi(x,a)^\top \Phi_h^\top (K_h+\lambda I)^{-1} \Phi_h \phi(x,a) \big)^{1/2} \notag \\
&=  \beta \cdot \lambda^{-1/2} \cdot 
\big( \phi(x,a)^\top \phi(x,a) - \phi(x,a)^\top (\Phi_h^\top \Phi_h + \lambda \cdot I_\cH)^{-1}\Phi_h^\top \Phi_h \phi(x,a) \big)^{1/2} \notag\\
&= \beta \cdot \lambda^{-1/2} \cdot \Big( \phi(x,a)^\top
\big( I_\cH  -  ( \Phi_h^\top \Phi_h + \lambda \cdot I_\cH)^{-1}\Phi_h^\top \Phi_h  \big) \phi(x,a) \Big)^{1/2}\notag \\
&= \beta \cdot \lambda^{-1/2} \cdot \Big( \phi(x,a)^\top  ( \Phi_h^\top \Phi_h + \lambda \cdot I_\cH)^{-1}
\big(( \Phi_h^\top \Phi_h + \lambda \cdot I_\cH)  -   \Phi_h^\top \Phi_h  \big) \phi(x,a) \Big)^{1/2} \notag\\
&= \beta \cdot \big( \phi(x,a)^\top ( \Phi_h^\top \Phi_h + \lambda \cdot I_\cH)^{-1} \phi(x,a) \big)^{1/2}. 
\#
Combining Equations~\eqref{eq:rewrite_penalty} and~\eqref{eq:rewrite_gamma}, 
we have 
\$
\Gamma_h(x,a) \leq \sqrt{2} \beta \cdot \big[\log\det\big( I+K_h(z)/\lambda \big) - \log\det(I+K_h/\lambda) \big]^{1/2},
\$
which completes the proof of Theorem~\ref{thm:rkhs}.
\end{proof}

\begin{lemma}\label{lem:rkhs_uc}
Suppose Assumption~\ref{assump:rkhs_bound} holds. We set 
$\beta=B$ in Algorithm~\ref{alg:pess_rkhs} where $\lambda\geq 1+1/K$ 
and $B>0$ satisfies
\$
2\lambda R_Q^2   + 
2   G(K/H,1+1/K) + 2 K/H  \cdot \log (1+1/K)+ 4   \log\big(H /\xi\big) 
 \leq (B/H)^2.
\$
Then $\{\Gamma_h\}_{h\in[H]}$ constructed in Equation~\eqref{eq:gamma_rkhs}
is a $\xi$-uncertainty quantifier, where 
$\{\hat\BB_h\}_{h\in[H]}$ and $\{\hat{V}_{h}\}_{h\in[H]}$ are obtained from Algorithm~\ref{alg:pess_rkhs}. 

\end{lemma}

\begin{proof}[Proof of Lemma~\ref{lem:rkhs_uc}]
% It suffices to show that $\{\Gamma_h\}_{h\in[H]}$ in Algorithm~\ref{alg:pess_rkhs}
% are $\xi$-uncertainty quantifies under the regularity assumptions for $\cH$. 

% We begin with an equivalent set of notations for the algorithm. 
We first establish the closed form of $\hat{f}_h$ in Equation~\eqref{eq:hatf_form} 
for any $h\in[H]$.
Since $\cH$ is an RKHS, 
there exists a feature map $\phi\colon \cZ\to \cH$ such that 
$f(z) = \langle f,\phi(z)\rangle_{\cH}$ for all $f\in \cH$ and all $z\in \cZ$
and $K(z,z') = \langle \phi(z), \phi(z')\rangle_\cH$ for all $z,z'\in \cZ$. 
Therefore, for each step $h\in [H]$, 
the solution $\hat{f}_h$ of the kernel ridge regression could be written as 
\$
\hat{f}_h  = \argmin_{f\in \cH}~ \sum_{\tau \in \cI_h} \Big( r_h^\tau + \hat{V}_{h+1}(x_{h+1}^\tau) - \big\langle  \phi(x_h^\tau,a_h^\tau), f \big\rangle_\cH \Big)^2 + \lambda \cdot \|f\|_\cH^2.
\$
By the property of Hilbert spaces, $\cH$ admits the orthogonal decomposition $\cH = \cM\oplus \cM^\bot$, where $\cM$ is the span of $\big\{\phi(x_h^\tau,a_h^\tau)\colon \tau \in [K]\big\}$. 
In light of this decomposition, we claim that $\hat{f}_h\in \cM$.
To see this, if $\hat{f}_h \notin \cM$, 
we could write $\hat{f}_h = f_1 + f_2$, where $f_1\in \cM$ and $f_2\in \cM^\bot$ with $f_2\neq 0$, 
so that $\big\langle  \phi(x_h^\tau,a_h^\tau), \hat{f}_h \big\rangle_\cH = \big\langle  \phi(x_h^\tau,a_h^\tau), f_1 \big\rangle_\cH $ for all $\tau\in[K]$ 
but $\|\hat{f}_h\|_\cH^2 = \|f_1\|_\cH^2 + \|f_2\|_\cH^2 > \|f_1\|_\cH^2$, 
a contradiction to the definition of $\hat{f}_h$. 
Therefore, we have the equivalent representation 
\$
&\hat{f}_h = \sum_{\tau \in \cI_h}  \phi(x_h^\tau,a_h^\tau) [\hat\theta_h]_\tau,\\
  \text{where}~~&
\hat\theta_h = \argmin_{\theta\in \RR^{N_h}} ~ \sum_{\tau \in \cI_h}
\Big( r_h^\tau + \hat{V}_{h+1}(x_{h+1}^\tau) - K_h\theta \Big)^2 
+ \lambda \cdot \theta^\top K_h \theta,
\$
in which all multiplications are for real-valued vectors and matrices. 
With the above reduced ridge regression problem, 
we obtain the closed form 
\#\label{eq:theta_h_rkhs}
\hat\theta_h = (K_h + \lambda \cdot I)^{-1} y_h
\#
for the Gram matrix $K_h$ and response vector $y_h$ defined in Equation~\eqref{eq:def_Kkh}.
Combining Equation~\eqref{eq:theta_h_rkhs} with 
the feature representation $\hat{f}_h(z) = \langle \hat{f},\phi(z)\rangle_{\cH} = k_h(z)^\top \hat\theta_h$ where $k_h(z)$ is defined in Equation~\eqref{eq:def_Kkh}, 
we obtain the original closed form in Equation~\eqref{eq:hatf_form}.

Also,  %the Gram matrix satisfies 
for all $h\in[H]$, it holds that
$K_h = \Phi_h\Phi_h^\top$  and 
\$
\hat{f}_h = \Phi_h^\top \hat\theta_h = \Phi_h^\top \big( \Phi_h\Phi_h^\top + \lambda \cdot I \big)^{-1} y_h.
\$
In particular, it holds that 
\$
(\Phi_h^\top \Phi_h + \lambda \cdot I_\cH) \Phi_h^\top  =  \Phi_h^\top (\Phi_h\Phi_h^\top + \lambda \cdot I).
\$
Since both the matrix $\Phi_h\Phi_h^\top + \lambda \cdot I$  
and the operator $\Phi_h^\top \Phi_h + \lambda \cdot I_\cH$ are strictly positive definite, 
we have 
\#\label{eq:rkhs_equiv}
  \Phi_h^\top (\Phi_h\Phi_h^\top + \lambda \cdot I)^{-1} =  (\Phi_h^\top \Phi_h + \lambda \cdot I_\cH)^{-1}\Phi_h^\top,
\#
hence the fitted value function admits the form 
\#\label{eq:hatf_rkhs_form_matrix}
\hat{f}_h(z) = \phi(z)^\top \hat{f}_h = \phi(z)^\top \Lambda_h^{-1} \Phi_h^\top y_h. 
\#

We are now ready to bound the uncertainty in the empirical Bellman update 
$|\hat\BB_h \hat{V}_{h+1}(x,a) - \BB_h \hat{V}_{h+1}(x,a)|$ 
for all $(x,a)\in \cS\times \cA$ and all $h\in[H]$. 
By Assumption~\ref{assump:rkhs_bound}, there exists some $f_h\in \cH$ 
such that $\BB_h \hat{V}_{h+1} = f_h$ and 
$(\BB_h\hat{V}_{h+1})(x,a) = f_h(x,a) = \langle \phi(x,a), f_h\rangle_\cH$ 
under the feature representation. 
Here due to the boundedness in Assumption~\ref{assump:rkhs_bound},
we also have $\|f_h\|_\cH \leq R_Q H$. 
For any $h\in[H]$, recalling Equation~\eqref{eq:hatf_rkhs_form_matrix}
and utilizing the matrix representation, 
it follows from the triangular inequality that 
\$
& \big| (\hat\BB_h \hat{V}_{h+1})(x,a) - (\BB_h \hat{V}_{h+1})(x,a) \big| \\
&= \big| \phi(x,a)^\top ( \Phi_h^\top \Phi_h + \lambda \cdot I_\cH)^{-1} \Phi_h^\top y_h - \phi(x,a)^\top f_h \big|\\
&=\big| \underbrace{ \phi(x,a)^\top ( \Phi_h^\top \Phi_h + \lambda \cdot I_\cH)^{-1} \Phi_h^\top  \Phi_h f_h -  \phi(x,a)^\top f_h }_{\textrm{(i)}} \big| \\
&\qquad  +\big| \underbrace{\phi(x,a)^\top ( \Phi_h^\top \Phi_h + \lambda \cdot I_\cH)^{-1} \Phi_h^\top \big(y_h - \Phi_h f_h\big) }_{\textrm{(ii)}}\big|.
\$
In the sequel, we bound terms (i) and (ii) separately. 
By the Cauchy-Schwarz inequality, 
\$
\big|\text{(i)}\big| &= \big| \phi(x,a)^\top ( \Phi_h^\top \Phi_h + \lambda \cdot I_\cH)^{-1} \Phi_h^\top  \Phi_h f_h -  \phi(x,a)^\top f_h \big| \\
&= \big| \phi(x,a)^\top ( \Phi_h^\top \Phi_h + \lambda \cdot I_\cH)^{-1} \big[  \Phi_h^\top  \Phi_h  -  ( \Phi_h^\top \Phi_h + \lambda \cdot I_\cH) \big] f_h \big| \\
% &= \lambda \cdot \big| \phi(x,a)^\top ( \Phi_h^\top \Phi_h + \lambda \cdot I_\cH)^{-1} f_h \big| \\
&= \lambda \cdot \big| \phi(x,a)^\top \Lambda_h^{-1}  f_h \big| \\
&\leq \lambda \cdot \big\| \Lambda_h^{-1} \phi(x,a) \big\|_\cH \cdot \|f_h\|_{\cH},
\$
where $\Lambda_h = \Phi_h^\top \Phi_h + \lambda \cdot I_\cH$. 
Therefore, it holds that
\#\label{eq:rkhs_bd_i}
\big|\text{(i)}\big| \leq 
\lambda \cdot  \|\Lambda_h^{-1/2} \Lambda_h^{-1/2} \phi(x,a)\|_\cH \cdot \|f_h\|_{\cH} 
&\leq \lambda^{1/2} \cdot  \|\Lambda_h^{-1/2} \phi(x,a)\|_\cH \cdot \|f_h\|_{\cH} \notag \\
& \leq R_Q H\cdot \lambda^{1/2} \cdot \|\phi(x,a)\|_{\Lambda_h^{-1}},
\#
where we write $\|\phi(x,a)\|_{\Lambda_h^{-1}}^2 := \big\langle \phi(x,a), \Lambda_h^{-1} \phi(x,a)\big\rangle = \|\Lambda_h^{-1/2} \phi(x,a)\|_\cH^2$. 
On the other hand, recalling the definition of $\Phi_h$ in Equation~\eqref{eq:rkhs_def_Phi}, we have 
\#
\big|\text{(ii)}\big| &= \big| \phi(x,a)^\top \Lambda_h^{-1} \Phi_h^\top \big(y_h - \Phi_h f_h\big) \big| \notag \\
&= \Big| \phi(x,a)^\top \Lambda_h^{-1} \sum_{\tau \in \cI_h} \phi(x_h^\tau,a_h^\tau)\big[ r_h^\tau + \hat{V}_{h+1}(x_{h+1}^\tau) - \phi(z_h^\tau)^\top f_h   \big] \Big| \notag\\
&= \Big| \phi(x,a)^\top \Lambda_h^{-1} \sum_{\tau\in \cI_h} \phi(x_h^\tau,a_h^\tau)\big[ r_h^\tau + \hat{V}_{h+1}(x_{h+1}^\tau) - (\BB_h \hat{V}_{h+1})(z_h^\tau)  \big]\Big| \notag\\
&\leq \big\|\phi(x,a)\big\|_{\Lambda_h^{-1}} \cdot \bigg\|\sum_{\tau\in \cI_h} \phi(x_h^\tau,a_h^\tau)\big[ r_h^\tau + \hat{V}_{h+1}(x_{h+1}^\tau) - (\BB_h \hat{V}_{h+1})(z_h^\tau)  \big]\bigg\|_{\Lambda_h^{-1}}, 
\label{eq:rkhs_mtg_diff}
\#
where the last inequality follows from the Cauchy-Schwarz inequality.  

In the sequel, we aim to bound the RHS of Equation~\eqref{eq:rkhs_mtg_diff}
by martingale concentration inequalities. 
To this end, we note that for $h\in[H-1]$, $\hat{V}_{h+1}$ is constructed 
with  $\cD_{h+1}\cup \cdots\cup \cD_{H}$. 
Recall the inverse order of the sample splitting, and  
$\cI_h = \{K/H\cdot(h-1)+1,\dots, K/H \cdot h\}$. 
We then define the filtration 
\$
\cF_{h,\tau} = \sigma\big( \{(x_h^j, a_h^j)\}_{j=1}^{(\tau+1)\vee K} \cup \{ (r_h^j, x_{h+1}^j)  \}_{j=1}^\tau   \big) 
\$
for $\tau\in \cI_h$, 
where $\sigma(\cdot)$ denotes the $\sigma$-algebra 
generated by the set of random variables and $(\tau+1)\vee k =\min\{\tau+1,K\}$. 
The inverse order in the data splitting implies that 
for any $\tau \in \cI_h$,  
\$
\sigma \big(\cup_{h'=h+1}^H \cD_{h'} \big) \subseteq \sigma\big(\{(x_h^j, a_h^j, r_h^j)\}_{j=1}^{\tau-1} \big),
\$
hence $\hat{V}_{h+1}\in \cF_{h,\tau}$. 
Therefore, the stochastic process 
\$
\epsilon_h^\tau = r_h^\tau + \hat{V}_{h+1}(x_{h+1}^\tau) - (\BB_h \hat{V}_{h+1})(z_h^\tau)
\$
is adapted to the filtration $\{\cF_{h,\tau}\}_{\tau\in \cI_h}$. 
Meanwhile,  the compliance assumption of 
dataset $\cD$ imply 
\$
\EE\big[ r_h^\tau + \hat{V}_{h+1}(x_{h+1}^\tau) \biggiven \cF_{h,\tau-1} \big] 
= \EE\big[ r_h(s_h,a_h) + \hat{V}_{h+1}(s_{h+1})\biggiven s_h = x_h^\tau, a_h=a_h^\tau \big]
= (\BB_h \hat{V}_{h+1})(z_h^\tau) ,
\$
where the second expectation is induced by  
the underlying MDP, and $\hat{V}_{h+1}$ is viewed as a fixed function. 
We thus have $\EE[\epsilon_h^\tau \given \cF_{h,\tau-1}]=0$. 
Applying Lemma~\ref{lem:rkhs_concentration} to 
$\epsilon_\tau = \epsilon_h^\tau $ and $\sigma^2 = H^2$ 
as $\epsilon_h^\tau  \in [-H,H]$,  
for any $\eta>0$ and any $\xi>0$, it holds with probability at least 
$1-\xi/H$ 
that 
\# \label{eq:rkhs_single_3}
&E_h^\top \big[(K_h + \eta\cdot I)^{-1} + I\big]^{-1} E_h \notag  \\
&\leq H^2 \cdot \log\det \big[ (1+\eta)\cdot I + K_h   \big] 
+ 2H^2 \cdot \log\big(H /\xi\big),
\#
where $E_h = (\epsilon_h^\tau)_{\tau\in\cI_h}^\top$ for $\epsilon_k = \epsilon_h^\tau(V)$. 

We now translate the bound in Equation~\eqref{eq:rkhs_single_3} 
to the desired form. 
We note that 
\#
\bigg\|\sum_{\tau \in \cI_h} \phi(x_h^\tau,a_h^\tau)\epsilon_h^\tau \bigg\|_{\Lambda_h^{-1}}^2 
&%= E_K^\top \Phi_h \Lambda_h^{-1} \Phi_h^\top E_k 
= E_h^\top \Phi_h \big( \Phi_h^\top \Phi_h + \lambda \cdot I_\cH\big)^{-1} \Phi_h^\top E_h \notag \\
&= E_h^\top \Phi_h \Phi_h^\top \big(  \Phi_h \Phi_h^\top+ \lambda \cdot I \big)^{-1}  E_h \notag \\
&= E_h^\top K_h \big(  K_h+ \lambda \cdot I \big)^{-1}  E_h \notag \\
&=  E_h^\top E_h - \lambda \cdot E_h^\top  \big(  K_h+ \lambda \cdot I \big)^{-1}  E_h \notag \\
&= E_h^\top E_h -  E_h^\top  \big(  K_h/\lambda +  I \big)^{-1}  E_h, \label{eq:rkhs_single_1}
\#
where the first equality follows from the definition of $\Lambda_h$, 
the second equality follows from Equation~\eqref{eq:rkhs_equiv},
and the third equality follows from the fact that $K_h = \Phi_h\Phi_h^\top$. 
Since 
the operator norm of $(K_h/\lambda +I)^{-1}\in \RR^{K/H\times K/H}$ is 
increasing in $\lambda >0$, 
the right hand side of Equation~\eqref{eq:rkhs_single_1} 
is non-increasing in $\lambda >0$. Thus,  
for any $\underline{\lambda}>1$ such that  $\lambda \geq \underline{\lambda}$,
it holds that  
\$
\bigg\|\sum_{\tau \in \cI_h} \phi(x_h^\tau,a_h^\tau)\epsilon_h^\tau \bigg\|_{\Lambda_h^{-1}}^2 
\leq E_h^\top K_h \big(  K_h+ \underline{\lambda} \cdot I \big)^{-1}  E_h.
\$
For any $\eta>0$, noting that 
$
\big((K_h + \eta\cdot I)^{-1}+I\big)(K_h + \eta\cdot I) = K_h + (1+\eta)\cdot I,
$
we have 
\#\label{eq:rkhs_equiv2}
\big((K_h + \eta\cdot I)^{-1}+I\big)^{-1} = (K_h + \eta\cdot I)\big(K_h + (1+\eta)\cdot I\big)^{-1} .
\#
% Therefore, 
% % writing $E_K =  (\epsilon_h^\tau)_{\tau\in \cI_h}^\top$,
% it holds that
% % probability at least $1-\delta/|\cN_h^\rkhs(\varepsilon\,;R,B,\lambda)|$  that
% \#
% \bigg\|\sum_{\tau \in \cI_h} \phi(x_h^\tau,a_h^\tau)\epsilon_h^\tau \bigg\|_{\Lambda_h^{-1}}^2 
% &%= E_K^\top \Phi_h \Lambda_h^{-1} \Phi_h^\top E_k 
% = E_h^\top \Phi_h \big( \Phi_h^\top \Phi_h + \lambda \cdot I_\cH\big)^{-1} \Phi_h^\top E_h \notag \\
% &= E_h^\top \Phi_h \Phi_h^\top \big(  \Phi_h \Phi_h^\top+ \lambda \cdot I \big)^{-1}  E_h \notag \\
% &= E_h^\top K_h \big(  K_h+ \lambda \cdot I \big)^{-1}  E_h, \label{eq:rkhs_single_1}
% \#
% where the first equality follows from the definition of $\Lambda_h$, 
% the second equality follows from Equation~\eqref{eq:rkhs_equiv},
% and the last equality follows from the fact that $K_h = \Phi_h\Phi_h^\top$. 
Meanwhile, taking $\eta = \underline{\lambda} -1 >0$, we have 
\#
E_h^\top K_h \big(  K_h+ \underline{\lambda} \cdot I \big)^{-1}  E_h
&\leq E_h^\top (K_h +\eta \cdot I)\big(  K_h+ \underline{\lambda} \cdot I \big)^{-1}  E_h \notag \\
&= E_h^\top \big[(K_h + \eta\cdot I)^{-1}+I\big]^{-1}  E_h,\label{eq:rkhs_single_2}
\#
where the second line follows from Equation~\eqref{eq:rkhs_equiv2}. 
For any fixed $\xi>0$, combining Equations~\eqref{eq:rkhs_single_3},~\eqref{eq:rkhs_single_1} and~\eqref{eq:rkhs_single_2}, 
we know that
\#\label{eq:rkhs_union_2}
& \bigg\|\sum_{\tau \in \cI_h} \phi(x_h^\tau,a_h^\tau)\epsilon_h^\tau \bigg\|_{\Lambda_h^{-1}}^2 
\leq H^2 \cdot \log\det \big[ \underline{\lambda} \cdot I + K_h   \big] 
+ 2H^2 \cdot \log\big(H /\xi\big)
\#
holds with probability at least $1-\xi/H$. 
Also, note that $\underline{\lambda}\cdot I + K_h =\underline{\lambda}\cdot (I+K_h/\underline{\lambda})$, hence 
$\log\det \big[ \underline{\lambda}\cdot I + K_h   \big] = N_h\log\underline{\lambda} + \log\det \big[ I + K_h  /\underline{\lambda} \big]$, where $N_h = |\cI_h| = K/H$.  
% Combining Equations~\eqref{eq:rkhs_union_1} and~\eqref{eq:rkhs_union_2}
% and taking a union bound over $h\in[H]$, 
As a result, 
for any $\varepsilon>0$ and any $1<\underline{\lambda}\leq \lambda$, 
\#\label{eq:rkhs_bd_ii}
& \bigg\|\sum_{\tau \in \cI_h} \phi(x_h^\tau,a_h^\tau)\epsilon_h^\tau \bigg\|_{\Lambda_h^{-1}}^2 
\leq   H^2 \cdot \log\det \big[ I + K_h / \underline{\lambda} \big] +  H^2 N_h\log \underline{\lambda}
+ 2H^2 \cdot \log\big(H/\xi\big) 
\#
holds simultaneously for all $h\in[H]$ with probability at least $1-\xi$. 
% Combining Lemma~\ref{lem:rkhs_norm} with Equations~\eqref{eq:rkhs_bd_i},~\eqref{eq:rkhs_mtg_diff},~\eqref{eq:rkhs_bd_ii}
% and taking $\lambda = 1+1/K$ and $R = 2H\sqrt{K\cdot G(K,\lambda)}$, 
% if in the construction of $\{\Gamma_h\}_{h\in[H]}$ in Equation~\eqref{eq:gamma_rkhs}
% we take $\beta = B$ where $B>0$ satisfies
% \$
% 8\cdot G(K,1+1/K) + 16 \cdot \log\big(H\cdot|\cN_h^\rkhs(H/K\,;R,B,1+1/K)|/\xi\big) + 16 
% + 8 R_Q^2 \leq (B/H)^2, 
% \$
% then 
Combining Equations~\eqref{eq:rkhs_bd_i},~\eqref{eq:rkhs_mtg_diff} 
and~\eqref{eq:rkhs_bd_ii}, 
with probability at least $1-\xi$, it holds simultaneously for all $h\in[H]$ that 
\$
&\big| (\hat\BB_h \hat{V}_{h+1})(x,a) - (\BB_h \hat{V}_{h+1})(x,a) \big| \\
& \leq \big\|\phi(x,a)\big\|_{\Lambda_h^{-1}}  \Bigg[ R_QH\cdot \sqrt{\lambda} +  \bigg\|\sum_{\tau\in \cI_h} \phi(x_h^\tau,a_h^\tau)\epsilon_h^\tau \bigg\|_{\Lambda_h^{-1}} \Bigg] \\
&\leq \big\|\phi(x,a)\big\|_{\Lambda_h^{-1}}  \Bigg[ 2\lambda  R_Q^2 H^2 +  2  \bigg\|\sum_{\tau\in \cI_h} \phi(x_h^\tau,a_h^\tau)\epsilon_h^\tau(V)\bigg\|_{\Lambda_h^{-1}}^2 \Bigg]^{1/2} \\
&\leq  \big\|\phi(x,a)\big\|_{\Lambda_h^{-1}}
\Big[ 2\lambda R_Q^2 H^2 + 
2H^2 \cdot G(K/H,\underline{\lambda}) + 2 KH  \cdot \log \underline{\lambda}+ 4H^2 \cdot \log\big(H /\xi\big)  \Big]^{1/2} \\ 
&\leq B \cdot \big\|\phi(x,a)\big\|_{\Lambda_h^{-1}} = \Gamma_h(x,a),
\qquad \forall~(x,a)\in \cS\times \cA,
\$
where the second inequality follows from the Cauchy-Schwarz inequality, 
the third inequality follows from the definition of $G(K,\lambda)$ in Equation~\eqref{eq:def_Gklambda}, 
and the last equality follows from the definition 
of $\{\Gamma_h\}_{h\in[H]}$ in Equation~\eqref{eq:gamma_rkhs}. 
Therefore, $\{\Gamma_h\}_{h\in[H]}$ is a $\xi$-uncertainty quantifier for $\PP_{\cD}$
with the construction of $\{\hat\BB_h\}_{h\in[H]}$ and $\{\hat{V}_{h}\}_{h\in[H]}$
in Algorithm~\ref{alg:pess_rkhs}. 
Finally, choosing $\underline{\lambda} = 1+1/K$, 
we conclude the proof of Lemma~\ref{lem:rkhs_uc}.
\end{proof}

\subsection{Proof of Proposition~\ref{prop:rkhs_decay}} 
\label{app:proof_prop_rkhs_decay}
\begin{proof}[Proof of Proposition~\ref{prop:rkhs_decay}]
The condition of $B$ in Theorem~\ref{thm:rkhs} translates to 
\$
B \geq H\cdot \big[2 \lambda R_Q^2 + 
2 \cdot G(K/H,1+1/K) + 2 K/H \cdot \log(1+1/K) + 4 \cdot \log (H /\xi )\big]^{1/2}.
\$ 
Note that $2K/H \cdot \log(1+1/K)\leq 2/H\leq 1 \leq \log(H/\xi)$. 
Then it suffies to have  
\$
B/2\geq 2\lambda R_Q^2\cdot H,\quad \text{and}\quad 
B/2 \geq H\cdot \big[  
2 \cdot G(K/H,1+1/K) + 5 \cdot \log (H /\xi )\big]^{1/2}.
\$
We now proceed to upper bound the right-handed side above. 

\paragraph{(i): $\gamma$-finite spectrum.}
In this case, since $1+1/K\in[1,2]$, 
by Lemma~\ref{lem:RKHS_dim}, there exists some absolute constant $C $ 
that only depends on $d,\gamma$ such that 
\$
G(K/H,1+1/K) \leq C  \cdot \gamma \cdot \log (K/H).
\$
Hence we could set  
$B = c \cdot  H \cdot \sqrt{ \gamma \cdot \log(  KH/\xi)}$ for some sufficiently large 
constant $c >0$.

\paragraph{(ii): $\gamma$-exponential decay.}
By Lemma~\ref{lem:RKHS_dim}, there exists some absolute constant $C$ 
that only depends on $d,\gamma$ such that 
\$
G(K/H,1+1/K) \leq C  \cdot (\log (K/H))^{1+1/\gamma}.
\$
We can thus choose 
$B = c \cdot H\cdot (\log (KH/\xi))^{1/2+1/(2\gamma)}$
for some sufficiently large absolute constant $c>0$ depending on $d,\gamma,C_1,C_2$ and $C_\psi$.

\paragraph{(iii): $\gamma$-polynomial decay.}
By Lemma~\ref{lem:RKHS_dim}, there exists some absolute constant $C $ 
that only depends on $d,\gamma$ such that 
\$
G(K/H,1+1/K) \leq C \cdot  (K/H)^{\frac{d+1}{\gamma +d }}\cdot \log (K/H).
\$
Thus, it suffices to choose 
$B = c  \cdot K^{\frac{d+1}{2(\gamma +d) }}  H^{1-{\frac{d+1}{2(\gamma +d) }}} \cdot \sqrt{\log(KH/\xi)}$, 
where $c>0$ is a sufficiently large absolute constant depending on $d,\gamma$.  
\end{proof}

\subsection{Proof of Corollary~\ref{cor:rkhs_iid}}
\label{app:subsec_cor_rkhs_iid}

\begin{proof}[Proof of Corollary~\ref{cor:rkhs_iid}]
Firstly, by Jensen's inequality, 
% the information quantity in Theorem~\ref{thm:rkhs} 
% can be bounded as 
\$
&\sum_{h=1}^H \EE_{\pi^*}\Big[  \big\{\log\det\big( I+K_h(s_h,a_h)/\lambda \big) - \log\det(I+K_h/\lambda) \big\}^{1/2}\Biggiven s_1=x\Big] \\
&\leq \sum_{h=1}^H \Big\{ \EE_{\pi^*}  \big[   \log\det\big( I+K_h(s_h,a_h)/\lambda \big) - \log\det(I+K_h/\lambda) \biggiven s_1=x\big] \Big\}^{1/2} \\ 
&= \sum_{h=1}^H \Big\{ \EE_{\pi^*}  \big[ \phi(s_h,a_h)^\top (\Phi_h^\top \Phi_h + \lambda \cI_{\cH})^{-1} \phi(s_h,a_h)  \biggiven s_1=x\big] \Big\}^{1/2} = d_{\textrm{eff}}^{\text{sample}},
\$
where $\Phi_h^\top \Phi_h = \sum_{\tau\in \cI_h}\phi(z_h^\tau)\phi(z_h^\tau)^\top$, 
and the last line uses the feature map representation 
in~\eqref{eq:rewrite_gamma}. 
Therefore, fixing any $\xi\in(0,1)$, 
we set $B>0$ as in Proposition~\ref{prop:rkhs_decay} 
with a sufficiently large constant $C>0$ and 
some $\lambda\geq 1+1/K$ to be specified later. 
Then Theorem~\ref{thm:rkhs} and Proposition~\ref{prop:rkhs_decay} 
indicate that with probability at least $1-\xi/2 $, 
it holds simultaneously for all $x\in \cX$ that 
\$
&\textrm{SubOpt}\big( \pess (\cD); x \big) 
\leq 
2\sqrt{2}B \cdot d_{\textrm{eff}}^{\text{sample}}.
\$

In the sequel, we relate $d_{\textrm{eff}}^{\text{sample}}$ 
to $d_{\text{eff}}^{\text{pop}}$ by properly 
setting $\lambda>0$ under the eigenvalue decay 
conditions in Assumption~\ref{assump:rkhs_decay}.
Recalling $\Lambda_h = \Phi_h^\top \Phi_h + \lambda \cI_\cH$, 
% $\phi(s_h,a_h)^\top (\Phi_h^\top \Phi_h + \lambda \cI_{\cH})^{-1} \phi(s_h,a_h)  
% = \|\phi(s_h,a_h)\|_{\Lambda_h^{-1}}$ where 
the operator norm of $\Lambda_h^{-1}$ 
is lower bounded as $\|\Lambda_{h}^{-1}\|_{\oper}\geq 1/\lambda$. 
Furthermore, as stated in Section~\ref{subsec:rkhs_intro}, 
the feature mapping $\phi\colon \cZ\to \cH$ 
can be expanded with respect to 
the orthogonal basis $\{\sqrt{\sigma_j}\cdot \psi_j \}_{j\geq 0}$ as
\#\label{eq:phi_z_decomp}
\phi(z) = \sum_{j=1}^\infty \sigma_j \cdot \psi_j(z)\cdot \psi_j
= \sum_{j=1}^\infty \sqrt{\sigma_j} \cdot \psi_j(z)\cdot (\sqrt{\sigma_j}\psi_j) 
\#
  for any $z\in \cZ$. 
For any $m \in \NN$, we define 
\#\label{eq:def_rm}
R_m = \Bigg(\sum_{j=m+1}^\infty \sigma_j \cdot \|\psi_j\|_\infty^2\Bigg)^{1/2},
\#
where $\|\psi_j\|_\infty = \sup_{z\in \cZ}|\psi_j(z)|$. 
% The following lemma shows that once we could balance $R_m$
The following lemma establishes the concentration of $\Lambda_h$ to 
certain population quantities, whose proof is in Appendix~\ref{app:subsec_rkhs_concentration_Lambda_h}. 

\begin{lemma}\label{lem:rkhs_concentration_Lambda_h}
Fix any $\delta\in(0,1)$, any $\varepsilon>0$, and any $m\in \NN$. 
Let $n=K/H$. 
In Algorithm~\ref{alg:pess_rkhs}, we set 
\#\label{eq:rkhs_lambda_general}
\lambda \geq 16\log(2/\delta) + 16m \log(1+2/\varepsilon).
\#
Then with probability at least $1-\delta$, it holds that 
\#\label{eq:rkhs_con}
\frac{n_0\Sigma_h + \lambda \cI_\cH}{2}  - \big((3n+\lambda)\varepsilon + 5n  R_m \big)\cdot \cI_\cH \preceq \Lambda_h \preceq  \frac{3(n \Sigma_h + \lambda \cI_\cH)}{2}  + \big((3n+\lambda)\varepsilon + 5n  R_m \big)\cdot \cI_\cH.
\#
\end{lemma}

We now specify $m$ and $\varepsilon$ 
in Lemma~\ref{lem:rkhs_concentration_Lambda_h} 
to establish the error bounds 
for each eigenvalue decay condition in Assumption~\ref{assump:rkhs_decay} 
and compute the constant $B>0$ accordingly. 
Throughout, we set $\delta=\xi/2$ and show that 
\#\label{eq:bound_sample_rkhs}
d_{\text{eff}}^{\text{sample}} 
\leq  4 \sum_{h=1}^H  \tr \big( (K/H\cdot \Sigma_h + \lambda \cI_\cH)^{-1} \Sigma_h^* \big)^{1/2} = 4 d_{\text{eff}}^{\text{pop}}
\#
with probability at least $1-\xi/2$. 
Taking a union bound, we know that 
with probability at least $1-\xi$, 
it holds simultaneously for all $x\in \cX$ that 
\$
&\textrm{SubOpt}\big( \pess (\cD); x \big) 
\leq 
8\sqrt{2}B \cdot d_{\textrm{eff}}^{\text{pop}}.
\$

\paragraph{(i) $\gamma$-finite spectrum.}
In this case, we could simply take $m=\gamma$ and $R_m=0$. 
We also take 
\$ 
\lambda = 32 \gamma \log (2n /\delta),\quad \varepsilon = \frac{\lambda}{24n } = \frac{4\gamma \log(2n /\delta)}{3n }.
\$
We first verify this choice satisfies Equation~\eqref{eq:rkhs_lambda_general}. 
Note that $\lambda/2 \geq 16\log(2/\delta)$, and 
\$
16\gamma \log(1+2/\varepsilon) \leq 16\gamma \cdot \log\Big(1+\frac{3n}{2\gamma\log(2n /\delta)}\Big) \leq 16\gamma \log(2n /\delta) =  \lambda/2,
\$
hence Equation~\eqref{eq:rkhs_lambda_general} holds. 
Meanwhile, 
this choice ensures 
$
(3n +\lambda)\varepsilon \leq \lambda/4  
$
when $n \geq 256\gamma \cdot \log(2n /\delta)$, 
which further leads to 
\$
\Lambda_h \succeq \frac{n \Sigma_h + \lambda \cI_\cH}{2}  - (3n+\lambda)\varepsilon  \cdot \cI_\cH
\geq \frac{n \Sigma_h + \lambda \cI_\cH}{2} - \frac{\lambda}{4}\cI_{\cH}
\succeq \frac{n \Sigma_h + \lambda \cI_\cH}{4}
\$
with probability at least $1-\delta$. Consequently, on the same event, 
\$%#\label{eq:bound_sample_rkhs}
d_{\text{eff}}^{\text{sample}} 
& = \sum_{h=1}^H \Big\{ \EE_{\pi^*}  \big[ \phi(s_h,a_h)^\top \Lambda_h^{-1} \phi(s_h,a_h)  \biggiven s_1=x\big] \Big\}^{1/2} \notag \\ 
&\leq 4 \sum_{h=1}^H \Big\{ \EE_{\pi^*}  \big[ \phi(s_h,a_h)^\top (K/H\cdot \Sigma_h + \lambda \cI_\cH)^{-1} \phi(s_h,a_h)  \biggiven s_1=x\big] \Big\}^{1/2} \notag \\
&= 4 \sum_{h=1}^H  \tr \big( (K/H\cdot \Sigma_h + \lambda \cI_\cH)^{-1} \Sigma_h^* \big)^{1/2},
\$%#
which is exactly Equation~\eqref{eq:bound_sample_rkhs}. 
Meanwhile, this choice of $\lambda$ leads to 
\$
B \leq C\cdot H  \big\{\sqrt{\lambda} + \sqrt{\gamma\cdot \log (KH/\xi)}\big\} 
\leq C' \cdot H    \sqrt{\gamma\cdot \log(KH/\xi)}
\$
for some sufficiently large constant $C'>0$ that does not 
depend on $K$ or $H$. 

\paragraph{(ii) $\gamma$-exponential decay.}
We follow the computation in~\citet{yang2020function} to compute $R_m$ for 
$\gamma$-exponential decay, where we assume that $\|\psi_j\|_\infty \leq C_\psi \cdot \sigma^{-\tau}$ for all $j\geq 1$. Thus, 
\$
\sum_{j=m+1}^\infty \sigma_j \cdot \|\psi_j\|_\infty^2 
&\leq 
C_\psi^2 \sum_{j=m+1}^\infty \sigma^{1-2\tau} 
\leq C_\psi^2 \cdot C_1^{1-2\tau}\cdot \sum_{j=m+1}^\infty \exp\big(-C_2\cdot (1-2\tau)\cdot j^\gamma\big),
\$
where $\tau\in[0,1/2)$. 
For notational simplicity, we denote the constants 
$C_{1,\tau} = C_\psi\cdot C_1^{1/2-\tau}$
and $C_{2,\tau} = C_2\cdot(1-2\tau)$, both of which are positive. 
We thus have 
\$
R_m^2 \leq C_{1,\tau}^2  \sum_{j=m+1}^\infty \exp (-C_{2,\tau}\cdot j^\gamma )  
\leq C_{1,\tau}^2 \int_{m}^\infty \exp ( -C_{2,\tau} \cdot u^\gamma ) \ud u
\$
by the monotonicity of $\exp(-C_{2,\tau}\cdot u^\gamma)$ in $u\in[m,\infty)$. 
In the following, 
we bound $R_m$ for two cases $\gamma\geq 1$ and $\gamma\in(0,1)$, separately; 
this follows exactly the same calculations as in~\citet[Section E.2]{yang2020function}, 
while we include the details here for completeness. 
When $\gamma\geq 1$, for $u\geq m\geq 1$ it holds that $u^{\gamma-1}\geq 1$. 
Hence with a change of variable $v = u^\gamma$, one has 
\$
\int_{m}^\infty \exp ( -C_{2,\tau} \cdot u^\gamma ) \ud u 
&\leq \int_m^\infty \gamma \cdot u^{\gamma-1} \cdot  \exp ( -C_{2,\tau} \cdot u^\gamma ) \ud u \\
&= \int_{m^\gamma}^\infty e^{-C_{2,\tau} v}\ud v = C_{2,\tau}^{-1} \cdot \exp ( -C_{2,\tau} \cdot m^\gamma ).
\$
When $\gamma\in(0,1)$, with a change of variable $v=u^\gamma$ one has 
\#\label{eq:rkhs_exp_1}
&\int_{m}^\infty \exp ( -C_{2,\tau} \cdot u^\gamma ) \ud u 
= \frac{1}{\gamma} \int_{m^\gamma}^\infty v^{1/\gamma-1}\cdot e^{-C_{2,\tau}  v} \ud v \notag \\ 
&\qquad = \frac{m^{1-\gamma}}{\gamma\cdot C_{2,\tau}}  \exp(-C_{2,\tau}\cdot m^\gamma) + \frac{1-\gamma}{\gamma^2\cdot C_{2,\tau}} 
\int_{m^\gamma}^\infty v^{1/\gamma-2} e^{-C_{2,\tau} v} \ud v,
\#
where the last equality is integration by parts. Furthermore, 
since $1/v\leq 1/m^\gamma$ for all $v\geq m$, 
the second integral in Equation~\eqref{eq:rkhs_exp_1} 
can be bounded as 
\#\label{eq:rkhs_exp_2}
\int_{m^\gamma}^\infty v^{1/\gamma-2} \cdot e^{-C_{2,\tau} v} \ud v
&\leq \frac{1}{m^\gamma} \int_{m^\gamma}^\infty v^{1/\gamma-1} \cdot e^{-C_{2,\tau} v} \ud v 
= \frac{1}{m^\gamma} \int_{m }^\infty   e^{-C_{2,\tau} u^\gamma} \ud u,
\#
where the last equality uses a change of variable 
$u = v^\gamma$. Combining Equations~\eqref{eq:rkhs_exp_1}
and~\eqref{eq:rkhs_exp_2} leads to
\$
\int_{m}^\infty \exp ( -C_{2,\tau} \cdot u^\gamma ) \ud u 
\leq \frac{m^{1-\gamma}}{\gamma\cdot C_{2,\tau}}  \exp(-C_{2,\tau}\cdot m^\gamma) + \frac{1/\gamma-1}{C_{2,\tau}\cdot m^\gamma} \int_{m}^\infty \exp ( -C_{2,\tau} \cdot u^\gamma ) \ud u .
\$
Then for sufficiently large $m$ satisfying $m^\gamma \cdot C_{2,\tau} > 2(1/\gamma-1)$, 
solving the above inequality leads to 
\$
\int_{m}^\infty \exp ( -C_{2,\tau} \cdot u^\gamma ) \ud u 
\leq \frac{2m^{1-\gamma}}{\gamma\cdot C_{2,\tau}} \exp(-C_{2,\tau}\cdot m^\gamma).
\$
To summarize, when $\gamma\geq 1$, we have 
\$
R_m \leq C_{1,\tau} C_{2,\tau}^{-1/2} \cdot \exp(-C_{2,\tau}/2\cdot m^\gamma),
\$
while when $\gamma\in (0,1)$, we have 
\$
R_m \leq C_{1,\tau} C_{2,\tau}^{-1/2} \cdot \sqrt{2/\gamma} \cdot m^{(1-\gamma)/2} \cdot \exp(-C_{2,\tau}/2 \cdot m^\gamma).
\$

We now specify a proper set of 
$(m, \lambda, \varepsilon)$ 
such that Equation~\eqref{eq:rkhs_lambda_general} holds, 
and $\Lambda_h \succeq (n\Sigma_h + \lambda \cI_\cH)/4$ 
with high probability. 
We consider $\gamma\geq 1$ and $\gamma\in(0,1)$ separately. 

When $\gamma\geq 1$, we choose 
the smallest $m\in \NN$ such that  
$C_{1,\tau} C_{2,\tau}^{-1/2} \cdot \exp(-C_{2,\tau}/2\cdot m^\gamma) \leq \lambda/(40n) $, 
where we recall that $C_{1,\tau}$ and $C_{2,\tau}$ 
are absolute constants. Equivalently, we choose $m \in \NN$ such that 
\$
m \geq \bigg\{\frac{2}{C_{2,\tau}} \log\bigg( \text{\small $\frac{40n \cdot C_{1,\tau}}{\lambda \cdot C_{2,\tau}^{1/2}}$} \bigg)\bigg\}^{1/\gamma}.
\$
We now let $\lambda = \max\{C_3 (\log (n/\delta))^2,1\}$ 
for some sufficiently large absolute constant $C_3$, 
and $\varepsilon = \lambda/(36n)$. 
We first verify the condition~\eqref{eq:rkhs_lambda_general} holds 
with this choice. 
When $n$ is sufficiently large, we have 
$\lambda/2 \geq 16\log (2/\delta)$. 
Meanwhile, there exists absolute constants $C_4,C_5,C_6>0$ 
(i.e., only depending on $C_{1,\tau}$ and $C_{2,\tau}$) such that 
\$
16 m \log(1+2/\varepsilon) 
&\leq C_4 \cdot \bigg\{\frac{2}{C_{2,\tau}} \log\bigg( \text{\small $\frac{40n \cdot C_{1,\tau}}{\lambda \cdot C_{2,\tau}^{1/2}}$} \bigg)\bigg\}^{1/\gamma} \log\bigg( 1 + \text{$\frac{72n}{\lambda }$} \bigg) \\ 
&\leq C_4 \cdot \max\bigg\{1, \frac{2}{C_{2,\tau}} \log\bigg( \text{\small $\frac{40n \cdot C_{1,\tau}}{\lambda \cdot C_{2,\tau}^{1/2}}$} \bigg) \bigg\} \cdot \log\bigg( 1 + \text{$\frac{72n}{ \lambda }$} \bigg) \\
&\leq C_5 \cdot \max\Big\{ \log(C_6\cdot n/\lambda) , \big[ \log(C_6\cdot n/\lambda) \big]^2 \Big\},
\$
where the second inequality uses $1/\gamma\leq 1$. 
Therefore, we can choose some sufficiently large 
absolute constant $C_3>0$, which only depends on 
$C_\psi$, $C_1$, $C_2$ and $\tau$, such that 
$16m \log(1+2/\varepsilon) \leq \lambda/8$ 
for sufficiently large $n$. 
Thus, we know that Equation~\eqref{eq:rkhs_lambda_general} 
holds. 
When $n$ is sufficiently large such that 
$n\geq 2\lambda/3$, we have $(3n+\lambda)\varepsilon \leq \lambda/8$. 
On the other hand, 
$
5n  R_m \leq \lambda/8 
$
since $R_m \leq \lambda/(40n)$. 
Together with Equation~\eqref{eq:rkhs_con}, 
such choice leads to 
\$
\Lambda_h \succeq \frac{n \Sigma_h + \lambda \cI_\cH}{2}  - (3n+\lambda)\varepsilon  \cdot \cI_\cH - 5nR_m \cdot \cI_\cH
\geq \frac{n \Sigma_h + \lambda \cI_\cH}{2} - \frac{\lambda}{4}\cI_{\cH}
\succeq \frac{n \Sigma_h + \lambda \cI_\cH}{4}
\$
with probability at least $1-\delta$. 
On the same event, 
we similarly have Equation~\eqref{eq:bound_sample_rkhs}.

When $\gamma\in(0,1)$, we choose 
$m\in \NN$ such that 
$C_{1,\tau} C_{2,\tau}^{-1/2} \cdot \sqrt{2/\gamma} \cdot m^{(1-\gamma)/2} \cdot \exp(-C_{2,\tau}/2 \cdot m^\gamma) \leq \lambda/(40n)$, 
and set $\varepsilon = \lambda/(36n)$. 
Since $\gamma\in(0,1)$, it suffices to choose $m$ such that 
\$
C_{1,\tau} C_{2,\tau}^{-1/2} \cdot \sqrt{2/\gamma} \cdot \sqrt{m} \cdot \exp(-C_{2,\tau}/2 \cdot m^\gamma) \leq \lambda/(40n),
\$
or equivalently, 
\$
\frac{C_{2,\tau}}{2} \cdot m^\gamma \geq \log\bigg( \text{\small $ \frac{40n \cdot C_{1,\tau}}{\lambda  \sqrt{\gamma \cdot C_{2,\tau}/2}}$} \bigg)
+ \frac{\log m}{2}. 
\$
To this end, it suffices to choose 
some sufficiently large $m$ (larger than 
an absolute constant that only depends on $C_{2,\tau}$ and $\gamma$) such that $C_{2,\tau} \cdot m^\gamma 
\geq 2\log m$, and 
\$
m \geq  \bigg\{ \frac{2}{C_{2,\tau}}\log\bigg( \text{\small $ \frac{40n \cdot C_{1,\tau}}{\lambda  \sqrt{\gamma \cdot C_{2,\tau}/2}}$} \bigg) \bigg\}^{1/\gamma}.
\$
With a slight abuse of notations for absolute constants, 
we now show that we could choose $\lambda = C_4 \cdot \big[ \log (n/\delta)\big]^{1 +1/\gamma }$ for some sufficiently large 
absolute constant $C_4$ that only depends on 
$C_\psi$, $C_1$, $C_2$ and $\tau$. 
Without loss of generality we always have  
and $\lambda/2 \geq 16\log (2/\delta)\geq 16$. 
Also, there exists an absolute constant $C_5, C_6$ such that
\$
16m \log(1+2/\varepsilon) 
&\leq C_5 \cdot \bigg\{  \log\bigg( \text{\small $ \frac{40n \cdot C_{1,\tau}}{\lambda  \sqrt{\gamma \cdot C_{2,\tau}/2}}$} \bigg) \bigg\}^{1/\gamma} \log\bigg( 1 + \text{$\frac{72n}{\lambda }$} \bigg) \\ 
&\leq C_5 \cdot \bigg\{  \log\bigg( \text{\small $ \frac{40n \cdot C_{1,\tau}}{ \sqrt{ \gamma\cdot  C_{2,\tau}/2}}$} \bigg) \bigg\}^{1/\gamma} \log\big( 1 + 5n \big) \\ 
&\leq C_5 \cdot \big[ \log(C_6\cdot n)\big]^{1+1/\gamma}.
\$
Therefore, we can choose a sufficiently large 
absolute constant $C_4>0$ such that 
$2 C_5 \cdot \big[ \log(C_6\cdot n)\big]^{1+1/\gamma} 
\leq C_4 \cdot \big[ \log (n/\delta)\big]^{1+1/\gamma } = \lambda/2$. 
This verifies Equation~\eqref{eq:rkhs_lambda_general}. 
At the same time, such choice of $(\lambda, m, \varepsilon)$ 
ensures  $5nR_m \leq \lambda/8$ and 
$n\geq 2\lambda/3$ for sufficiently large $n$, hence 
$
(3n+\lambda) \varepsilon + 5nR_m \leq  \frac{\lambda}{4}.
$
Therefore, as we've verified the 
condition~\eqref{eq:rkhs_lambda_general}, 
from Lemma~\ref{lem:rkhs_concentration_Lambda_h}
we know that~\eqref{eq:bound_sample_rkhs} 
holds with probability at least $1-\delta$. 

Summarizing these two cases, we let 
$\lambda = C_4 \cdot \big[ \log (n/\delta)\big]^{1 +1/\gamma }$ for some sufficiently large 
absolute constant $C_4$ that only depends on 
$C_\psi$, $C_1$, $C_2$, $\gamma$ and $\tau$. 
When $n$ is sufficiently large, 
Equation~\eqref{eq:bound_sample_rkhs} 
holds with probability at least $1-\delta$. 
Finally, such choice of $\lambda$ leads to 
\$
B &= C \cdot H\cdot \big\{ \sqrt{\lambda} + (\log(KH/\xi))^{1/2 + 1/(2\gamma)}\big\}
\leq C' \cdot H \cdot \sqrt{(\log(KH/\xi))^{1  + 1 \gamma }}
\$
for some sufficiently large absolute constant $C'>0$ 
that does not depend on $K$ or $H$. 

\paragraph{(iii) $\gamma$-polynomial decay.}
In this case, we have 
\$
R_m^2 \leq C_\psi^2 \cdot C_1^{1-2\tau}  \sum_{j=m+1}^\infty 
j^{-\gamma(1-2\tau)}
\leq C_{1,\tau}^2 \int_{m}^\infty u^{-\gamma(1-2\tau)} \ud u
= C_{1,\tau}^2 \cdot \frac{m^{1- \gamma(1-2\tau)  }}{\gamma(1-2\tau)-1},
\$
where we define $C_{1,\tau} = C_\psi \cdot C_1^{1/2-\tau}$ 
for simplicity, and the second inequality uses the 
monotonicity of $u^{-\gamma(1-2\tau)}$ in $u\in [m,\infty)$. 
We also denote the absolute constant 
$C_{\gamma,\tau} = [ \gamma(1-2\tau)-1]/2 >0$, such that 
$R_m \leq 2C_{1,\tau} C_{\gamma,\tau }^{-1/2}\cdot m^{-C_{\gamma,\tau}}$ 
for $m\in \NN$.  
Without loss of generality, we always have 
$\lambda/2\geq 16\log(2/\delta)\geq 16$. 
We now let $\lambda = C_5\cdot n^{1/C_{\gamma,\tau}} \log(n/\delta)$
for some sufficently large absolute constant $C_5>0$ 
that only depends on $C_\psi$, $C_1$, $C_2$, $\gamma$ and $\tau$. 
We also choose $m\in \NN$ such that 
\#\label{eq:rkhs_m_poly}
2C_{1,\tau} C_{\gamma,\tau}^{-1/2} \cdot m^{-C_{\gamma,\tau}} \leq \frac{\lambda}{40 n},
\#
as well as $\varepsilon = \lambda/(36 n)$; 
as a result, we have 
$
(3n+\lambda) \varepsilon + 5nR_m \leq  \frac{\lambda}{4} 
$
once $n\geq 2\lambda/3$. 
We then verify that such choice satisfies Equation~\eqref{eq:rkhs_lambda_general}. 
Since we already have $\lambda/2 \geq 16\log(2/\delta)$, 
we only need to show $16m \log(1+2/\varepsilon) \leq \lambda/2$. 
By the choice of $m$ in Equation~\eqref{eq:rkhs_m_poly}, 
there exsists some absolute constants 
$C_6, C_7>0$ such that 
\$
16m \log(1+2/\varepsilon) 
\leq C_6\cdot n^{1/C_{\gamma,\tau}} \cdot \log(1+ 72n/\lambda) 
\leq C_6\cdot n^{1/C_{\gamma,\tau}} \cdot \log n\leq \lambda/2, 
\$
where we set
$\lambda = C_5\cdot n^{1/C_{\gamma,\tau}} \log(n/\delta)$ 
for some sufficiently large absolute constant $C_5>0$. 
Thus Equation~\eqref{eq:rkhs_lambda_general} holds, 
and Lemma~\ref{lem:rkhs_concentration_Lambda_h} implies that 
when $n$ is sufficiently large, 
Equation~\eqref{eq:bound_sample_rkhs} 
holds with probability at least $1-\delta$. 
Such choise of $\lambda$ leads to 
\$
B &= C  \cdot   H\cdot\big\{\sqrt{\lambda} + K^{\frac{d+1}{2(\gamma+d)}} H^{-\frac{d+1}{2(\gamma+d)}} \cdot \sqrt{\log(KH/\xi)} \big\} \\ 
&\leq C'' \cdot H \cdot \big\{ (K/H)^{-\frac{1}{1-\gamma(1-2\tau)}} 
\sqrt{\log(KH/\xi)}
+ K^{\frac{d+1}{2(\gamma+d)}} H^{-\frac{d+1}{2(\gamma+d)}} \cdot \sqrt{\log(KH/\xi)}\big\} \\ 
&\leq C' \cdot K^{\kappa^*} H^{\nu^*} \cdot \sqrt{\log(KH/\xi)}
\$
for some absolute constant $C'>0$ that does not 
depend on $K$ or $H$, where 
\$
\kappa^* = \frac{d+1}{2(\gamma +d)} + \frac{1}{\gamma(1-2\tau)-1},\quad 
\nu^* = 1 - \frac{d+1}{2(\gamma+d)} - \frac{1}{\gamma(1-2\tau)-1}.
\$
Therefore, we conclude the proof of Corollary~\ref{cor:rkhs_iid}. 
\end{proof}

\subsection{Proof of Lemma~\ref{lem:rkhs_concentration_Lambda_h}}
\label{app:subsec_rkhs_concentration_Lambda_h}

\begin{proof}[Proof of Lemma~\ref{lem:rkhs_concentration_Lambda_h}]

Recall that the opreator norm of some positive-definite operator
$\Lambda\colon \cH\to \cH$ is defined as 
$\|\Lambda\|_{\oper} = \sup_{\|x\|_{\cH}=1} x^\top \Lambda x
= \sup_{\|x\|_\cH=1} = \langle x, \Lambda x\rangle_{\cH}$. 
For any $m \in \NN$, we define $\Pi_m\colon \cH\to \cH$ 
as the projection operator onto the subspace spanned by 
$\{\psi_j\}_{j\in[m]}$. 

\begin{lemma}\label{lem:rkhs_proj}
For any $m\in \NN$, let $R_m$ be defined in Equation~\eqref{eq:def_rm}. 
Then for any $z\in \cZ$,  
\$
\big\|\phi(z) - \Pi_m\big[ \phi(z) \big] \big\|_\cH \leq R_m.
\$
Furthermore, for any $n\in \NN$ and any $\phi_i = \phi(z_i)$ 
for $z_i\in \cZ$, $i\in[n]$, it holds that 
\$
\bigg\|\frac{1}{n}\sum_{i=1}^n \phi_i \phi_i^\top - \frac{1}{n}\sum_{i=1}^n \Pi_m[\phi_i]\big(\Pi_m[\phi_i]\big)^\top \bigg\|_{\oper} \leq 2R_m.
\$
Also, for any random variable $z\in \cZ$, it holds that  
\$
\bigg\|\EE\big[  \phi (z) \phi(z)^\top\big] - \EE\Big[ \Pi_m[\phi(z)]\big(\Pi_m[\phi (z)]\big)^\top \Big] \bigg\|_{\oper} \leq 2R_m.
\$
\end{lemma}

\begin{proof}[Proof of Lemma~\ref{lem:rkhs_proj}]
Recalling the decomposition of $\phi(z)$ with respect to the 
orthogonal basis $\{\sqrt{\sigma_j}\cdot\psi_j\}_{j\geq 1}$ of $\cH$ in Equation~\eqref{eq:phi_z_decomp}, we know that 
\$
\big\|\phi(z) - \Pi_m\big[ \phi(z) \big] \big\|_\cH 
&= \bigg\| \sum_{j=m+1}^\infty \sqrt{\sigma_j}\cdot \psi_j(z) \cdot \sqrt{\sigma_j} \cdot \psi_j  \bigg\|_{\cH} 
= \bigg\{ \sum_{j=m+1}^\infty  \sigma_j \cdot \psi_j(z)^2  \bigg\}^{1/2} 
\leq R_m,
\$
where the second equality follows from $\|\sqrt{\sigma_j}\cdot \psi_j\|_\cH=1$ for all $j\geq 1$. 
Thus, for any fixed $x\in \cH$ with $\|x\|_\cH=1$ and any $i\in[n]$, we  have 
\$
x^\top \Big(\phi_i \phi_i^\top - \Pi_m[\phi_i]\big(\Pi_m[\phi_i]\big)^\top \Big)x
&= \big\langle \phi_i + \Pi_m[\phi_i], x \big\rangle_{\cH} 
\cdot \big\langle \phi_i - \Pi_m[\phi_i], x \big\rangle_{\cH} \\ 
&\leq \|x\|_\cH^2 \cdot \big\| \phi_i + \Pi_m[\phi_i] \big\|_{\cH} 
\cdot \big\| \phi_i - \Pi_m[\phi_i] \big\|_{\cH}  
  \leq 2 R_m,
\$
where we uses 
the fact that $\big\| \phi_i + \Pi_m[\phi_i] \big\|_{\cH}  \leq \big\| \phi_i  \big\|_{\cH}  + \big\|  \Pi_m[\phi_i] \big\|_{\cH} \leq 2$. 
Averaging over $i\in[n]$ or taking expectation over a random variable $z\sim \PP$, 
by the definition of operator norm, 
we complete the proof of Lemma~\ref{lem:rkhs_proj}.
\end{proof}

For notational simplicity, in the following we denote  $\phi_m(z):= \Pi_m\big[\phi(z)\big]$ 
for any $z\in \cZ$, $\Lambda_{m,h} = \sum_{\tau\in \cI_h}\phi_m(z_h^\tau)\phi_m(z_h^\tau)^\top + \lambda \cI_h$, and 
\$
\Sigma_{m,h} = \EE_{\pi^b}\big[ \phi_m(z_h)\phi_m(z_h)^\top  \biggiven s_1=x\big],\quad 
\Sigma_{m,h}^* = \EE_{\pi^*}\big[ \phi_m(z_h)\phi_m(z_h)^\top  \biggiven s_1=x\big]
\$
for any $h\in[H]$ and $m\in \NN$. 
Then Lemma~\ref{lem:rkhs_proj} implies 
\$
\big\|\Sigma_{m,h} - \Sigma_h \big\|_{\oper} \leq 2R_m,\quad 
\big\| \Lambda_{m,h} - \Lambda_h \big\|_{\oper} \leq 2n\cdot R_m. 
\$

The following lemma relates the sample and population effective dimensions 
projected to the subspace spanned by $\{\psi_j\}_{j\in[m]}$, 
adapted from \citet[Lemma 39]{zanette2021cautiously}. Let $n =K/H$ 
denote the size of the fold $\cI_h$ for any $h\in[H]$. 
Define the operators 
$\hat\Sigma_{m,h} = n^{-1} (\Lambda_{m,h} - \lambda \cI_\cH)=  n^{-1}\sum_{\tau\in \cI_h}\phi_m(z_h^\tau)\phi_m(z_h^\tau)^\top$ for any $h\in [H]$ and any $m\in\NN$. 

\begin{lemma}
\label{lem:inv_cov_concentration}
Fix any $x\in \cH$ with $\|x\|_{\cH}\leq 1$  
and any $\delta\in(0,1)$. 
Suppose $\lambda \geq 24\log(2/\delta)$. 
Then for any fixed $h\in[H]$ and $m\in\NN$, it holds with probability at least $1-\delta$ that 
\$
\big|x^\top \hat\Sigma_{m,h} x - x^\top \Sigma_{m,h} x \big| \leq \frac{1}{2}\Big(x^\top \Sigma_{m,h} x + \frac{\lambda}{n }\Big).
\$
\end{lemma}

\begin{proof}[Proof of Lemma~\ref{lem:inv_cov_concentration}]
Firstly, we note that 
\$
x^\top \hat\Sigma_{m,h} x = \frac{1}{n } \sum_{\tau\in \cI_h} x^\top \phi_m(z_h^\tau) \phi_m(z_h^\tau)^\top x = \frac{1}{n} \sum_{\tau\in \cI_h} \big(x^\top \phi_m(z_h^\tau)\big)^2,
\$
where $\big\{(x^\top \phi_m(z_h^\tau))^2\big\}_{\tau\in \cI_h}$ are i.i.d.~random variables whose expectation is given by  
\$
x^\top  \Sigma_{m,h} x =  \EE_{\pi^b}\big[ x^\top \phi_m(z_h ) \phi_m(z_h )^\top x \big]= \EE_{\pi^b}\left[ \big(x^\top \phi_m(z_h )\big)^2\right],
\$
and the expectation is with respect to 
the distribution of $z_h = (x_h,a_h)$ induced 
by the behavior policy $\pi^b$. 
Meanwhile, since $\|x\|_{\cH}\leq 1$, we know $|x^\top \phi_m(z)|\leq \|x\|_{\cH}\cdot\|\phi_m(z)\|_{\cH}\leq \|x\|_{\cH}\cdot\|\phi(z)\|_{\cH}\leq 1$ for any $z\in \cZ$. Hence 
\$
\Var\left(\big(x^\top \phi_m(z_h^\tau)\big)^2\right)\leq \EE_{\pi^b}\left[ \big(x^\top \phi_m(z_h )\big)^4\right] \leq \EE_{\pi^b}\left[ \big(x^\top \phi_m(z_h )\big)^2\right]  =  x^\top \Sigma_{m,h} x.
\$
Applying Lemma~\ref{lem:matrix_concentration}
to $\{(x^\top \phi_m(z_h^\tau))^2\}_{\tau\in \cI_h} \subset \RR$ implies that for any $t\in \RR$, 
\$
\PP\Big( \big| x^\top \hat\Sigma_{m,h} x
-  x^\top \Sigma_{m,h} x\big|  \geq t \Big) \leq 2\exp \bigg(  -\frac{n  t^2/2}{ x^\top \Sigma_{m,h} x +  t/3 }  \bigg).
\$
For any fixed $\delta\in (0,1)$, 
taking $t=\sqrt{\frac{4x^\top \Sigma_{m,h} x}{n }\log \frac{2}{\delta}} + \frac{4}{3n }\log\frac{2}{\delta}$ leads to 
\$
\big| x^\top \hat\Sigma_{m,h} x
-  x^\top \Sigma_{m,h} x\big|  \leq \sqrt{\frac{4x^\top \Sigma_{m,h} x}{n }\log \frac{2}{\delta}} + \frac{4}{3n }\log\frac{2}{\delta}
\$
with probability at least $1-\delta$. 
Now recall that $\lambda \geq 24 \log (2/\delta)$. 
% We discuss two cases as follows. 
Firstly, if $x^\top \Sigma_{m,h} x \leq \lambda/n $, then 
\$
\sqrt{\frac{4x^\top \Sigma_{m,h} x}{n }\log \frac{2}{\delta}} + \frac{4}{3n }\log\frac{2}{\delta}\leq 
\sqrt{\frac{4 \lambda}{n^2}\log \frac{2}{\delta}} + \frac{\lambda}{18n } \leq \frac{\lambda}{2n}
\leq \frac{1}{2}\Big(x^\top \Sigma_{m,h} x + \frac{\lambda}{n}\Big).
\$
Otherwise if $x^\top \Sigma_{m,h}x >\lambda/n$, then $x^\top \Sigma_{m,h}x > 24\log(2/\delta)/n$, hence 
\$
\sqrt{\frac{ (x^\top \Sigma_{m,h} x)^2}{6} } + \frac{x^\top \Sigma_{m,h}x}{18} 
\leq \frac{x^\top \Sigma_{m,h}x}{2} \leq \frac{1}{2}\Big(x^\top \Sigma_{m,h} x + \frac{\lambda}{n}\Big).
\$
Combining the two cases, we complete the proof of Lemma~\ref{lem:inv_cov_concentration}.
\end{proof}

For any fixed $m\in\NN$, 
we define 
% and any $\varepsilon \geq 0$, 
% we let $\cN_{\cH}(m,\varepsilon)$ be a
% $\varepsilon$-covering of 
$\cB(m,\cH) := \{x\in \Pi_m[\cH]\colon \|x\|_{\cH}\leq 1\}$, 
i.e., the unit RKHS norm ball in 
the subspace spanned by $\{\sqrt{\sigma_j}\cdot \psi_j\}_{j\in[m]}$. 
Then for any $x \in \cB(m,\cH)$, we have the decomposition 
$
x = \sum_{j=1}^m w_j \cdot \sqrt{\sigma_j} \cdot \psi_j
$
where $\sum_{j=1}^m w_j^2 = \|x\|_{\cH}^2 \leq 1$. 
On the other hand, for any $w\in \RR^m$, there exists some 
$x = \sum_{j=1}^m w_j \cdot \sqrt{\sigma_j} \cdot \psi_j$ 
that satisfies $x\in \cB(m,\varepsilon)$. 
Thus,  $\Pi\colon x\mapsto (w_j)_{j\in[m]}$ 
is a one-to-one mapping from $\cH$ to the Euclidean space $\RR^m$. 
Also, $\|x\|_{\cH} = \|\Pi(x)\|$ for any $x\in \cZ$ 
where $\|\cdot\|$ is the Euclidean norm in $\RR^m$. 
Consequently, the $\varepsilon$-covering number of $\cB(m,\varepsilon)$ 
equals the $\varepsilon$-covering number of a unit Euclidean ball in $\RR^m$. 
By \citet[Corollary 4.2.13]{vershynin2018high}, we can take 
an $\varepsilon$-covering of $\cB(m,\cH)$ under the RKHS norm,  
denoted as $\cN_\cH(m,\varepsilon) = \{x_1,\dots,x_{N_\cH(m,\varepsilon)}\}\subset \cB(m,\cH)$, 
which satisfies
\#\label{eq:rkhs_cover}
\log N_\cH(m,\varepsilon) := \log \big|\cN_\cH(m,\varepsilon)\big| \leq m \log(1+2/\varepsilon).
\#
% Now let $\cN_\cH(m,\varepsilon) = \{x_1,\dots,x_{N_\cH(m,\varepsilon)}\}$ be an $\varepsilon$-covering 
% of $\cB(m,\cH) := \{x\in \Pi_m[\cH]\colon \|x\|_{\cH}\leq 1\}$. 
Taking a union bound over $x_i/\|x_i\|_\cH$ for all $x_i\in\cN_\cH(m,\varepsilon)$, we know from 
Lemma~\ref{lem:rkhs_concentration} that if $\lambda \geq 16\log(2 N_\cH(m,\varepsilon)/\delta)$, 
then with probability at least $1-\delta$, 
\#\label{eq:unif_concen_rkhs}
\big|x_i^\top \hat\Sigma_{m,h} x_i - x_i^\top \Sigma_{m,h} x_i \big| \leq \frac{1}{2}\Big(x_i^\top \Sigma_{m,h} x_i + \frac{\lambda}{n } x_i^\top x_i\Big).
\# 
holds 
simultaneously for all $x_i\in \cN_{\cH}(m,\varepsilon)$.

For any $x\in \cH$ with $\|x\|_{\cH}\leq 1$, 
since $\phi_m(z)\in \Pi_m(\cH)$ for any $z\in \cZ$, we know 
\$
x^\top \phi_m(z ) \phi_m(z)^\top x  
= \big(\Pi_m[x]\big)^\top \phi_m(z ) \phi_m(z)^\top \big(\Pi_m[x]\big).
\$
Therefore, the definition of $\hat\Sigma_{m,h}$ and $\Sigma_{m,h}$ implies 
\$
x^\top \hat\Sigma_{m,h}x = \big(\Pi_m[x]\big)^\top \hat\Sigma_{m,h} \big(\Pi_m[x]\big), \quad 
x^\top  \Sigma_{m,h}x = \big(\Pi_m[x]\big)^\top \Sigma_{m,h} \big(\Pi_m[x]\big)
\$
Meanwhile, we have $\Pi_m[x]\in \cB(m,\cH)$ since $\|\Pi_m[x]\|_\cH\leq \|x\|_\cH\leq 1$. 
Therefore, there exists some $x_i \in \cN_\cH(m,\varepsilon)$ such that 
$\|\Pi_m[x]- x_i\|_\cH\leq \varepsilon$. 
Meanwhile, $\|\hat\Sigma_{m,h}\|_{\oper}\leq 1$ and 
$\|\Sigma_{m,h}\|_{\oper}\leq 1$ as $\sup_{z\in \cZ}\|\phi_m(z)\|_{\cH}\leq 1$. 
Thus by Cauchy Schwarz inequality, 
\$
\big|x ^\top \hat\Sigma_{m,h} x  - x ^\top \Sigma_{m,h} x  \big| 
&\leq \big|x_i^\top \hat\Sigma_{m,h} x_i - x_i^\top \Sigma_{m,h} x_i \big| 
+ \big( \|\Sigma_{m,h}\|_{\oper}   + 
\|\hat\Sigma_{m,h}\|_{\oper} \big) \cdot \|\Pi_m[x]- x_i\|_\cH^2 \\ 
&\leq \big|x_i^\top \hat\Sigma_{m,h} x_i - x_i^\top \Sigma_{m,h} x_i \big| 
+ 2\varepsilon^2.
\$
% and similarly, 
% \$
% x_i^\top \Sigma_{m,h} x_i \leq x^\top \Sigma_{m,h} x + \varepsilon^2,\quad 
% x_i^\top x_i \leq x^\top x + \varepsilon^2.
% \$
% Thus, by Cauchy-Schwarz inequality, 
Therefore, on the event that Equation~\eqref{eq:unif_concen_rkhs} holds, 
we have 
\$%#\label{eq:rkhs_proj_bd1}
\big|x ^\top \hat\Sigma_{m,h} x  - x ^\top \Sigma_{m,h} x  \big| %\notag \\ 
% &\qquad \leq \big|x_i^\top \hat\Sigma_{m,h} x_i - x_i^\top \Sigma_{m,h} x_i \big| 
% + \big( \|\Sigma_{m,h}\|_{\oper}   + 
% \|\hat\Sigma_{m,h}\|_{\oper} \big) \cdot \|\Pi_m[x]- x_i\|_\cH^2\notag \\ 
&\leq \frac{1}{2}\Big(x_i^\top \Sigma_{m,h} x_i + \frac{\lambda}{n } x_i^\top x_i\Big) + 2\varepsilon^2 \notag\\ 
& \leq \frac{1}{2}\Big(x ^\top \Sigma_{m,h} x  + \frac{\lambda}{n } x ^\top x \Big) 
+  \Big( \|\Sigma_{m,h}\|_{\oper} + \frac{\lambda}{n }    \Big) \cdot \big\|x_i-\Pi[x]\big\|_\cH^2
+ 2\varepsilon^2 \notag\\
& \leq \frac{1}{2}\Big(x ^\top \Sigma_{m,h} x  + \frac{\lambda}{n } x ^\top x \Big)  + (3+\lambda/n ) \varepsilon^2 
\$%#
for all $x\in \cH$ such that $\|x\|_{\cH}\leq 1$, which further implies 
\#\label{eq:rkhs_proj_bd2}
& \big|x ^\top \hat\Sigma_{ h} x  - x ^\top \Sigma_{ h} x  \big| \notag \\
& \leq \big|x ^\top \hat\Sigma_{m,h} x  - x ^\top \Sigma_{m,h} x  \big| + 
  \big\|\Sigma_{m,h} - \Sigma_h \big\|_{\oper} + \big\|\hat\Sigma_{m,h} - \hat\Sigma_h \big\|_{\oper} \notag \\ 
% &\leq \big|x ^\top \hat\Sigma_{m,h} x  - x ^\top \Sigma_{m,h} x  \big| + 4R_m  \\ 
&\leq \frac{1}{2}\Big(x ^\top \Sigma_{m,h} x  + \frac{\lambda}{n } x ^\top x \Big)  + (3+\lambda/n) \varepsilon^2  +  \big\|\Sigma_{m,h} - \Sigma_h \big\|_{\oper} + \big\|\hat\Sigma_{m,h} - \hat\Sigma_h \big\|_{\oper} \notag \\ 
&\leq \frac{1}{2}\Big(x ^\top \Sigma_{h} x  + \frac{\lambda}{n } x ^\top x \Big) 
+ \frac{3}{2} \big\|\Sigma_{m,h} - \Sigma_h \big\|_{\oper}  + (3+\lambda/n ) \varepsilon^2  + \big\|\hat\Sigma_{m,h} - \hat\Sigma_h \big\|_{\oper} .
\#
Finally,  
combining Equation~\eqref{eq:rkhs_proj_bd2} and Lemma~\ref{lem:rkhs_proj}, we have 
\$
 \big|x ^\top \hat\Sigma_{ h} x  - x ^\top \Sigma_{ h} x  \big|   
&\leq \frac{1}{2}\Big(x ^\top \Sigma_{h} x  + \frac{\lambda}{n } x ^\top x \Big)  + (3+\lambda/n ) \varepsilon^2  + 5R_m 
\$
for all $x\in \cH$ such that $\|x\|_{\cH}\leq 1$, 
where we use the fact that $\big\|\Sigma_{m,h} - \Sigma_h \big\|_{\oper}\leq 2R_m$ 
and $\big\|\hat\Sigma_{m,h} - \hat\Sigma_h \big\|_{\oper} \leq 2R_m$. 
Therefore, recalling the definition  $\Lambda_h = n \hat\Sigma_h + \lambda \cI_\cH$, we 
know that it holds with probability at least $1-\delta$ that 
\$
\frac{n \Sigma_h + \lambda \cI_\cH}{2}  - \big((3n +\lambda)\varepsilon + 5n  R_m \big)\cdot \cI_\cH \preceq \Lambda_h \preceq  \frac{3(n \Sigma_h + \lambda \cI_\cH)}{2}  + \big((3n +\lambda)\varepsilon + 5n  R_m \big)\cdot \cI_\cH
\$
as long as 
$
\lambda \geq 16\log(2N_\cH(m,\varepsilon)/\delta).
$
By Equation~\eqref{eq:rkhs_cover}, it suffices to take 
\$%#\label{eq:rkhs_lambda_general}
\lambda \geq 16\log(2/\delta) + 16m \log(1+2/\varepsilon).
\$%#
Therefore, we complete the proof of Lemma~\ref{lem:rkhs_concentration_Lambda_h}.
\end{proof}

% }

\section{Supporting Lemmas}
%\begin{flushleft}
The following lemma characterizes the deviation of the sample mean of a random matrix. See, e.g., Theorem 1.6.2 of \cite{10.1561/2200000048} and the references therein.

\begin{lemma}[Matrix Bernstein Inequality \citep{10.1561/2200000048}] \label{lem:matrix_concentration}
Suppose that  $\{ A_k\}_{k=1}^n  $ are independent and  centered random matrices in $\RR^{d_1\times d_2}$, that is, $\EE[ A_k] = 0$ for all $k \in [n]$. 
Also, suppose that such random matrices are uniformly upper bounded in the matrix operator norm, that is, $\|A_k\|_{\oper} \leq L$ for all  $k \in [n]  $. Let $Z=\sum_{k=1}^n A_k$ and 
\$
v(Z) = \max\big\{ \|\E[ZZ^\top] \|_{\oper},\|\E[Z^\top Z]\|_{\oper} \big\}= \max\bigg\{ \Big\|\sum_{k=1}^n \E[A_k A_k^\top] \Big\|_{\oper},   \Big\|\sum_{k=1}^n \E[A_k^\top A_k]\Big\|_{\oper}\bigg\}.
\$ 
For all $t\geq 0$, we have 
\begin{equation*}
\PP \big (\|Z\|_{\oper} \geq t \big )\leq (d_1+d_2)\cdot  \exp\Bigl(-\frac{t^2/2}{v(Z)+L / 3 \cdot t }\Bigr).
\end{equation*}
%Furthermore,
%\begin{equation*}
%\E\|Z\|_2 \leq \sqrt{2v(Z) \log(d_1+d_2)} + \frac{1}{3}L\log(d_1+d_2).
%\end{equation*}
%
\end{lemma}
\begin{proof}
See, e.g.,  Theorem 1.6.2 of \cite{10.1561/2200000048} for a detailed proof.
\end{proof} 

The following lemma, which is obtained from \cite{abbasi2011improved},  establishes the concentration of 
 self-normalized processes.

\begin{lemma}[Concentration of Self-Normalized Processes \citep{abbasi2011improved}]
Let $\{\cF_t \}^\infty_{t=0}$ be a filtration and $\{\epsilon_t\}^\infty_{t=1}$ be an $\RR$-valued stochastic process such that $\epsilon_t$ is $\cF_{t} $-measurable for all $t\geq 1$.
Moreover, suppose that conditioning on $\cF_{t-1}$, 
 $\epsilon_t $ is a  zero-mean and $\sigma$-sub-Gaussian random variable for all $t\geq 1$, that is,  
 \$%\label{eq:def_subgaussian} 
  \EE[\epsilon_t\given \cF_{t-1}]=0,\qquad \EE\bigl[ \exp(\lambda \epsilon_t) \biggiven \cF_{t-1}\bigr]\leq \exp(\lambda^2\sigma^2/2) , \qquad \forall \lambda \in \RR. 
  \$
 Meanwhile, let $\{\phi_t\}_{t=1}^\infty$ be an $\RR^d$-valued stochastic process such that  $\phi_t $  is $\cF_{t -1}$-measurable for all $ t\geq 1$. 
Also, let  $M_0 \in \RR^{d\times d}$ be a  deterministic positive-definite matrix and 
\$
M_t = M_0 + \sum_{s=1}^t \phi_s\phi_s^\top
\$ for all $t\geq 1$. For all $\delta>0$, it holds that
\begin{equation*}
\Big\| \sum_{s=1}^t \phi_s \epsilon_s \Big\|_{ M_t ^{-1}}^2 \leq 2\sigma^2\cdot  \log \Bigl( \frac{\det(M_t)^{1/2}\cdot \det(M_0)^{- 1/2}}{\delta} \Bigr)
\end{equation*}
for all $t\ge1$ with probability at least $1-\delta$.
\label{lem:concen_self_normalized}
\end{lemma}
\begin{proof}
	See Theorem 1 of \cite{abbasi2011improved} for a detailed proof. 
\end{proof}
%\end{flushleft}

% \section{Technical Lemmas for RKHS}

% {\color{blue}
The following lemma from~\cite{abbasi2011improved} 
and~\cite{yang2020function}
establishes the bounds on 
self-normalized processes.

% \section{Technical Lemmas for RKHS}
\begin{lemma}[\cite{abbasi2011improved}, Lemma E.3 in~\cite{yang2020function}]
\label{lem:rkhs_norm_inner}
Let $\{\phi_t\}_{t\geq 1}$ be a sequence in the RKHS $\cH$.
Let $\Lambda_0 = \lambda\cdot \cI_\cH\colon \cH\to \cH$ for $\lambda \geq 1$
and $\cI_\cH$ is the identity mapping on $\cH$. 
For any $t\geq 1$, we define a self-adjoint and positive-definite operator 
$\Lambda_t = \Lambda_0 + \sum_{j=1}^t \phi_j\phi_j^\top$, so that 
$\Lambda_t f = \Lambda_0 f + \sum_{j=1}^t \langle \phi_j, f\rangle \phi_j$
for any $f\in \cH$.
Then for any $t\geq 1$, it holds that 
\$
\sum_{j=1}^t \min\big\{1, \phi_j^\top \Lambda_{j-1}^{-1} \phi_j \big\}
\leq 2\log\det (I + K_t/\lambda),
\$
where $K_t\in \RR^{t\times t}$ is the Gram matrix 
whose $(j,j')$-th element is given by $[K_t]_{j,j'} = \langle \phi_j, \phi_{j'}\rangle_\cH$
for any $j,j'\in [t]$. 
Moreover, if it further holds that $\sup_{t\geq 0}\{\|\phi_t\|_\cH\}\leq 1$, 
then 
\$
\log\det (I + K_t/\lambda) 
\leq \sum_{j=1}^t \min\big\{1, \phi_j^\top \Lambda_{j-1}^{-1} \phi_j \big\}
\leq 2\log\det (I + K_t/\lambda).
\$
\end{lemma}

\begin{proof}[Proof of Lemma~\ref{lem:rkhs_norm_inner}]
See Lemma E.3 in \cite{yang2020function} for a detailed proof. 
\end{proof}

\begin{lemma}[Concentration of Self-Normalized Processes in RKHS, \cite{chowdhury2017kernelized}]\label{lem:rkhs_concentration}
Let $\cH$ be an RKHS defined over $\cX\subseteq \RR^d$ 
with kernel function $K(\cdot,\cdot)\colon \cX\times \cX\to \RR$. 
Let $\{x_\tau\}_{\tau=1}^\infty \subset \cX$ be a discrete time stochastic process 
that is adapted to the filtration $\{\cF_t\}_{t=0}^\infty$. 
Let $\{\epsilon_\tau\}_{\tau=1}^\infty$ be a real-valued stochastic process 
such that (i) $\epsilon_\tau\in \cF_\tau$ and 
(ii) $\epsilon_\tau$ is zero-mean and $\sigma$-sub-Gaussian conditioning on $\cF_{\tau-1}$, i.e., 
\$
\EE[\epsilon_\tau \given \cF_{\tau-1}] = 0,\qquad 
\EE\big[e^{\lambda\epsilon_\tau}\biggiven \cF_{\tau-1} \big] \leq e^{\lambda^2\sigma^2/2},\quad \forall \lambda\in \RR. 
\$
Moreover, for any $t\geq 2$, let $E_t = (\epsilon_1,\dots,\epsilon_{t-1})^\top \in \RR^{t-1}$
and $K_t\in \RR^{(t-1)\times (t-1)}$ be the Gram matrix of $\{x_\tau\}_{\tau\in[t-1]}$.
Then for any $\eta>0$ and any $\delta\in (0,1)$, with probability at least $1-\delta$, 
it holds simultaneously for all $t\geq 1$ that 
\$
E_t^\top \big[(K_t+ \eta\cdot I)^{-1} + I\big]^{-1} E_t 
\leq \sigma^2 \cdot \log\det\big[ (1+\eta)\cdot I + K_t \big] + 2\sigma^2 \cdot \log(1/\delta).
\$
\end{lemma}

\begin{proof}[Proof of Lemma~\ref{lem:rkhs_concentration}]
See Theorem 1 in \cite{chowdhury2017kernelized} for a detailed proof. 
\end{proof}

\begin{lemma}[Lemma D.5 in~\cite{yang2020function}]\label{lem:RKHS_dim}
Let $\cZ$ be a compact subset of $\RR^d$ and $K\colon \cZ\times\cZ\to \RR$
be the RKHS kernel of $\cH$. We assume $K$ is a bounded kernel so that 
$\sup_{z\in\cZ}K(z,z)\leq 1$, 
and $K$ is continuously differentiable on $\cZ\times\cZ$.
Moreover, let $T_K$ be the integral operator induced by $K$ and the 
Lebesgue measure on $\cZ$, defined in Equation~\eqref{eq:def_itg_op}. 
Let $\{\sigma_j\}_{j\geq 1}$ be the non-increasing sequence of eigenvalues of $T_K$. 
Recall the definition of maximal information gain in Equation~\ref{eq:def_Gklambda}. 
We assume $\{\sigma_j\}_{j\geq 1}$ satisfies one of the following eigenvalue decay conditions:
\begin{enumerate}[(i)]
  \item $\gamma$-finite spectrum: $\sigma_j = 0$ for all $j>\gamma$, where $\gamma$ is a positive integer.
  \item $\gamma$-exponential decay: there exists some constants $C_1,C_2>0$ %, $\tau\in[0,1/2)$ and $C_\psi>0$ 
  such that $\sigma_j \leq C_1 \cdot \exp(-C_2 \cdot j^\gamma)$ %and $\sup_{z\in \cZ} \sigma_j^\tau \cdot |\psi_j(z)|\leq C_\psi$ 
  for all $j\geq 1$, where $\gamma>0$ is a positive constant. 
  \item $\gamma$-polynomial decay: there exists some constants $C_1>0$ %, $\tau\in[0,1/2)$ and $C_\psi>0$  
  such that $\sigma_j \leq C_1\cdot j^{-\gamma}$ %and $\sup_{z\in \cZ} \sigma_j^\tau \cdot |\psi_j(z)|\leq C_\psi$ 
  for all $j\geq 1$, where $\gamma\geq 2+1/d$ is a constant. %$\gamma>1$. 
\end{enumerate}
Suppose $\lambda\in [c_1,c_2]$ for absolute constants $c_1,c_2$. Then we have 
$$
G(K,\lambda) \leq 
\begin{cases}
  C \cdot \gamma \cdot \log K \quad &\gamma\textrm{-finite spectrum},\\
  C \cdot (\log K)^{1+1/\gamma}\quad  & \gamma\textrm{-exponential decay},\\
  C \cdot K^{(d+1)/(\gamma+d)}  \cdot \log K \quad  & \gamma\textrm{-polynomial decay},
\end{cases}
$$
where $C $ is an absolute constant that only depends on $d,\gamma, C_1,C_2,C,c_1$ and $c_2$. 
\end{lemma}
\begin{proof}[Proof  of Lemma~\ref{lem:RKHS_dim}]
See Lemma D.5 of \cite{yang2020function} for a detailed proof. 
\end{proof}

% \begin{lemma}[Lemma 4.4 of~\cite{zhang2005learning}, adapted from~\cite{yurinsky2006sums}]
% \label{lem:concentration_rkhs}
% Let $\xi_i$ be zero-mean independent random vectors in 
% a Hilbert space. Suppose there exist $B,M>0$ such that 
% for all $\ell\geq 2$, $\ell\in \NN$, 
% $\frac{1}{n}\sum_{i=1}^n \EE\|\xi_i \|_{\cH}^\ell \leq \frac{B^2}{2}\ell! M^{\ell-2}$. 
% Then for any $\delta>0$, it holds 
% with probability at least $1-\delta$ that 
% \$
% \Big\|\frac{1}{n}\sum_{i=1}^n \xi_i \Big\|_{\cH} \leq \frac{2M\log(2/\delta)}{n} + \sqrt{\frac{2\log(2/\delta)}{n}B}.
% \$
% \end{lemma}

% }

\end{document}